\pdfoutput=1
\documentclass{article}

\usepackage[main, final]{neurips_2026}

\usepackage[utf8]{inputenc} %
\usepackage[T1]{fontenc}    %
\usepackage{hyperref}       %
\usepackage{url}            %
\usepackage{booktabs}       %
\usepackage{amsfonts}       %
\usepackage{nicefrac}       %
\usepackage{microtype}      %
\usepackage{enumitem}
\usepackage{dsfont}

\def\Xint#1{\mathchoice
   {\XXint\displaystyle\textstyle{#1}}%
   {\XXint\textstyle\scriptstyle{#1}}%
   {\XXint\scriptstyle\scriptscriptstyle{#1}}%
   {\XXint\scriptscriptstyle\scriptscriptstyle{#1}}%
   \!\int}
\def\XXint#1#2#3{{\setbox0=\hbox{$#1{#2#3}{\int}$}
     \vcenter{\hbox{$#2#3$}}\kern-.5\wd0}}

\def\dashint{\Xint-}

\usepackage{amsmath}
\usepackage{amssymb}
\usepackage{mathtools}
\usepackage{empheq}
\usepackage{amsthm}
\usepackage{aliascnt}
\usepackage{bbm}
\usepackage{comment}
\usepackage[dvipsnames]{xcolor}
\usepackage[capitalize]{cleveref}
\usepackage{settings}
\usepackage{etoc}

\theoremstyle{plain}
\newtheorem{theorem}{Theorem}

\newtheorem{lemma}{Lemma}

\newtheorem{conjecture}{Conjecture}
\newtheorem{claim}{Claim}
\newaliascnt{result}{claim}
\newtheorem{result}[result]{Result}
\aliascntresetthe{result}

\newtheorem{corollary}{Corollary}
\theoremstyle{definition}

\theoremstyle{remark}
\newtheorem{remark}{Remark}

\title{Factual recall in linear associative memories: \\sharp asymptotics and mechanistic insights}

\author{%
  Alessio Giorlandino \\
  International School of \\
  Advanced Studies (SISSA) \\
  Trieste, Italy \\
  \texttt{agiorlan@sissa.it} \\
  \And
  Sebastian Goldt \\
  International School of \\
  Advanced Studies (SISSA) \\
  Trieste, Italy \\
  \texttt{sgoldt@sissa.it} \\
  \And
  Antoine Maillard \\
  INRIA Paris \& DI ENS \\
  PSL University \\
  Paris, France \\
  \texttt{antoine.maillard@inria.fr} \\
}

\begin{document}

\maketitle

\begin{abstract}
  Large language models excel at factual recall, yet the
fundamental limits of storing input--output associations with neural networks
remain unclear. We study these limits in a minimal setting: a linear associative
memory that maps $p$ input embeddings in $\mathbb{R}^d$ to their
$d$-dimensional targets via a linear layer, requiring each mapped input to score its own target
above all other outputs. Since all inputs compete for the same outputs, these
constraints are strongly correlated, which makes the storage capacity $p_c$ ---
the largest number of associations that can be stored --- hard to characterise.
We address this in three steps. First, we introduce a decoupled model, in which
each input competes with its own independent set of outputs, and provide
numerical and analytical evidence that it is equivalent to the original model, both in capacity and in the structure of the learnt solution. Second,
using tools from statistical physics, we show that the decoupled model has
capacity $p_c \log p_c / d^2 = 1/2$, extend this computation to linear
two-layer networks, and go beyond the capacity by predicting the spectrum of the
learnt weights and the storage mechanism. The latter explains how the optimal
solution outperforms a naïve Hebbian rule: rather than boosting input--output
alignments with broad fluctuations, it raises the correct scores just above the
extreme-value threshold set by the competing outputs. Third, we prove the decoupled model's
capacity thresholds rigorously, including in the non-convex two-layer case;
combined with a recent proof for the original model, this shows that both
models have the same capacity for full-rank weights. These findings give a sharp
characterisation of factual storage in linear networks, and suggest decoupling as
a general tool for memorisation and learning problems in which each input must be
matched to one of an extensive number of outputs.

\end{abstract}

\etocdepthtag.toc{mainmatter}
\begin{figure*}[t]
    \centering
    \includegraphics[width=\linewidth]{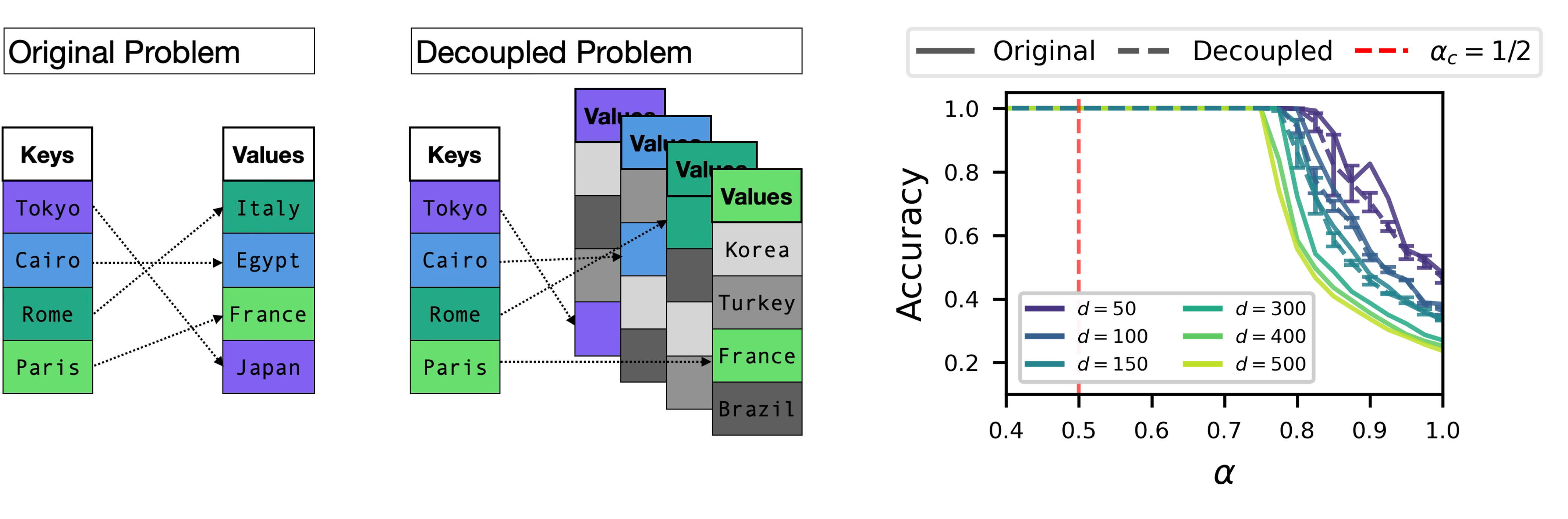}
    \vspace*{-3em}
    \caption{    \label{fig:figure-1} \textbf{Left:} The task is to memorise associations between inputs (keys) and outputs (values), illustrated here as cities and their countries. In the \emph{original problem}, all inputs share a common set of outputs. In the \emph{decoupled problem}, each input has its own independent set of competing outputs.
    \textbf{Right:} Empirical accuracy as a function of the load parameter
    $\alpha = p \log p / d^2$, shown for the original problem~\eqref{eq:OP}
    (solid), the decoupled problem~\eqref{eq:DP} (dashed), across several embedding dimensions $d$. The
    original and decoupled problems appear consistently equally hard, suggesting
    a common capacity threshold. Models are trained with Adam on the cross-entropy loss (see
    \cref{sec:experimental-details}). The red dashed line marks the asymptotic
    threshold $\alpha_c = 1/2$ of both problems (\cref{cor:full_rank}), which the
    transitions approach slowly due to finite-size effects of order $O((\log
    p)^{-1})$, see~\cref{sec:replica_results}.  }
\end{figure*}

\section{Introduction}

Large language models have a remarkable capacity to memorise factual associations from their training data \citep{petroni-etal-2019-language}. This observation has motivated a growing body of work aimed at localising and quantifying memorised facts within such models \citep{geva2021transformer,meng2022locating}. These empirical findings also raise fundamental theoretical questions, among them: how many distinct associations can neural networks reliably store, and what are the absolute limits governing factual memory in modern learning systems \citep{roberts2020much,allenzhu2024physicslanguagemodels33}?

In this paper, we seek to understand the fundamental limits of learning a set of associations with a simple neural network. We illustrate the task on the left of \cref{fig:figure-1}: the task is to learn the mapping of $N$ input tokens to $M$ output tokens according to some unknown ground-truth rule $f^*: [N] \to [M]$. We are given embeddings $\{e_x\}_{x \in [N]} \subset \bbR^d$ and $\{u_y\}_{y \in [M]} \subset \bbR^d$ for the input and output vocabularies respectively, and the goal is to learn a parametric model $F_W: \bbR^d \to \bbR^d$ 
that, given $e_x$, correctly associates the input to the corresponding output token, via the rule $\arg\max_{y \in [M]} u_y^\top F_W(e_x) = f^*(x)$. This setting was first introduced by \citet{cabannes2023scaling}, and the simplest instance, which we refer to as \emph{linear associative memory}, consists in taking $F_W$ as a linear map, i.e. $F_W(e_x) = We_x$ with $W \in \bbR^{d \times d}$. 
This simple setting, which we consider in our work, already exhibits some of the core phenomenology of associative memories, 
and is the subject of a very active line of research \citep{cabannes2023scaling, cabannes2024learning, nichani2024understanding,vural2026learning,kim2026sharpcapacityscalingspectral}.

More precisely, \citet{cabannes2023scaling,cabannes2024learning} study scaling laws and
training dynamics in the regime where the number of input patterns $N$ scales with the
embedding dimension $d$, while the number of possible outputs is $M = O(1)$, which is
closer to a typical classification setting with high-dimensional inputs and a finite
number of classes. \citet{vural2026learning} and
\citet{kim2026sharpcapacityscalingspectral} study the early-stage learning dynamics of
linear associative memories with various models for the embeddings, but do not address
the fundamental capacity limits of these architectures. 
To the best of our knowledge, \citet{nichani2024understanding} were the first to study
the capacity of associative memories in the more realistic case of a large number of
output tokens $M$ in a \emph{high-dimensional limit}, where we have an injective rule
$f^\star$, Gaussian input and output embeddings with dimension $d$, and a number of
associations $p \coloneqq N = M$. They show that if $d^2 \gtrsim p
\operatorname{polylog}(p)$, then with high probability over the draw of the embeddings,
the \emph{Hebbian ansatz} for the weight matrix $W_\Hebb \coloneqq \sum_{z \in [p]}
u_{f^*(z)} e_z^\top$ satisfies all the association constraints, and that it can be
obtained from a single step of gradient descent on the correlation loss. They also
provide numerical results for the capacity achieved by the minimiser of the
cross-entropy loss, which suggest that the storage capacity transition occurs in the
regime $d^2 \sim p \log p$.

In this work, we consider the same setting as~\citet{nichani2024understanding}
and crucially go beyond the Hebbian ansatz. Our goal is to provide sharp
thresholds for the maximum number of associations that can be stored as a
function of the embedding dimension $d$. Since we are interested in the optimal
weights~$W$ beyond a specific ansatz, we have to characterise the volume of the
space of weights that satisfy a given set of associations. This perspective is
inspired by classical works on the storage capacity of single-layer feed-forward
neural networks for binary classification, see~\citet{gardner:hal-03285587} and
the further references discussed below. Compared to storing binary labels for a
set of high-dimensional inputs in a single neuron, the associative memory problem
raises significant new challenges. Rather than assigning binary labels, each
input must be associated with one among many candidate outputs, and the
learnable parameters form a weight matrix rather than a weight vector. The main
difficulty in determining the optimal storage capacity lies in the strong
correlations between the constraints of different associations. 

\textbf{Our main contribution} is the introduction of a ``decoupled'' version of the storage problem, which allows us to derive a sharp capacity threshold of $p_c \log p_c / d^2 = 1/2$ for linear associative memories, and to show that this limit is achieved by a storage mechanism that is fundamentally different than the Hebbian construction. We prove this threshold rigorously for the decoupled problem, together with its generalisation to two-layer linear networks; combined with the proof of~\citet{barnfield2026sharpcapacitythresholdslinear} for the original problem, which appeared after the first version of this work, this shows that the two problems have the same capacity for full-rank matrices.

More precisely, our main contributions are as follows:
\begin{enumerate}[leftmargin=*, topsep=2pt, itemsep=2pt, parsep=2pt]
    \item \textbf{Decoupled formulation:} 
    We introduce a decoupled variant of the associative memory problem in which the constraints are independent in \cref{sec:conjecture}, see \cref{fig:figure-1} for an illustration.
    \item \textbf{Evidence for the equivalence}:
    We conjecture that this variant is equivalent to the original associative memory in the high-dimensional limit, based on three strands of evidence:
    \begin{enumerate}[label=$( \roman* )$, leftmargin=*, topsep=1pt, itemsep=1pt, parsep=2pt]
    \item Both models exhibit the same capacity, see \cref{sec:evidence-thresholds} and \cref{fig:accuracy-equivalence};
    \item The optimal solution in both models has the same asymptotic singular value distribution, see \cref{sec:evidence-spectra} and \cref{fig:spectra};
    \item Both models exhibit the same storage mechanism, see \cref{sec:evidence-scores} and \cref{fig:hist}.
    \end{enumerate}
    \item \textbf{Mechanistic insights}: We explain how the optimal solution beats the Hebbian solution by boosting the correct output scores just above a deterministic threshold, see \cref{sec:mechanistic-insights}.
    \item \textbf{Sharp capacity characterisation:} Using analytic techniques from statistical physics, 
    we derive an exact expression for the optimal storage capacity of the linear associative memory in terms of the load parameter $\alpha \coloneqq p \log p / d^2$ (recall $p$ is the number of associations and $d$ the embedding dimension) 
    and show good agreement with numerical simulations, see \cref{sec:replica_results}.
    \item \textbf{Two-layer (rank-constrained) case:} We study matrices $W = Q R^\T$ with $Q, R \in \bbR^{d \times m}$, corresponding to two linear layers with a smaller hidden dimension of size $m = \kappa d$. 
    For generic $\kappa \in (0, 1]$, we derive the sharp capacity threshold, generalising the full-rank case, and quantify how the inner layer size affects the memorisation capacity, %
    see \cref{subsec:capacity_rank_constrained}.
    We also give a theoretical prediction for the singular value distribution of the optimal solution close to capacity, showcasing excellent agreement with numerical simulations, see~\cref{fig:spectra}.
    \item \textbf{Rigorous capacity threshold:} We prove that, for every $\kappa \in (0,1]$, the capacity threshold of the decoupled problem is exactly the one predicted by the replica-symmetric computation, although for $\kappa < 1$ the problem is non-convex, see~\cref{thm:rank_capacity}. The achievability part also holds for the original problem, which can therefore store at least as many associations as the decoupled one.
\end{enumerate}
Collectively, our results offer the first precise characterisation of the capacity of linear associative memories in the high-dimensional regime, providing new insights into the mechanisms and the fundamental limits of factual memory in neural networks.

\subsection*{Further related work} 

\paragraph{Storage capacity in supervised learning} A classical line of work characterises the storage capacity of neural networks. The capacity of a perceptron, a single neuron used for binary classification, can be sharply characterised and is linear in the input dimension, $p \sim d$, for a variety of rules~\citep{covergeometrical, gardner:hal-03285587,krauth1989storage}. More recently, non-convex variants of the problem have been studied in both statistical physics~\citep{franz2017universality} and probability theory~\citep{montanari2024tractability, huang2024capacity}. In factual recall, rather than assigning binary labels, the learnable parameters form a weight matrix rather than a vector, meaning that in our setting the maximal number of associations $p$ scales with the embedding dimension as $d^2 \sim p \log p$. The quadratic dependence reflects the $d^2$ degrees of freedom of the matrix~$W$, while the $\log p$ factor arises because each constraint involves an optimisation over $p$ competing outputs.

\paragraph{Storage capacity of associative memories} The canonical example of an auto-associative memory are Hopfield networks \citep{hopfield1982neural}, where the goal is to store a number of patterns as fixed points of the network dynamics. The Hopfield model can store $p \simeq 0.14 d$ patterns, a result that was first obtained using techniques from statistical physics by \citet{amit1985spin}. More recently, it was shown that \emph{dense} associative memories with non-polynomial activation functions can store an exponential number of patterns~\citep{krotov2016dense, 
Demircigil_2017, krotov2020large, ramsauer2021hopfieldnetworksneed}; 
notably, \citet{lucibello2024capacity} provided sharp asymptotics on the capacity using tools from statistical physics. Rather than simply storing patterns, here we instead seek to map inputs to 
distinct outputs as in classical correlation-matrix memories 
\citep{anderson1972simple, kohonen1972correlation, kosko1988bidirectional} 
and modern key--value retrieval architectures \citep{weston2014memory}.  

\paragraph{Parallel work} This work was conducted in parallel with an independent study by \citet{barnfield2026sharpcapacitythresholdslinear} that also explores the optimal capacity of linear associative memories.
In the first version of their work, \citet{barnfield2026sharpcapacitythresholdslinear} prove that the capacity threshold scales as $d^2 \sim p \log p$, obtain explicit bounds for the critical capacity on this scale, and conjecture its exact value $\alpha_c = 1/2$.
In a revised version, which appeared after the first version of the present work, they prove rigorously that $\alpha_c = 1/2$ for the original problem, confirming the threshold we derive from the decoupled model. Building on their approach, we prove the capacity threshold for rank-constrained weights in the decoupled problem, and its achievability in the original one (\cref{thm:rank_capacity}).
Beyond the argmax objective, they also introduce a convex Tail-Average Margin relaxation in the regime $p \asymp d^2$, for which they derive exact asymptotic predictions and an explicit phase transition in the ridgeless limit.
In our work, we take another approach, and instead argue that the mechanism behind memorisation remains intact when decoupling the different constraints.
Beyond the capacity threshold, this allows us to characterise properties of the optimal memorisation matrix (such as its asymptotic spectrum and the storage mechanism) and to generalise our results to two-layer linear models.
Taken together, both works shed new light on the mechanisms driving the memorisation capacity of linear associative memories.

\section{A decoupled variant of associative memory}%
\label{sec:conjecture}

In this section, we recall the original associative memory problem in its precise form, following~\citet{cabannes2023scaling,nichani2024understanding}. We then propose a decoupled variant of the problem and conjecture that the two problems share the same storage capacity.

\paragraph{Original Problem (OP) --}
Let $E \coloneqq \{e_\mu\}_{\mu \in [p]} \subseteq \bbR^d$ and $U \coloneqq \{u_\rho\}_{\rho \in [p]} \subseteq \bbR^d$ denote sets of input and output embeddings, each drawn independently from a standard Gaussian distribution $\mathcal{N}(0, \Id_d)$.\footnote{This choice is without loss of generality for Gaussian embeddings with full-rank covariances $\Sigma_e, \Sigma_u$: writing $e_\mu = \Sigma_e^{1/2}\tilde e_\mu$ and $u_\rho = \Sigma_u^{1/2}\tilde u_\rho$, one has $u_\rho^\top W e_\mu = \tilde u_\rho^\top \tilde W \tilde e_\mu$ with $\tilde W \coloneqq \Sigma_u^{1/2} W \Sigma_e^{1/2}$, and $W \mapsto \tilde W$ is a rank-preserving bijection of $\bbR^{d \times d}$. The capacity threshold is thus the same as for isotropic embeddings.} We are given $p$ associations specified by an injective mapping $f^* : [p] \to [p]$, yielding the paired set $\{(e_\mu, u_{f^*(\mu)})\}_{\mu=1}^p$. The objective is to learn a matrix~$W \in \bbR^{d \times d}$ such that, for every $\mu \in [p]$, the correct output vector achieves the highest score:
\begin{equation}
  (\mathrm{OP}): \quad \forall \mu \in [p], \quad
  \arg\max_{\rho \in [p]} \; u_\rho^\top W e_\mu \;=\; f^*(\mu).
  \label{eq:OP}
\end{equation}
Without loss of generality, we assume throughout the remainder of the manuscript that the associations are ordered, i.e., $f^*(\mu) = \mu$ for all $\mu \in [p]$. We focus on two classes of weight matrices: general $W \in \bbR^{d \times d}$, and the rank-constrained case $W = QR^\top$ with $Q, R \in \bbR^{d \times m}$ and hidden dimension $m = \kappa d$ for $\kappa \in (0,1]$, corresponding to a two-layer linear network~\citep{nichani2024understanding}.

\begin{figure}[t]
    \centering
    \includegraphics[width=1\linewidth]{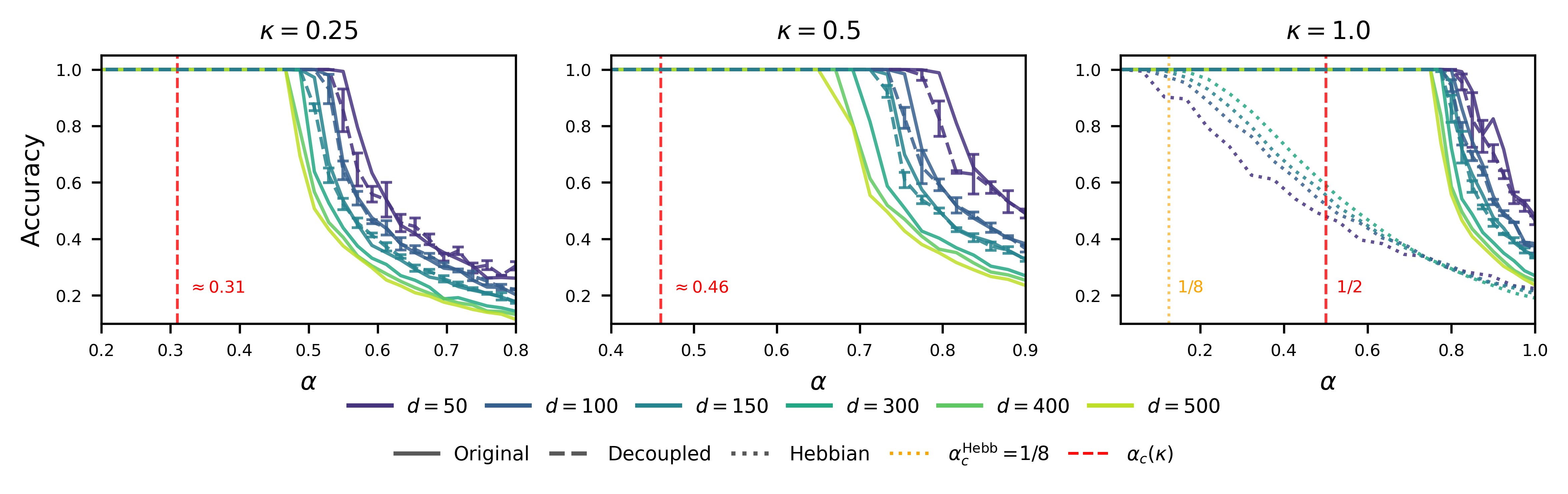}
    \caption{\label{fig:accuracy-equivalence} Empirical accuracy as a function
        of the load parameter $\alpha = p \log p / d^2$, for three rank
        fractions $\kappa \in \{0.25, 0.5, 1.0\}$. In each
        panel, curves correspond to the original problem~\eqref{eq:OP} (solid)
        and the decoupled problem~\eqref{eq:DP} (dashed), across several
        embedding dimensions $d$. In the full-rank case ($\kappa = 1$) the
        Hebbian ansatz (dotted) is also shown. Red dashed lines mark the
        asymptotic threshold $\alpha_c(\kappa)$ given by
        \cref{thm:rank_capacity} ($\alpha_c \approx 0.31, 0.46, 0.50$). The
        transitions for the original and decoupled problems are
        indistinguishable across all $\kappa$, supporting the equivalence
        conjecture also in the rank-constrained case, and both substantially
        outperform the Hebbian ansatz at $\kappa = 1$. Models are trained
        with Adam on the cross-entropy loss with $W = QR^\top$,
        $Q, R \in \mathbb{R}^{d \times (\kappa d)}$;
        see~\cref{sec:experimental-details}.}
        \vspace*{-1em}
\end{figure}

\paragraph{Decoupled Problem (DP) --}
The difficulty in analysing the original problem of eq.~\eqref{eq:OP} directly is that the constraints are coupled across examples, 
in sharp contrast with classical perceptron-type models à la
\citet{gardner:hal-03285587}, where each input vector is
mapped to a binary label that is independent of all other labels and inputs. 
In contrast, in factual recall the model must select the correct target for each input vector
$e_\mu$ from the \emph{same} set targets $U$.
This observation motivates the study of a less correlated variant in which each input $e_\mu$ is associated with its own independent set of candidate outputs $U^{(\mu)}$.
In this \emph{decoupled} setting, which we formally define below, the competition set differs across inputs,
substantially decoupling the constraints and simplifying greatly the analytical
treatment, while preserving the essential structure of the associative
memory task.

As in the original problem, let 
$E \coloneqq \{ e_\mu\}_{\mu \in [p]}  \iid \mcN(0, \Id_d)$ denote a set of input vectors.
In contrast to the original problem, we introduce \emph{for each} $\mu \in [p]$ a set of output vectors
\[
U^{(\mu)} \coloneqq \{ u^{(\mu)}_\rho\}_{\rho \in [p]} \subseteq \bbR^d,
\]
where all vectors are drawn i.i.d.\ from $\mathcal{N}(0, \Id_d)$ and are
\emph{independent} across different values of $\mu \in [p]$. This modification is the only
change with respect to the original problem; the objective remains to learn a
matrix~$W \in \bbR^{d \times d}$ such that, for every input index $\mu$, the
designated output $u^{(\mu)}_\mu$ achieves the largest score among its local
candidate set:
\begin{equation}
(\DP): \quad \forall \mu \in [p], \quad
\arg\max_{\rho \in [p]} \; u^{(\mu)\top}_\rho W e_\mu \;=\; f^*(\mu) .
\label{eq:DP}
\end{equation}
As in the original problem, we assume without loss of generality that $f^*(\mu) = \mu$ for all $\mu \in [p]$.

On a technical level, we expect the decoupled problem to fall within the scope of Gaussian universality: each constraint depends on \(W\) only through bulk statistics of the weight matrix, so the high-dimensional behaviour is governed primarily by its covariance structure, in line with prior work on universality in high-dimensional inference \citep{hu2022universality, maillard2023exact, goldt2022gaussian, montanari2022universality} (see~\cref{app:decoupled}). In contrast, the shared outputs in the original problem induce strong correlations between constraints, which make a similar analysis of \((\OP)\) intractable with these tools (see \cref{sec:Original-appendix}, where we also formulate the equivalence of the two problems as a statement on orthant probabilities).

\paragraph{Equivalence between original and decoupled problems --} Our first main result is the following conjecture on the equivalence between the original and the decoupled problems. Here, we state the conjecture; in \cref{sec:evidence_conjecture}, we present different strands of evidence.
\begin{conjecture}[Capacity Equivalence]
\label{conjecture-capacity}
Assume that $d,p \to \infty$ with a fixed ratio
\begin{equation}
    \alpha \coloneqq \lim_{d \to \infty} \frac{p \log p}{d^2} \in (0, \infty),
    \label{eq:alpha}
\end{equation}
and fix a rank fraction $\kappa \in (0,1]$, the case $\kappa = 1$ corresponding to general matrices $W \in \bbR^{d \times d}$.
Let $\mcP \in \{(\OP), (\DP)\}$ denote either problem formulation.
For each $\mcP$, there exists a critical capacity threshold $\alpha_c^{\mcP}(\kappa)$ such that, with high probability over the random draw of the input and output vectors, there exists a solution $W^\star$ of rank at most $m = \kappa d$ storing all associations if $\alpha < \alpha_c^{\mcP}(\kappa)$, whereas for $\alpha > \alpha_c^{\mcP}(\kappa)$ no such solution exists. 
Moreover we have:
\begin{equation*}
    \alpha_c^{\OP}(\kappa) = \alpha_c^{\DP}(\kappa).
\end{equation*}
\end{conjecture}
We write $\alpha_c^{\mcP} \coloneqq \alpha_c^{\mcP}(1)$ in the full-rank case. \cref{conjecture-capacity} is supported by two sets of evidence. The equivalence between $(\DP)$ and $(\OP)$ is corroborated by numerical and analytical evidence presented below, for both full-rank and rank-constrained matrices. The threshold of the decoupled problem is then characterised sharply: the replica method yields $\alpha_c^{\DP} = 1/2$ in the full-rank case (\cref{claim:capacity}) and an explicit formula for $\alpha_c^{\DP}(\kappa)$ (\cref{claim:rank}), which we prove rigorously for every $\kappa$ (\cref{thm:rank_capacity}).

\textbf{Status of~\cref{conjecture-capacity} --} For full-rank matrices, \cref{conjecture-capacity} is now proven (\cref{cor:full_rank}): since the first version of this work appeared, \citet{barnfield2026sharpcapacitythresholdslinear} have proved that $\alpha_c^{\OP} = 1/2$, and \cref{thm:rank_capacity} gives $\alpha_c^{\DP} = 1/2$. For $\kappa < 1$, \cref{thm:rank_capacity} determines $\alpha_c^{\DP}(\kappa)$ rigorously, so that \cref{conjecture-capacity} predicts the same capacity for the original problem, as supported by our numerical experiments (\cref{sec:evidence-thresholds}). The same theorem shows that this capacity is achievable in the original problem as well; the matching impossibility bound remains open. Finally, our evidence indicates that the equivalence extends beyond the capacity threshold: the two problems also share the spectrum of the learnt weights and the storage mechanism (\cref{sec:evidence-spectra,sec:evidence-scores}).

\section{Evidence for \cref{conjecture-capacity}} %
\label{sec:evidence_conjecture}

\subsection{Evidence I: Retrieval accuracy and capacity threshold}
\label{sec:evidence-thresholds}

We first validate numerically that the original problem~$(\OP)$ and the decoupled problem~$(\DP)$ share the same critical capacity. In \cref{fig:accuracy-equivalence} (right), we show the empirical retrieval accuracy as a function of the load parameter $\alpha = 
p\log p / d^2$ for a linear model 
trained 
using the cross-entropy loss, for different rank fractions $\kappa$ and embedding dimensions $d$. For both~$(\OP)$ and~$(\DP)$, the 
accuracy undergoes a sharp transition from perfect memorisation (and an accuracy equal to 1) to an extensive number of violated constraints. The two transition points are indistinguishable across all 
dimensions tested, supporting the conjectured equivalence $\alpha_c^{\OP}(\kappa) = 
\alpha_c^{\DP}(\kappa)$ in terms of their critical capacity. 

For the considered dimensions, the networks store more patterns than predicted by
the high-dimensional theory, consistently with the finite-size corrections of size
$\tmcO(1 / \log p)$ predicted by our derivation (see~\cref{sec:F-computation}), which
make the threshold converge slowly to its asymptotic value, as confirmed by the
finite-size analysis of~\cref{fig:finite-size}. The agreement between the two problems goes beyond the
capacity threshold: already in \cref{fig:figure-1}, the full curves of the
accuracy as a function of $\alpha$ match closely between $(\OP)$ and $(\DP)$,
not just at the transition point. The cross-entropy loss curves, shown
in~\cref{fig:loss}, further confirm this agreement.

\begin{figure}
    \centering
    \includegraphics[width=.9\linewidth]{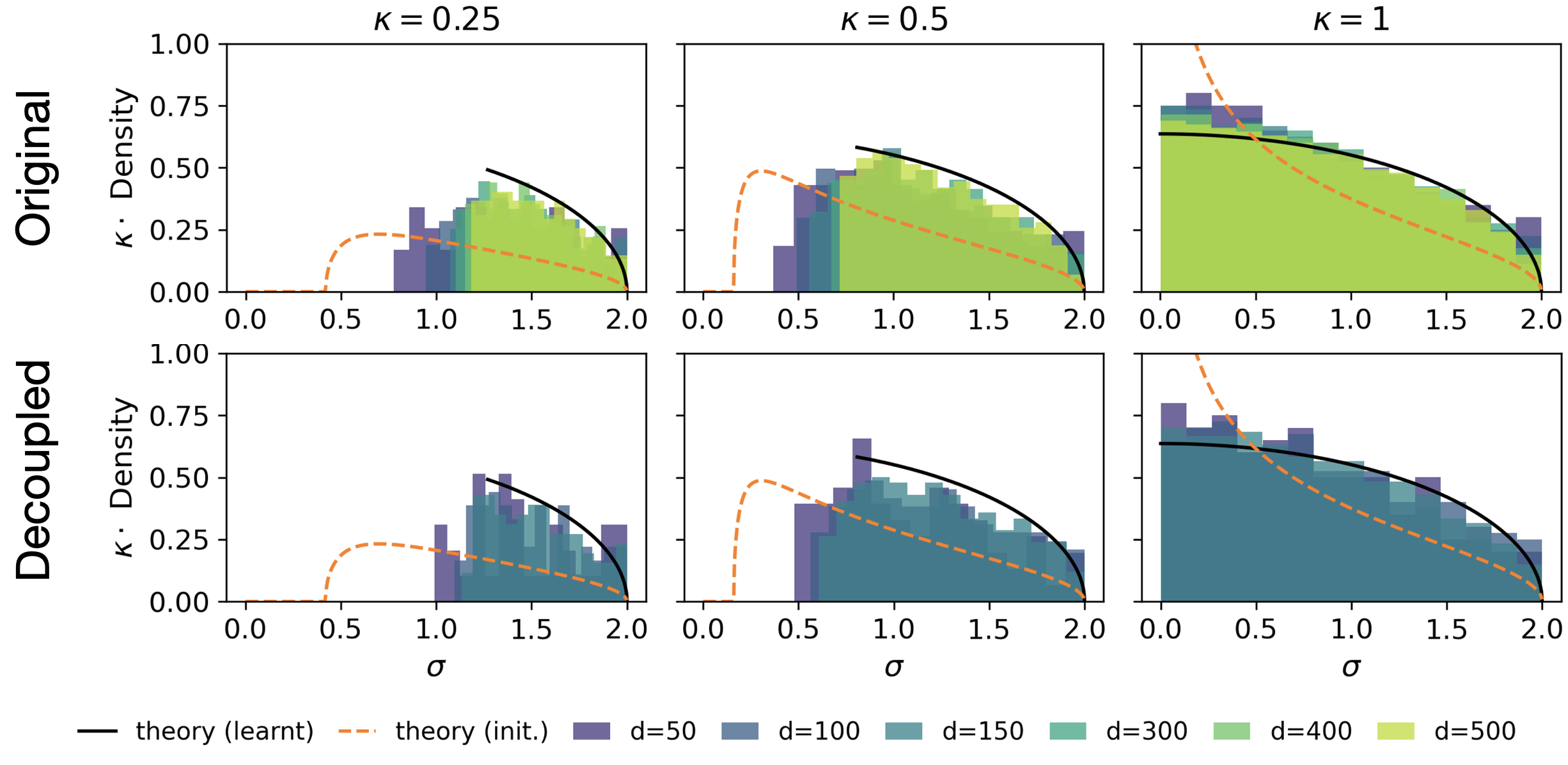}
    \caption{\label{fig:spectra} Distribution of the non-zero singular values at
    the capacity threshold in the $(\OP)$ (top) and $(\DP)$ (bottom) for various $d$ and $\kappa$, normalised to have top singular value equal to $2$. Each panel shows histograms from
    simulation data (shaded) alongside the theoretical prediction at capacity
    from \cref{claim:rank} (solid line) and the spectrum at initialisation
    (dashed line) for the two-layer parameterisation $W = QR^\top$ with i.i.d.
    Gaussian $Q, R \in \mathbb{R}^{d \times m}$, $m = \kappa d$. 
    The strong similarity between the spectra of the optimal solutions
    in the two problems provides further evidence for their equivalence. }
\end{figure}

\subsection{Evidence II: Structure of the learnt weights}
\label{sec:evidence-spectra}

Our second line of evidence concerns the structure of the learnt weights. \Cref{fig:spectra} shows the singular value distribution of the optimal weight matrices at the capacity threshold for different rank ratios $\kappa$ and input dimensions $d$. The first observation is that the spectra show a clear departure from the spectra at initialisation, which we indicate with the orange dashed line. Second, we observe again a good match between the spectra of the original and the decoupled problem. Finally, we note that the spectral density is well predicted by our theory; we plot the prediction of~\cref{claim:rank} with the black solid line. Interestingly, the strong agreement with the theory suggests that the singular value distribution at capacity converges much faster than the capacity threshold itself.

\subsection{Evidence III: Scores histograms}
\label{sec:evidence-scores}

Our final line of evidence for the equivalence between $(\OP)$ and~$(\DP)$ comes from an analysis of the scores $u_\rho^\top W e_\mu$ (and analogously $u^{(\mu)\top}_\rho W e_\mu$ in the decoupled problem), which also sheds light on the mechanism of pattern storage in the associative memory. Successful storage can in principle be achieved through multiple mechanisms: one may (1) increase the target scores where $\rho=\mu$ while leaving the non-target scores $\rho \neq \mu$ essentially unchanged, (2) decrease the non-target scores while keeping the target scores fixed, or (3)  combine both effects. 

We found that the associative memories opt for the strategy (1). In \cref{fig:hist}, we find that scores for non-target pairs follow a standard Gaussian distribution for various loads. In contrast, for the target pair $\rho = \mu$, learning drives the corresponding score towards the right tail of the Gaussian distribution of non-target scores, reaching values that must be larger than the maximum of the other $p-1$ Gaussian random variables, which is around $\sqrt{2 \log p}$. This appears to be the case for both the original and decoupled problems. Note that~\cref{fig:hist}
aggregates all target scores and all non-target scores across indices: 
a representative single-constraint histogram is shown in~\cref{fig:hist-mu}.

\section{Mechanistic insights into memory storage and retrieval}
\label{sec:mechanistic-insights}

The distribution of the scores gives additional insights into the storage
mechanisms of optimal associative memories. For a Hebbian weight matrix $W_\Hebb \coloneqq
\sum_{z \in [p]} u_{f^*(z)} e_z^\top$, off-target scores are approximately normally
distributed with mean zero and variance $1$ (in our units), while target scores
are normally distributed with mean $\sqrt{\alpha^{-1} \log p}$ and variance 1.
This approximation suggests that the model fails to store all associations when the distributions
overlap too much, in a sense that we make precise in~\cref{sec_app:hebbian}, and
that allows to derive an heuristic for the capacity threshold of the Hebbian
rule at $\alpha_c^{\mathrm{Hebb}} = 1 / 8$ (see~\cref{thm:Hebbian}), which
matches well numerical observations and agrees with the precise analysis of \citet{barnfield2026sharpcapacitythresholdslinear}.
The Hebbian weight matrix thus has a significantly lower capacity than the
optimal weight, cf.\ \cref{fig:accuracy-equivalence}.

The optimal weight matrix achieves the higher capacity $\alpha_c = 1 / 2$ by
adopting a different strategy close to the capacity threshold. Interestingly, as
we detail in~\cref{sec:F-computation}, our theory predicts that below
capacity, the diagonal scores rather concentrate (to leading order as $p \to \infty$) around a deterministic value
that is just enough to allow for correct mapping of inputs to outputs. This seems to match the observations of~\cref{fig:hist},
where we see the variance of the red histograms increase significantly as we
cross capacity, while in the Hebbian matrix the variance of the diagonal scores is always large (\cref{fig:hebbian}). This effect occurs in both problems, hinting at a common mechanism behind optimal capacity. 

\begin{figure}[t]
    \centering
    \includegraphics[width=\linewidth]{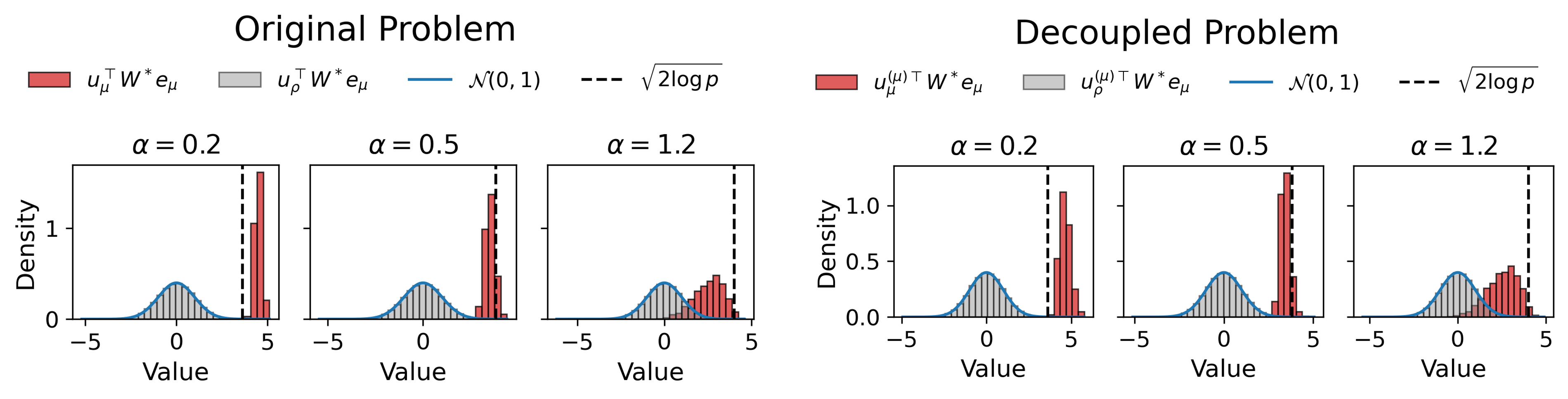}
    \caption{    \label{fig:hist}
Histograms of the target scores (red) and non-target scores (grey) aggregated 
over indices $\mu \in [p]$, for $(\OP)$(\eqref{eq:OP};left) and $(\DP)$(\eqref{eq:DP};right), shown for three representative 
values of the load parameter $\alpha$ ($\alpha = 0.2, 0.5, 1.2$), with $d = 150$. 
After training, the learned weight matrix $W^*$ is normalised so that the 
non-target scores $s_{\mu\rho} = u_\rho^{\top} W^* e_\mu$ have unit variance 
(the problem being invariant under rescaling of $W$, this fixes the scale); 
target scores are rescaled accordingly.
    }
\end{figure}

\section{Sharp capacity thresholds of the decoupled problem}
\label{sec:replica_results}

Motivated by the evidence of~\cref{sec:evidence_conjecture}, we now outline our computation of the memorisation capacity of the decoupled problem, with and without rank constraints; its decoupled structure is what makes this approach tractable (see~\cref{sec:Original-appendix,app:decoupled}).

\subsection{Statistical physics approach}  %

Following a classical statistical physics approach, we compute the \emph{fractional volume of the solution space} for the decoupled problem, conditioned on fixed realisations of the input and output embeddings, $(E, U)$.
Recall that the embeddings $E = (e_\mu)_{\mu\in[p]}$ and $U = (u^{(\mu)}_\rho)_{\mu, \rho \in [p]}$ are drawn independently in $\bbR^d$ from $\mcN(0, \Id_d)$.
We seek a weight matrix $W \in \bbR^{d \times d}$ with rank at most $m =\kappa d$ (for a constant $\kappa \in (0, 1]$) such that $W \in \Gamma(E, U)$,
with
\begin{equation}
\label{eq:volume-DP}
\begin{dcases}
\Gamma(E, U) &\coloneqq \left\{W \in \bbR^{d \times d} \, : \, \forall \mu \in [p], \, \max_{\rho \in [p]} u^{(\mu)\top}_\rho W e_\mu \leq u^{(\mu)\top}_\mu W e_\mu \right\}, \\
\mcV_{\DP}(E,U)
&\coloneqq P_m\left(\Gamma(E, U)\right).
\end{dcases}
\end{equation}
where $P_m$ is a probability measure on $W$ supported on the set of matrices of rank at most $m$, $\Gamma$ is the (convex) solution space, and $\mcV_{\DP}(E, U)$ its volume under $P_m$.
In practice, for $m = d$, we will take $P_m = \mcN(0, \Id_{d^2}/d)$ the standard Gaussian distribution with variance $1/d$, while for $m = \kappa d$ with $\kappa \in (0, 1]$ we take $P_m$ as the law of $W = Q R^\T / \sqrt{dm}$, for $Q, R \in \bbR^{d \times m}$ independently sampled with i.i.d.\ $\mcN(0, 1)$ elements. 
While the choice of $P_m$ will yield some analytic simplifications, the asymptotic capacity threshold is naturally agnostic to the choice of this prior.

\textbf{Scaling insights --}
For a fixed $W$,
by permutation invariance a given association $\mu \in [p]$ is satisfied with probability $1/p$, 
and so the probability that all the independent associations are satisfied is 
$(1/p)^p = e^{-p \log p}$. 
In particular, we have $\EE[\mcV_\DP(E, U)] = e^{-p \log p}$.
On the other hand, the parameter space has $d^2$ degree of freedom: a balancing with the entropic cost of the constraints suggests that the capacity phase transition occurs for a finite value of the load parameter
\begin{equation}\label{eq:alpha-scaling}
    \alpha \coloneqq \frac{p \log p}{d^2}.
\end{equation}
Notice that in the original problem, even the first moment of the volume of solutions is hard to characterise exactly, as we discuss in~\cref{sec:Original-appendix}.
Beyond this heuristic argument, the scaling of eq.~\eqref{eq:alpha-scaling} is naturally predicted by our formal capacity analysis, which we now outline.

\subsection{Full-rank linear associative memory}
\label{sec:claim-2}

Our main results are sharp computations of the asymptotics of the (normalised) volume:
\begin{equation}\label{eq:def_fentropy}
    \varphi_d \coloneqq d^{-2}\EE \log \mcV_{\DP}(E,U),
\end{equation}
where the average is over $(e_\mu, u_\rho^{(\mu)})_{\rho,\mu} \iid \mcN(0, \Id_d)$.
$\varphi_d$ is also called \emph{free entropy} in statistical physics.
Our first results concerns the case $\kappa = 1$ ($m = d$), i.e.\ without any rank constraint.
\begin{result}[Storage capacity of a decoupled linear associative memory]
\label{claim:capacity}
Consider $\varphi_d$ as defined in eq.~\eqref{eq:def_fentropy}, for $P_m = \mcN(0, \Id_{d^2}/d)$. In the limit $d, p \to \infty$ with $(p \log p)/d^2 \to \alpha > 0$:
\begin{itemize}[leftmargin=*, topsep=0pt, itemsep=0pt]
    \item If $\alpha < 1/2$, then (up to additive constants independent of $\alpha$)
    \begin{equation}\label{eq:free-entropy-RS}
    \lim_{d \to \infty} \varphi_d = \inf_{q \in [0,1)} \left\{
    \frac{1}{2}\log(1-q) + \frac{q}{2(1-q)} 
    + \alpha \lim_{p\to\infty} \frac{1}{\log p}
    \EE_{\eta \sim \mcN(0, \Id_p)}\bigl[\log f_p(q;\eta)\bigr]
    \right\},
\end{equation}
where $f_p(q;\eta)$ is defined in eq.~\eqref{eq:f-def}. In particular, $\varphi(\alpha) \coloneqq \lim_{d\to \infty} \varphi_d > -\infty$.
    \item As $\alpha \uparrow 1/2$, the minimiser $q^\star(\alpha)$ of eq.~\eqref{eq:free-entropy-RS} satisfies $q^\star(\alpha) \uparrow 1$, and $\varphi(\alpha) \to - \infty$.
\end{itemize}
\end{result}
In particular, \cref{claim:capacity} predicts the storage capacity threshold $\alpha_c^\DP = 1/2$, which we prove rigorously in~\cref{thm:rank_capacity}.
The derivation of~\cref{claim:capacity} is detailed in~\cref{sec:quenched-computation-DP}.

\textbf{The replica method --}
\cref{claim:capacity} is based on an application of the \emph{replica method} from statistical physics, a non-rigorous technique which has had
a wide range of applications, notably in statistical learning and the theory of neural networks~\citep{mezard1987spin,gardner:hal-03285587,gabrie2020mean,charbonneau2023spin,maillard2024bayes,barbier2025statistical}.
The replica method is widely believed to yield exact asymptotic predictions, as was established rigorously in a large range of problems~\citep{guerra2002thermodynamic,talagrand2006parisi,panchenko2013parisi,barbier2019optimal,gerbelot2022asymptotic,vilucchio2025asymptotics}.
More precisely, \cref{claim:capacity} relies on a so-called \emph{replica-symmetry} assumption, which is known to hold in a large class of convex constraint satisfaction problems~\citep{talagrand2010mean,barbier2022strong}.
Since the full-rank problem is convex, we expect this assumption to be exact here, and we therefore present this prediction as a Result; by contrast, the rank-constrained prediction of~\cref{subsec:capacity_rank_constrained} concerns a non-convex problem, for which replica symmetry rests on less firm grounds, and we state it as a Claim (\cref{claim:rank}). As we show in~\cref{subsec:capacity_rank_constrained}, the capacity threshold it predicts is nevertheless exact (\cref{thm:rank_capacity}).

Finally, while we do not prove it here, the statistical physics analysis predicts that $d^{-2} \log \mcV_\DP(E, U)$ concentrates on its average $\varphi_d$ in the asymptotic limit $d, p \to \infty$.
This concentration of the log-volume has 
been established in related models~\citep{talagrand2010mean},
and~\cref{claim:capacity}
should thus be understood as characterising the \emph{typical} size of the solution space.

\textbf{The parameter $q$ --}
One can associate to the volume of eq.~\eqref{eq:volume-DP} a Gibbs-type measure $\langle \cdot \rangle$, supported on $\Gamma(E, U)$, and with density proportional to $P_m(W)$. 
Geometrically, the value $q^\star(\alpha) \in [0, 1)$ achieving the infimum in eq.~\eqref{eq:free-entropy-RS} 
is such that, in the high-dimensional limit, $(W \cdot W') / (\|W\|_F \|W'\|_F) \to q^\star(\alpha)$ for $W, W' \sim \langle \cdot \rangle$ two samples under this measure.
The set of solutions $\Gamma(E, U)$ is a convex cone, 
and $(1-q)$ measures thus the ``angular width'' of this cone: as the load $\alpha$ increases towards the capacity threshold, $q$ increases as well, marking the shrinking of the set of solutions. $q$ approaches $1$ in the limit $\alpha \uparrow 1/2$, marking the critical threshold beyond which no solutions exist.

\textbf{Finite-size effects, and slow convergence to the asymptotic limit --}
Interestingly, the derivation detailed in~\cref{app:decoupled} predicts that eq.~\eqref{eq:free-entropy-RS} holds, up to $\smallO(1)$ error terms, for the values of $(d, p)$ corresponding to the embedding dimension and number of patterns, that is:
\begin{equation}\label{eq:fsize_effects}
     \varphi_d = 
    \frac{1}{2}\log(1-q^\star) + \frac{q^\star}{2(1-q^\star)} 
    + \frac{\alpha }{\log p}
    \EE_{\eta}\bigl[\log f_p(q^\star;\eta)\bigr]
    + \smallO(1).
\end{equation}
Our derivation (cf~\cref{sec:F-computation}) suggests that the last term in the right-hand-side of this equation converges to its asymptotic limit very slowly, with typical corrections of size $\tmcO(1/\log p)$ (up to multiplicative $\log \log p$ terms): in turn, this creates strong finite-size corrections to the capacity threshold with respect to its asymptotic limit $\alpha_c = 1/2$.
This mechanism explains the very slow convergence observed in~\cref{fig:figure-1}: in~\cref{fig:finite-size}, we show that these observations are compatible with a $\tmcO(1/\log p)$ convergence rate of the capacity threshold.
This slow convergence is a generic feature of the extreme-value statistics underlying the problem: each constraint requires the target score to exceed the maximum of $p-1$ Gaussian competitors, whose typical value approaches its asymptotic behaviour $\sqrt{2\log p}$ with relative corrections of order $\tmcO(1/\log p)$.

\subsection{Two-layer linear MLPs}%
\label{subsec:capacity_rank_constrained}

We finally turn to characterising the optimal storage capacity in the setting where 
the matrix is constrained to have rank $m = \kappa d$ with $\kappa \in (0, 1]$. 
Notably, this setting encompasses the full-rank case studied above when $\kappa = 1$, 
so that \cref{claim:rank,thm:rank_capacity} below provide a unified treatment of both cases. 
Recall that $\varphi_d$ was defined in eq.~\eqref{eq:def_fentropy}, and that we take $P_m$
as the law of $W = Q R^\T / \sqrt{dm}$ with $Q, R \in \bbR^{d 
\times m}$ with i.i.d.\ elements sampled from $\mcN(0, 1)$. 
Equivalently, we consider a two-layer linear network with hidden dimension $m$ as our input-to-output map in embedding space. 

\begin{claim}[Replica-symmetric prediction of the 2-layer linear MLP capacity]
\label{claim:rank}
Let $m=\kappa d$ with $\kappa\in(0,1]$ denote the rank of the weight matrix, and recall the limit $d, p \to \infty$ with $(p \log p)/d^2 \to \alpha > 0$. 
Within the so-called replica-symmetric assumption discussed below:
\begin{itemize}[leftmargin=*, topsep=0pt, itemsep=0pt]
    \item There exists $\alpha_c(\kappa) > 0$ such that for $\alpha < \alpha_c(\kappa)$, $\varphi(\alpha) = \lim_{d \to \infty} \varphi_d(\alpha) > - \infty$ is given by a formula similar to eq.~\eqref{eq:free-entropy-RS}, while $\varphi(\alpha) \to -\infty$ as $\alpha \uparrow \alpha_c(\kappa)$.
    Moreover, we have
    \begin{equation}\label{eq:alphac_kappa}
    \alpha_c(\kappa)
    = \frac{1}{2}\int_{X_{\qc}(1-\kappa)}^{2}
    \rho_{\qc}(\sigma)\,\sigma^2\,d\sigma,
    \end{equation}
    where $\rho_{\qc}(\sigma)\coloneqq(1/\pi)\sqrt{4-\sigma^2}$ is the 
    quarter-circle law, and $X_{\qc}:[0,1]\to[0,2]$ denotes its quantile 
    function, i.e.\ $X_{\qc}(p) = \inf\bigl\{\sigma : \int_0^\sigma 
    \rho_{\qc}(\sigma')\,\mathrm{d}\sigma' \geq p\bigr\}$.
    \item 
    Recall that we define $\langle \cdot \rangle$ as the measure 
    supported on $\Gamma(E, U)$, and with density proportional to $P_m(W)$. 
    For $\alpha < \alpha_c(\kappa)$, let $W \sim \langle \cdot \rangle$, and denote $\tW \coloneqq 2W / \|W\|_\op$, where $\|W\|_\op$ is the top singular value of $W$.
    Then the empirical singular value distribution of $\tW$ converges as $d \to \infty$ to a well-defined distribution $\rho[\kappa, \alpha]$, and further
    for $\alpha \uparrow \alpha_c(\kappa)$, $\rho[\kappa, \alpha] \to \rho_c[\kappa]$, given by
    \begin{equation}\label{eq:def_rhoc}
        \rho_c[\kappa](\sigma)
        \coloneqq (1-\kappa)\,\delta(\sigma)
        + \kappa\,\rho_{\qc}(\sigma)\,\mathbf{1}_{[X_{\qc}(1-\kappa),\,2]}(\sigma).
    \end{equation}
\end{itemize}
\end{claim}
Notice that $\alpha_c(1) = 1/2$, in accordance with~\cref{claim:capacity}; the variational formula for the asymptotic volume is given in eq.~\eqref{eq:variational-rank}.
Beyond the capacity, our derivation characterises the asymptotic singular value distribution of the optimal solution close to the capacity threshold, in excellent agreement with numerical experiments (black curves in~\cref{fig:spectra}).

\textbf{Rigorous capacity threshold --}
Beyond the replica-symmetric predictions of~\cref{claim:rank}, we also provide a formal proof that $\alpha_c(\kappa)$ is the capacity threshold of the decoupled problem for every $\kappa \in (0,1]$, including the non-convex case $\kappa < 1$, and that it is achievable in the original problem.
\begin{theorem}[Rank-constrained capacity threshold]
\label{thm:rank_capacity}
Let $d, p, m \to \infty$ with $m/d \to \kappa \in (0,1]$ and $(p \log p)/d^2 \to \alpha > 0$, and let $\alpha_c(\kappa)$ be given by eq.~\eqref{eq:alphac_kappa}.
With probability tending to one, if $\alpha < \alpha_c(\kappa)$ both the decoupled problem~\eqref{eq:DP} and the original problem~\eqref{eq:OP} admit a solution $W$ of rank at most $m$, while if $\alpha > \alpha_c(\kappa)$ the decoupled problem admits none.
\end{theorem}
We prove~\cref{thm:rank_capacity} in~\cref{app:rank_proof}, building on the approach of~\citet{barnfield2026sharpcapacitythresholdslinear} for the full-rank original problem.
For $\kappa = 1$, it gives $\alpha_c^{\DP} = 1/2$; together with their result for the original problem, this settles the full-rank case of~\cref{conjecture-capacity}.
\begin{corollary}[Full-rank capacity equivalence]
\label{cor:full_rank}
For $\kappa = 1$, $\alpha_c^{\OP} = \alpha_c^{\DP} = 1/2$.
\end{corollary}
\begin{remark}[Rank-constrained weights in the original problem]
\label{rem:op_rank}
For $\kappa<1$, the impossibility bound for~\eqref{eq:OP}, and thus~\cref{conjecture-capacity}, remains open. The dual-certificate argument of~\citet{barnfield2026sharpcapacitythresholdslinear} handles the coupled constraints but, being convex, is blind to the rank constraint; ours handles the rank constraint but relies on the decoupling.
\end{remark}

\textbf{Convexity and replica symmetry --}
It is important to notice that for $\kappa < 1$ the set of solutions $\Gamma_m(E, U) \coloneqq \Gamma(E, U) \cap \{W \in \bbR^{d \times d} \, : \, \rk(W) \leq m\}$ is non-convex, due to the rank constraint. 
In non-convex problems the replica-symmetric assumption generically yields only an upper bound on the capacity threshold (see e.g.~\cite{maillard2025injectivity}); here, \cref{thm:rank_capacity} shows that it is exact, while the spectral prediction of~\cref{claim:rank} remains derived \emph{within the replica-symmetric approximation}.
Non-convexity also means that gradient-based optimisation is not guaranteed to reach optimal capacity; we nevertheless find excellent agreement between theory and experiments (see e.g.~\cref{fig:spectra}).

\section{Outlook and open questions}

A natural open direction stemming from our work is a rigorous proof of~\cref{conjecture-capacity} for rank-constrained matrices. Since achievability holds for the original problem (\cref{thm:rank_capacity}), this requires an impossibility argument that handles both the coupling of the constraints and the non-convexity of the rank constraint.
More broadly, we believe that the equivalence between coupled and decoupled formulations may extend well beyond the linear associative memory studied here, and it would be interesting to understand in which problems it holds and where it breaks down. Already within our model, our derivation indicates that it should hold for extensive-rank matrices ($m = \kappa d$), for which each constraint depends on the weights only through their bulk statistics (see~\cref{subsubsec:moments_concentration}), while it clearly fails for finite-rank ones. Establishing this equivalence rigorously might require new probabilistic tools to disentangle the strong correlations between the constraints.
Another intriguing question arises from the validity of the replica-symmetric assumption in the non-convex rank-constrained problem, which~\cref{thm:rank_capacity} establishes at the level of the capacity threshold:
interestingly, the validity of the replica-symmetric assumption and the ability of gradient descent to reach the global optimum of a non-convex loss have also been observed in neural networks with similar rank constraints~\citep{maillard2024bayes,erba2025nuclear,barbier2025statistical,martin2026high}, possibly for a common reason.
Finally, extending our analysis to non-linear two-layer models or to the more realistic embedding distributions of \citet{kim2026sharpcapacityscalingspectral} is an exciting prospect.

\section*{Use of AI tools}

The replica-method computations and the resulting predictions were carried out
by the authors without AI assistance. The proof of~\cref{thm:rank_capacity}
(\cref{app:rank_proof}), including the extension of achievability to the
original problem, was developed with substantial assistance from Claude Opus 5.5
(Anthropic): a first version of the argument was generated by the model, and the
authors then worked with it to verify the proof and improve its readability and
presentation. The main text, including its content and structure, was written by
the authors; AI assistance was limited to editing. The authors have checked all
results and take full responsibility for the content of this paper.

\section*{Acknowledgments}
AM is particularly grateful to Lenka Zdeborov\'a for introducing him to this
problem as well as for stimulating discussions. AG wishes to thank Antoine
Maillard and the entire Inria ARGO team for their hospitality during the visit
in which this project was carried out. AG gratefully acknowledges funding from
Next Generation EU, in the context of the National Recovery and Resilience Plan,
Missione 4, Componente 1, Investimento 4.1 “Estensione del numero di dottorati
di ricerca e dottorati innovativi per la Pubblica Amministrazione e il
patrimonio culturale” (CUP G93C23000620003), and also gratefully acknowledges
financial support from Fondazione Zegna. SG gratefully acknowledges funding from
the European Research Council (ERC) for the project ``beyond2'', ID 101166056;
from the European Union--NextGenerationEU, in the framework of the PRIN Project
SELF-MADE (code 2022E3WYTY – CUP G53D23000780001), and from Next Generation EU,
in the context of the National Recovery and Resilience Plan, Investment PE1 --
Project FAIR ``Future Artificial Intelligence Research'' (CUP G53C22000440006).

\bibliographystyle{plainnat}
\bibliography{bibliography}

\newpage
\appendix
\etocdepthtag.toc{appendix}
\etocsettagdepth{mainmatter}{none}
\etocsettagdepth{appendix}{subsection}
\renewcommand{\contentsname}{Contents of the appendix}
\tableofcontents

\section{Numerical Experiments}

\label{sec:experimental-details}
\paragraph{Common training strategy.}
All experiments reported in this manuscript follow the same training
protocol. Inputs and outputs (in both the original and decoupled
problems) are drawn independently from a $d$-dimensional Gaussian
distribution with zero mean and identity covariance. The goal is to
satisfy the association constraints in~\cref{eq:OP,eq:DP}. To this end,
we train a model $F_W(e)$ with the cross-entropy loss on both variants
of the memory task.

\textit{Original problem.} The cross-entropy loss with a shared set of
outputs reads
\begin{equation}\label{eq:def_centropy_loss_op}
\mcL_p^{\mathrm{(OP)}}(W)
= -\frac{1}{p} \sum_{\mu=1}^{p}
\log
\frac{
\exp\!\left( u_\mu^{\top} F_W(e_\mu) \right)
}{
\sum_{\rho=1}^{p} \exp\!\left( u_\rho^{\top} F_W(e_\mu) \right)
}.
\end{equation}

\textit{Decoupled problem.} The variant with replicated sets of outputs
is trained on
\begin{equation}\label{eq:def_centropy_loss_dp}
\mcL_p^{\mathrm{(DP)}}(W)
= -\frac{1}{p} \sum_{\mu=1}^{p}
\log
\frac{
\exp\!\left( (u_\mu^{(\mu)})^{\top} F_W(e_\mu) \right)
}{
\sum_{\rho=1}^{p} \exp\!\left( (u_\rho^{(\mu)})^{\top} F_W(e_\mu) \right)
}.
\end{equation}

Both objectives are optimised with Adam at learning rate
$\eta = 10^{-2}$ for at most $512$ full-batch steps, using linear
warmup followed by cosine decay; training stops once $99.9\%$ accuracy
is reached. In the full-rank case, we use a linear model
$F_W(x) = Wx$ with $W \in \bbR^{d \times d}$ initialised with i.i.d.\
Gaussian entries of standard deviation $d^{-1}$. In the rank-constrained
case, we train a two-layer linear network $F_W(e) = QR^{\top} e$ with
$Q, R \in \bbR^{d \times m}$ and $m \leq d$, both initialised at the
same scale.

All experiments were run on a single NVIDIA A100-PCIE-40GB GPU and each
completes within a few hours. Given the small scale, reproduction is
feasible on less powerful hardware, including standard CPUs.

\subsection{Experiments in the Main Text}

\textbf{\Cref{fig:figure-1}}(Right) is obtained by following the common training protocol described above for $25$ values of the load parameter $\alpha = p \log p / d^2$, evenly spaced between $0.4$ and $1.0$. For the original problem we report results for $d =50, 100, 150, 300, 400, 500$. For the decoupled problem, simulations were limited to $d \leq 150$ due to memory constraints: in this setting, each input is associated with its own set of $p$ competing outputs, each in $\bbR^d$, requiring storage of order $O(d p^2)$. Since $p \log p$ scales as $d^2$ at fixed load, this quickly becomes prohibitive at larger dimensions. In contrast, for the original problem only a single shared set of $p$ outputs is stored, resulting in a memory footprint of order $O(d p)$. Each experiment is repeated $5$ times with independent random draws of the inputs and outputs. The lines report the empirical mean across repetitions and the standard deviation is reported only for the decoupled problem to avoid cluttering the figure.

\textbf{\Cref{fig:accuracy-equivalence}} The same experiment of \cref{fig:figure-1} is reported also for $\kappa\in\{ 0.25, 0.5\}$. We additionally provide the accuracy lines for the Hebbian ansatz in the case $\kappa=1$. The asymptotic thresholds are the ones of \cref{thm:rank_capacity}.

\textbf{\Cref{fig:spectra}} reports the distribution of the non-zero singular values of the experiments of \cref{fig:accuracy-equivalence}. Since the task is invariant under renormalisation of the weight matrix, after training we renormalise it to have $\|W\|_\op=2$.

\textbf{\Cref{fig:hist}} is obtained by training the full-rank linear models according to the common training protocol with $d = 150$ and for three values of the load parameter $\alpha$ ($0.2$, $0.5$, and $1.2$), corresponding respectively to regimes below, near, and above the capacity threshold. The choice $d = 150$ reflects the largest dimension that can be simulated for the decoupled problem; for consistency, the same dimension is used for the original problem.

After training, we extract the learned weight matrix $W^*$ and normalise it so that the non-target scores $s_{\mu\rho}=u_\rho^{\top} W^*e_\mu$ (and equivalently for the decoupled problem) have unit variance (the problem being invariant under re-normalisation of $W$, this fixes the scale); the target scores are rescaled accordingly. Target scores are shown in red and non-target scores in grey.

Note that overlap between the aggregated target and non-target distributions does not necessarily imply classification errors, since correctness depends on comparing each target score only with its own competing non-target scores rather than with the full set of scores. A representative single-instance, single-constraint histogram is shown in~\cref{fig:hist-mu}.

\subsection{Additional Experiments}

\textbf{\Cref{fig:loss}} reports the cross-entropy loss for the experiments in \cref{fig:accuracy-equivalence}, and the empirical standard deviation is reported for the decoupled problem lines.

\textbf{\Cref{fig:hist-mu}.}
This figure shows the single-constraint analogue of \cref{fig:hist}.

\textbf{\Cref{fig:finite-size}} reports the finite-size scaling of the empirical capacity threshold, for $d \in \{24, 32, 45, 64, 90, 128, 180, 256, 362, 512\}$ in the original problem and $d \leq 128$ in the decoupled one. For each $d$, the critical load $\alpha_c = p_c \log p_c / d^2$ is obtained by bisection on $p$, where $p_c$ is the largest number of patterns retrieved without errors after training (full-batch Adam, learning rate $10^{-2}$, at most $2^{15}$ epochs). Results are averaged over $3$ to $10$ realisations of the patterns, with more realisations at smaller $d$.

\begin{figure}[ht]
    \centering
\includegraphics[width=0.95\linewidth]{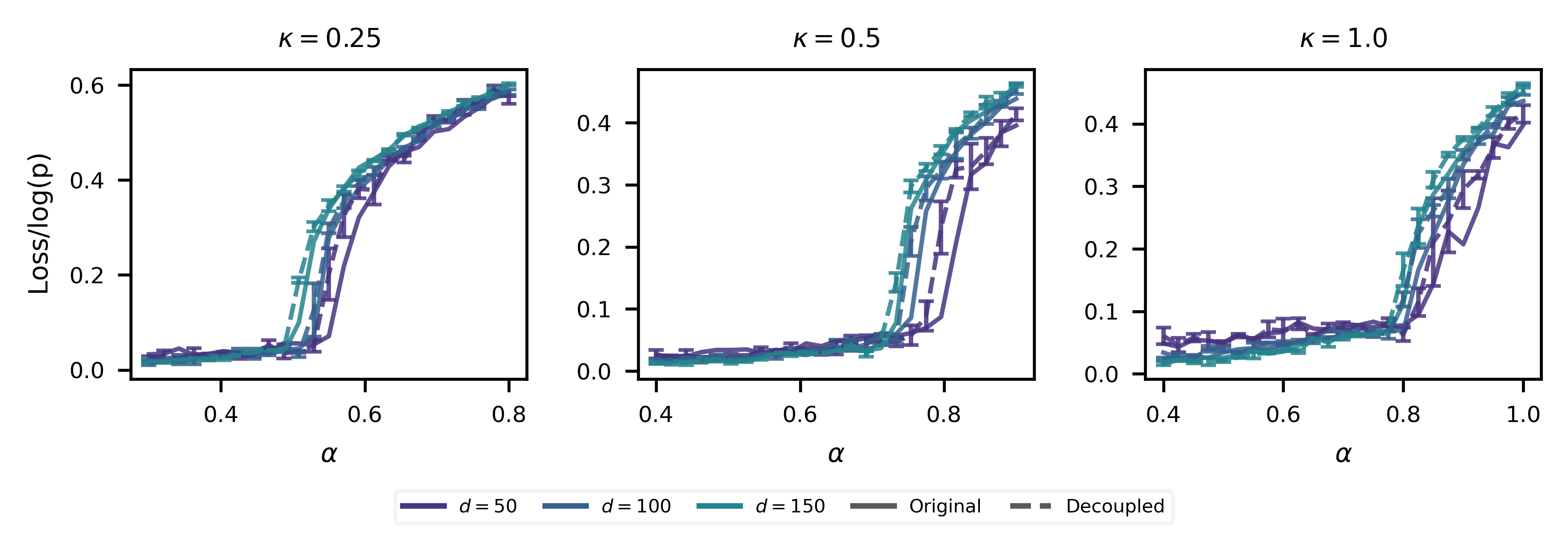}
    \caption{Cross-entropy loss for the experiment in \cref{fig:accuracy-equivalence}. Empirical standard deviation is reported for the decoupled problem lines.}
    \label{fig:loss}
\end{figure}

\begin{figure}[ht]
    \centering
    \includegraphics[width=0.7\linewidth]{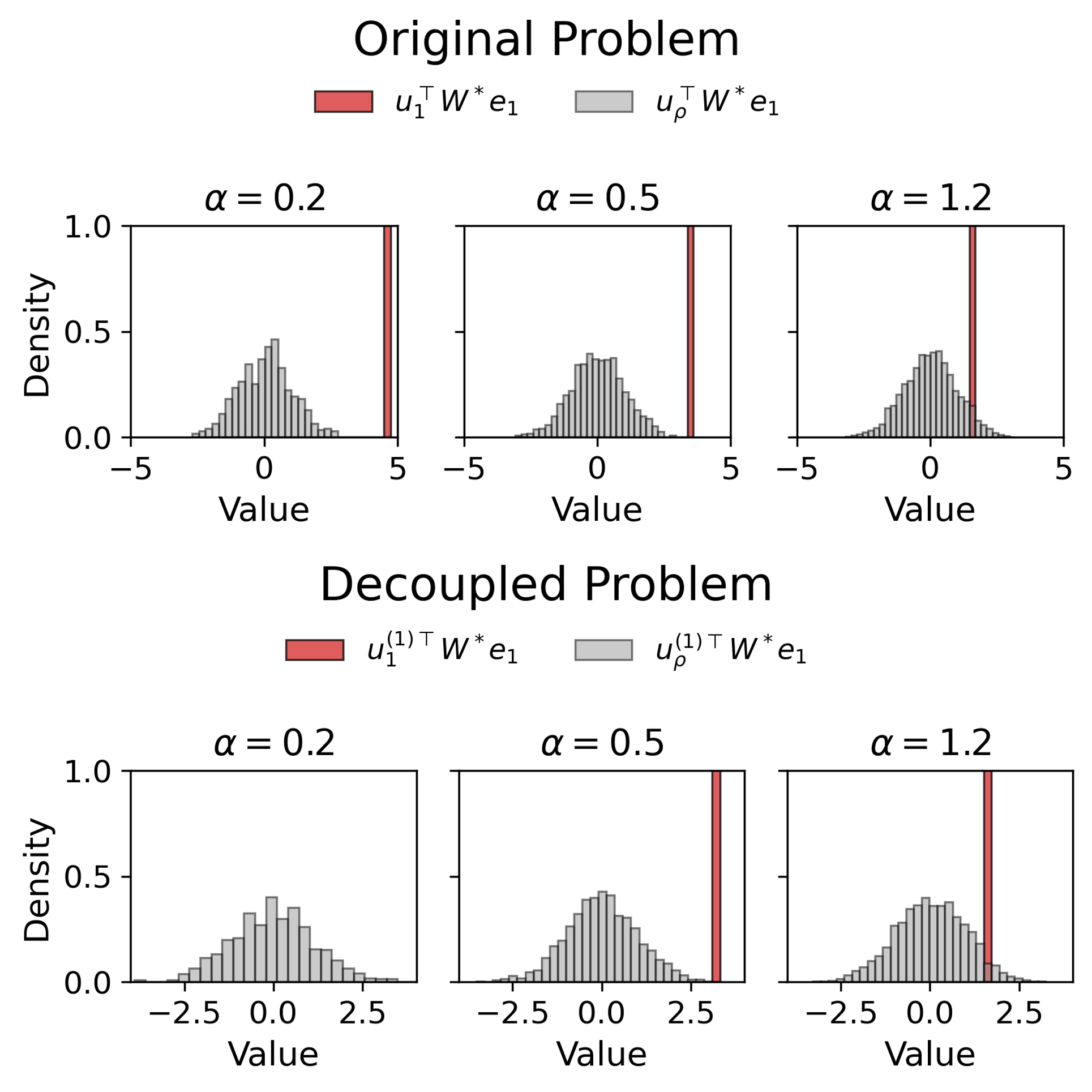}
    \caption{Single-constraint counterpart of \cref{fig:hist}. The constraint is violated whenever the target score (red) fails to exceed the maximum of its competing non-target scores (grey).}
    \label{fig:hist-mu}
\end{figure}

\clearpage
\section{Difficulties in the analysis of the original problem $(\OP)$}
\label{sec:Original-appendix}

Let us consider the case $m = d$ ($\kappa = 1$) for simplicity.
Like eq.~\eqref{eq:volume-DP}, we can associate a volume of solutions 
to the original problem (OP) defined in eq.~\eqref{eq:OP}:
\begin{align}
\label{eq:volume}
    \mcV_{\OP}(E,U) &\coloneqq \int P_d(\rd W)\,\prod_{\mu=1}^p \indi\!\left(\arg \max_{\rho \in [p]} u_\rho^\top W e_{\mu} = \mu\right), \\
    \nonumber
    &= 
    \int P_d(\rd W)\,
    \prod_{\mu=1}^p 
    \prod_{\rho (\neq \mu)}
    \Theta\!\left[
    (u_\mu - u_\rho)^\top W e_\mu
    \right],
\end{align}
where $\Theta(x) \coloneqq \indi\{x \geq 0\}$ is the Heaviside step function, 
and $P_d = \mcN(0, \Id_{d^2}/d)$.
Already, one can see that computing even the first moment of eq.~\eqref{eq:volume} (with respect to the realisation of $(E, U)$) leads to important difficulties.
Recall that the patterns $U = \{u_\mu\}_{\mu=1}^p$ and $\{e_\mu\}_{\mu=1}^p$ are drawn independently from the $d$-dimensional standard Gaussian distribution $\mcN(0, \Id_d)$. We thus have
\begin{equation}
    \EE\mcV_{\OP}(E, U)
    =
    \int P_d(\rd W)\,
    \underbrace{
    \EE_{E, U}
    \left[
    \prod_{\mu=1}^p 
    \prod_{\rho (\neq \mu)}
    \Theta\left(
    (u_\mu - u_\rho)^\top W e_\mu
    \right)
    \right]}_{\eqqcolon J(W)}.
    \label{eq:OP-annealed}
\end{equation}
While in the decoupled problem a simple permutation invariance argument yields $J(W) = e^{-p \log p}$ (see the main text), in the original problem such a simple argument does not apply.
One can notice that
conditionally on $W$ and $U$, the variables
\begin{equation*}
    z_\rho^{\mu} \coloneqq u_\rho^\top W e_\mu
\end{equation*}
are jointly Gaussian with zero mean and covariance
\begin{equation*}
    \EE_{E \mid W,U}[ z_\rho^{\mu} z_\sigma^{\nu}] = \delta_{\mu \nu} u_\rho^\top W W^\top u_\sigma \eqqcolon (\Sigma^{U})_{\rho\sigma} \delta_{\mu \nu}.
\end{equation*}
Thus 
\begin{equation}\label{eq:JW}
    J(W) = 
    \EE_U \prod_{\mu=1}^p 
    \left[
    \EE_{z^{(\mu)} \mid W,U}
    \prod_{\rho \neq \mu} \Theta\left(
    z^{\mu}_\mu - z^{\mu}_\rho\right)
    \right],
\end{equation}
where $z^{(\mu)} \sim \mcN(0,\Sigma^{U})$.
Eq.~\eqref{eq:JW} is the probability that a 
multivariate Gaussian vector with non-trivial covariance falls within a given orthant. 
This does not admit a closed-form expressions for generic covariance matrices, which prevents us from a simple computation even for the first moment of the volume of solutions in the original problem.

\textbf{The equivalence as a statement on orthant probabilities --}
We can pinpoint exactly where, at least formally, a difference between the two problems arises.
For $n$ replicas $W^1, \dots, W^n$, define the scores $z^a_{\mu\rho} \coloneqq d^{-1/2}\, u_\rho^\top W^a e_\mu$ and the event
\begin{equation*}
    \mathcal{E}_{\OP}(\{W^a\}) \coloneqq \bigcap_{a=1}^n \bigcap_{\mu=1}^p \bigcap_{\rho \neq \mu} \bigl\{ z^a_{\mu\mu} > z^a_{\mu\rho} \bigr\}.
\end{equation*}
Define $\mathcal{E}_{\DP}(\{W^a\})$ analogously, using an independent output set $\{u^{(\mu)}_\rho\}_{\rho \in [p]}$ for each input $\mu$.
The replica computations for $(\OP)$ and $(\DP)$ coincide at leading order, and in particular give the same capacity threshold, if, uniformly over the replica configurations contributing to the saddle point,
\begin{equation}
\label{eq:equivalence-orthant}
    \frac{1}{d^2} \log \bbP\bigl(\mathcal{E}_{\OP}(\{W^a\})\bigr) = \frac{1}{d^2} \log \bbP\bigl(\mathcal{E}_{\DP}(\{W^a\})\bigr) + \smallO(1).
\end{equation}
For the decoupled scores, this probability factorises over $\mu \in [p]$ and reduces to the Gaussian computation of~\cref{app:decoupled} (see the discussion of Gaussian equivalence in~\cref{subsubsec:moments_concentration}).
For the original problem, conditionally on $U$ the scores are centred Gaussian with covariance
\begin{equation*}
    \EE_{E \mid U}\bigl[z^a_{\mu\rho}\, z^b_{\nu\sigma}\bigr] = \delta_{\mu\nu}\, \frac{1}{d}\, u_\rho^\top W^a W^{b\top} u_\sigma,
\end{equation*}
so that $\bbP(\mathcal{E}_{\OP})$ is an orthant probability of a Gaussian vector whose covariance still depends on $U$ (for $n = 1$, this is eq.~\eqref{eq:JW}).
If $\|W^a\|_\op = \Theta(1)$, each entry of this covariance concentrates on its average over $U$, namely $Q_{ab}\,\delta_{\rho\sigma}$ with $Q_{ab}$ defined in eq.~\eqref{eq:def_Q}, with fluctuations of order $d^{-1/2}$.
Replacing the covariance by its average yields exactly eq.~\eqref{eq:equivalence-orthant}.
Given the evidence of~\cref{sec:evidence_conjecture}, we expect that only this average contributes to the leading order of the orthant probability; assuming eq.~\eqref{eq:equivalence-orthant} is then enough to extend our derivation, and hence~\cref{claim:capacity} and~\cref{claim:rank}, directly to the original problem.
Note that entrywise concentration does not by itself imply eq.~\eqref{eq:equivalence-orthant}, as the covariance matrix, of rank at most $d \ll p$, remains far from its average in operator norm.
Unfortunately, orthant probabilities of Gaussian vectors do not admit closed-form expressions in high dimension, and we have not been able to establish eq.~\eqref{eq:equivalence-orthant} in this way.

\section{Statistical physics analysis of the decoupled problem $(\DP)$}
\label{app:decoupled}

{\color{blue}

}

In this section we compute the asymptotic limit of $\varphi_d$ (called ``free entropy'' in the statistical physics language), as defined in eq.~\eqref{eq:def_fentropy}, using the heuristic replica method. 
We first detail the computation in the case $\kappa = 1$ ($m = d$) for a Gaussian prior $P_d = \mcN(0, \Id_{d^2}/d)$
in~\cref{sec:quenched-computation-DP}, and then extend these results to arbitrary $\kappa \leq 1$ ($m \leq d$), i.e.\ the two-layer linear model, in~\cref{app:rank_prior}).

\subsection{Derivation of \cref{claim:capacity}}
\label{sec:quenched-computation-DP}

Here we present the full derivation supporting \cref{claim:capacity} on the optimal 
storage capacity. Throughout this section, we take $P_d = \mcN(0, \Id_{d^2}/d)$. 
As detailed in the main text, we leverage the replica method from statistical physics~\citep{mezard1987spin,charbonneau2023spin}. It starts from the 
so-called ``replica trick'':
\begin{equation}\label{eq:replica_trick}
    \EE \log \mcV_\DP
    = \left[\frac{\partial}{\partial n} \log \EE[\mcV_\DP^n]\right]_{n = 0}.
\end{equation}
In the replica method, one proceeds to compute the moments $\EE[\mcV_\DP^n]$ for \emph{integer} $n \geq 0$, and then analytically expands the result to all $n > 0$, allowing to take the limit $n \to 0$ in eq.~\eqref{eq:replica_trick}. The replica method is a widely-discussed topic in the statistical physics of learning: 
 we thus refrain from giving a full introduction to its mechanisms, and refer instead the reader to~\cite{Engel_2001} for foundational applications, and to~\cite{montanari2024friendly,maillard2025injectivity} for mathematically-friendly introductions.
 We emphasize that the replica method is inherently non-rigorous, and therefore the following computations are performed at a level of rigour which is standard in the literature on the statistical physics of learning.

\subsubsection{Computation of the moments and Gaussian equivalence}
\label{subsubsec:moments_concentration}

Following the replica method,
we start from the computations of the moments of $\mcV_\DP(E, U)$.
By independence of the constraints, we have the crucial decoupling:
\begin{align}
\label{eq:replicated_VDP_1}
\nonumber
&\EE\mcV^n_{\DP}(E,U) \\
\nonumber
&= \int \prod_{a=1}^n P_d(\rd W^a) \, \prod_{\mu=1}^p 
\left[\EE_{U^{(\mu), e_\mu}}
\prod_{a=1}^n
\prod_{\rho (\neq \mu)} \Theta\left[\frac{1}{\sqrt{d}}
    (u_\mu^{(\mu)} - u_\rho^{(\mu)})^\top W^a e_\mu \right] \right], \\
&= \int \prod_{a=1}^n P_d(\rd W^a)\, 
\left[\underbrace{\EE_{U,e}
\prod_{a=1}^n
\prod_{\rho=2}^p \Theta\left[\frac{1}{\sqrt{d}}
    (u_1 - u_\rho)^\top W^a e \right]}_{\eqqcolon J(\{W^a\})} \right]^p.
\end{align}
\textbf{Gaussian universality --}
Let us introduce
\begin{equation}\label{eq:def_zarho}
    z_\rho^a \coloneqq \frac{u_\rho^\top W^{a} e}{\sqrt{d}}.
\end{equation}
Conditionally on $e$ and $(W^a)_{a=1}^n$, $z_\rho^a$ are Gaussian random variables, 
with zero mean and covariance
\begin{equation*}
    \EE_{U|\{W^a\}_{a\in[n]}, e}[z_\rho^a z_\sigma^b] = \delta_{\rho\sigma} \frac{1}{d}e^\top W^a W^{b\top}e\eqqcolon \delta_{\rho\sigma} Q_{ab}^e,
\end{equation*}
With respect to the randomness of $e$, we expect that for generic matrices\footnote{As we detail below, this holds except for matrices $W^a$ with atypically large operator norm with respect to their Euclidean norm (typically low-rank matrices), which this assumption effectively neglects to compute the asymptotics of $\mcV_\DP$.} $(W^a)_{a=1}^n$ the variables $(Q_{ab}^e)_{1 \leq a, b \leq n}$ 
concentrate as $d \to \infty$ around their expectation $Q_{ab} \coloneqq (1/d) \Tr[W^a W^b]$. 
Under this assumption, we thus reach that to leading order
\begin{align}
\label{eq:replicated_VDP_2}
\Phi_d(n)
\coloneqq \frac{1}{d^2} \log \EE [\mcV_{\DP}(E, U)^n]
&= \frac{1}{d^2} \log \int \prod_{a=1}^n P_d(\rd W^a) \, F_{n,p}(Q)^p + \smallO(1),
\end{align}
with 
\begin{equation}\label{eq:def_Q}
    Q_{ab} \coloneqq \frac{1}{d}\Tr[W^a W^{b \T}],
\end{equation}
and where we introduced for a generic symmetric positive semi-definite $Q \in \bbR^{n \times n}$:
\begin{align}
\label{eq:OrthantProbability}
    F_{n,p}(Q) &\coloneqq 
    \EE_{z\sim\mcN(0,\Id_p\otimes Q)}\prod_{a=1}^n
\prod_{\rho=2}^p \Theta\left[
   z_1^a-z_\rho^a \right], \\
    \label{eq:energetic-term}
    &= \EE_{z_1 \sim \mcN(0,Q)} \left[\left(\bbP_{z\sim\mcN(0,Q)}\big\{\forall a\in[n], z^a \leq z^a_1\big\} \right)^{p-1}\right].
\end{align}
The equivalence in high dimension of the law of the random projections $z^a_\rho$ with a Gaussian distribution with same first and second moment structure is a phenomenon known as \emph{Gaussian equivalence} in the statistics literature~\citep{goldt2022gaussian,hu2022universality,montanari2022universality,gerace2024gaussian,wen2025does}.
It has been shown for similar rank-one matrix projections (a transposition of eq.~\eqref{eq:def_zarho} with symmetric rank-one projections) in the context of ellipsoid fitting and of learning in large neural networks~\citep{maillard2023exact,maillard2024bayes,xu2025fundamental}.
At the statistical physics level of rigour, one typically postulates this equivalence as we did above (see e.g.~\cite{maillard2024fitting}). 
Nevertheless, for completeness we sketch in what follows an approach towards a rigorous establishment of this property.

\textbf{Towards a formal establishment of Gaussian equivalence --}
As we detail below, in order to formally control the concentration of $Q^e$, one must require that the matrices over which we integrate in eq.~\eqref{eq:volume-DP} have bounded operator norm $\|W\|_\op$ (i.e.\ largest singular value) as $d~\to~\infty$.
This can be done by a different choice of prior distribution $P_d$, supported on such matrices (as opposed to the Gaussian distribution, which allows arbitrarily large $\|W\|_\op$, albeit with exponentially small probability~\citep{Anderson_Guionnet_Zeitouni_2009}).
We note however that the support of these distributions avoids by definition matrices with very large operator norm (with respect to their Euclidean norm, which can be fixed by scale-invariance of the problem), i.e.\ approximately low-rank matrices.
This is a known difficulty in computing the free entropy of matrix problems, see e.g.\ the example of ellipsoid fitting~\citep{maillard2023exact}.
While we do not expect such low-rank matrices to be solutions of the memorisation problem (both from a mathematical\footnote{In particular we predict that optimal solutions close to the capacity threshold have bounded operator norm, see~\cref{claim:rank}.} and numerical viewpoint), in a mathematically rigorous approach one must rule out their existence with different methods, which we leave to future work focusing on a mathematical establishment of our predictions.

Let us thus assume that we chose such a distribution $P_d$, so that when considering $J(\{W^a\})$, all matrices $W^a$ are such that $\|W^a\|_\op \leq M$, for some fixed $M > 0$.
Recall now the Hanson-Wright inequality (see e.g.~\cite{vershynin2018high}), which we state for Gaussian vectors.
\begin{theorem}[Hanson--Wright Inequality, Gaussian vectors]\label{thm:hanson-wright}
Let $d \geq 1$ and  $X \sim \mcN(0, \Id_d)$.
There exists $C > 0$ such that for all $A \in \bbR^{d \times d}$ and $t > 0$:
\begin{equation*}
\bbP\Bigl( \bigl| X^\top A X - \EE \, X^\top A X \bigr| > t \Bigr)
\le
2 \exp\!\left[
- C \, \min\!\left(
\frac{t^2}{\|A\|_{F}^2},
\frac{t}{\|A\|_\op}
\right)
\right].
\end{equation*}
\end{theorem}
Applied to the random variables $Q_{ab}^e$, \cref{thm:hanson-wright} yields:
\begin{equation*}
\bbP_e\!\left(\left| Q_{ab}^e - Q_{ab} \right| > t \right)
\leq
2 \exp\!\left[
- C \min\!\left(
\frac{t^2d^2}{\|W^a W^{b\T}\|_F^2},\,
\frac{t d}{\|W^a W^{b\T}\|_\op}
\right)
\right].
\end{equation*}
We can bound the norms:
\begin{equation*}
\begin{dcases}
    \|W^aW^{b\top}\|_\op &\leq M^2, \\
    \|W^aW^{b\top}\|_F &\leq \sqrt{d} \|W^aW^{b\top}\|_\op \leq \sqrt{d} M^2.
\end{dcases}
\end{equation*}
Putting all together, we conclude using on top the union bound that for some constant $c = c(M) > 0$:
\begin{equation}
    \label{eq:overlap-concentration}
\bbP_e\!\left(
\max_{1 \le a,b \le n} \bigl| Q^e_{ab} - Q_{ab} \bigr| > t
\right)
\le 2 n^2 \, e^{-c d t}.
\end{equation}
Coming back to eq.~\eqref{eq:replicated_VDP_1}, we have 
\begin{align*}
    J(\{W^a\}) = \EE_e[F_{n,p}(Q^e)],
\end{align*}
and recall the definition of $F_{n,p}(Q)$ in eq.~\eqref{eq:OrthantProbability}.
As we will detail later on (see~\cref{sec:F-computation}), for the matrices $Q$ that dominate the integral in eq.~\eqref{eq:replicated_VDP_2}, $F_{n,p}$ is in the scale $\exp\{\Theta(\log p)\}$, and 
we can define a finite function $G_{n}(Q) \coloneqq \lim_{p \to \infty} \log F_{n,p}(Q) / \log p$.
This separation of scales between the concentration of $Q^e$ (exponentially fast in $d$) and the polynomial scale of $J$ implies that
\begin{equation}\label{eq:limit_J}
    \frac{1}{\log p} \log J(\{W^a\}) = \frac{1}{\log p}\log F_{n,p}(Q) + \smallO(1).
\end{equation}
Indeed, denoting $\varphi_e(Q^e)$ the law of $Q^e$ under the randomness of $Q$ and using Laplace's method, we have:
\begin{align*}
    \frac{1}{\log p} \log J(\{W^a\}) 
    &\simeq \frac{1}{\log p} \int \rd Q^e e^{(\log p) G_{n}(Q^e)} \varphi_e(Q^e), \\ 
    &\simeq \sup_{Q^e}\left[G_n(Q^e) + \lim_{p \to \infty} \frac{1}{\log p} \log \varphi_e(Q^e)\right].
\end{align*}
By eq.~\eqref{eq:overlap-concentration}, we must have since $\log p \ll d$ that 
$\lim_{p \to \infty} \log \varphi_e(Q^e) / \log p = -\infty$ except if $Q^e = Q$, which yields the result.
Eq.~\eqref{eq:limit_J}
then allows to replace $J(\{W^a\})$ by $F_{n,p}(Q)$ to leading order in $\exp\{\Theta(d^2)\}$ in eq.~\eqref{eq:replicated_VDP_1}.
While this argument can be formalised rigorously using standard large deviation theory, our derivation primarily proceeds at the theoretical physics level of rigour 
and we therefore leave a rigorous treatment to a future mathematical work.

\subsubsection{Replica-symmetric assumption}

Going back to eq.~\eqref{eq:replicated_VDP_2} and applying Laplace's method, we get: 
\begin{equation}\label{eq:Phin}
    \lim_{d \to \infty} \Phi_d(n) = \sup_{Q \in \mcS_n^+}\left[J(Q) + \alpha \lim_{p \to \infty} \frac{1}{\log p} \log F_{n,p}(Q) \right],
\end{equation}
where $\mcS_n^+$ represents the set of $n \times n$ symmetric positive semidefinite matrices, and where 
\begin{equation*}
    J(Q) \coloneqq \lim_{d \to \infty} \frac{1}{d^2} \log \int \prod_{a=1}^n \rd W^a \frac{e^{-\frac{\|W^a\|_F^2}{2}}}{(2\pi)^{d^2/2}} \, \prod_{1 \leq a \leq b \leq n} \delta(d^2 Q_{ab} - \Tr[W^a W^{b \T}]),
\end{equation*}
where we rescaled the matrices $W^a$.
Introducing Lagrange multipliers $\hQ_{ab}$, we get:
\begin{align}
\label{eq:JQ}
    J(Q) &= \extr_{\hQ \in \mcS_n}\left[\frac{1}{2} \sum_{a, b} Q_{ab} \hQ_{ab}  
    \right. \\ 
\nonumber
    &\left.+ \lim_{d \to \infty} \frac{1}{d^2} \log \int \prod_{a=1}^n \rd W^a \frac{e^{-\frac{\|W^a\|_F^2}{2}}}{(2\pi)^{d^2/2}} \, e^{-\frac{1}{2} \sum_{1 \leq a,b \leq n}\hQ_{ab} \Tr[W^a W^{b \T}] }\right],
\end{align}
where we wrote the variational problem as an extremum condition over $\hQ$.
By simple Gaussian integration, we obtain: 
\begin{align*}
& \frac{1}{d^2} \log \int \prod_{a=1}^n \rd W^a \frac{e^{-\frac{\|W^a\|_F^2}{2}}}{(2\pi)^{d^2/2}} \, e^{-\frac{1}{2} \sum_{1 \leq a,b \leq n}\hQ_{ab} \Tr[W^a W^{b \T}] }\\
 &= 
 \log \int_{\bbR} \prod_{a=1}^n \frac{\rd z^a}{\sqrt{2\pi}} e^{-\frac{(z^a)^2}{2} -\frac{1}{2} \sum_{1 \leq a,b \leq n}\hQ_{ab} z^a z^b] }, \\
&= -\frac{1}{2} \log \det[\Id_n + \hQ].
\end{align*}
And therefore we have in eq.~\eqref{eq:JQ} that the optimal $\hQ$ satisfies $Q = (\Id_n + \hQ)^{-1}$, and thus 
\begin{equation}\label{eq:JQ_2}
    J(Q) = \frac{n}{2} + \frac{1}{2} \log \det Q - \frac{\Tr(Q)}{2}.
\end{equation}
The replica-symmetric assumption~\citep{mezard1987spin}, justified here by the convexity of the problem~\citep{talagrand2010mean,barbier2022strong}, postulates the form of $(Q_{ab})$ extremizing the variational principle in eq.~\eqref{eq:Phin}. It reads:
\begin{equation}\label{eq:rs_assumption}
    Q_{aa}^{\text{RS}} = Q, \quad Q_{ab}^{\text{RS}} = q, \qquad \forall\, a \neq b,
\end{equation}
for some parameters $0 \leq q < Q$.
Under this assumption, we have from eq.~\eqref{eq:JQ_2}
\begin{equation}\label{eq:JQ_RS}
    J(Q^\RS) = \frac{n(1-Q)}{2} + \frac{n-1}{2} \log (Q-q) + \frac{1}{2} \log[Q + (n-1)q].
\end{equation}

\subsubsection{The limit $n \to 0$ and the final result}

From eq.~\eqref{eq:Phin} and the replica symmetric assumption we finally obtain: 
\begin{align}\label{eq:Phin_2}
    &\lim_{d \to \infty} \Phi_d(n) 
    \\ 
    \nonumber
    &= \sup_{0 \leq q < Q}\left[
    \frac{n(1-Q)}{2} + \frac{n-1}{2} \log (Q-q) + \frac{1}{2} \log[Q + (n-1)q]
    + \alpha \lim_{p \to \infty} \frac{1}{\log p} \log F_{n,p}(Q^\RS) \right].
\end{align}

\textbf{The ``energetic'' term --}
Recall that $F_{n, p}$ was defined in eq.~\eqref{eq:OrthantProbability}:
\begin{equation*}
    F_{n,p}(Q) \coloneqq 
    \EE_{z\sim\mcN(0,\Id_p\otimes Q)}\prod_{a=1}^n
\prod_{\rho=2}^p \Theta\left[
   z_1^a-z_\rho^a \right].
\end{equation*}
For $Q = Q^\RS$ we can simplify this drastically.
Indeed, noting that $(z_\rho^a)$ are zero-mean Gaussian variables with covariance $\EE[z_\rho^a z_\sigma^b] = [(Q-q)\delta_{ab} + q] \delta_{\rho \sigma}$, 
such a correlation structure can be realised by the representation
\begin{equation}\label{eq:scores_eta_x}
z_\rho^a
= \sqrt{q}\,\eta_\rho
+ \sqrt{Q - q}\, x_\rho^a,
\end{equation}
where $\{\eta_\rho\}_{\rho \ge 1}$, $\{x_\rho^a\}_{a, \rho}$, are independent standard Gaussian random variables. 
We then reach, with $t \coloneqq q / Q \in [0,1)$:
\begin{equation}\label{eq:Fnp_RS}
    F_{n,p}(Q^\RS) = G_{n,p}(t) \coloneqq
    \EE_{\eta}\left[\left(\EE_{x}
    \left[\prod_{\rho=2}^p \Theta\left(x_1 - x_\rho + \sqrt{\frac{t}{1-t}}(\eta_1 - \eta_\rho)\right)\right] 
    \right)^n\right].
\end{equation}
In particular we can solve the supremum over $Q$ in eq.~\eqref{eq:Phin_2}:
\begin{align}\label{eq:Phin_3}
\nonumber
     &\lim_{d \to \infty} \Phi_d(n) 
    \\ 
    \nonumber
    &= \sup_{t \in [0,1)} \sup_{Q \geq 0}\left[
    \frac{n(1-Q)}{2} + \frac{n}{2} \log Q + \frac{n-1}{2} \log(1-t) + \frac{1}{2} \log[1 + (n-1)t]
    \right. \\ 
    & \left. \hspace{2cm} + \alpha \lim_{p \to \infty} \frac{1}{\log p} \log G_{n,p}(t) \right], \\ 
    &= 
    \sup_{t \in [0,1)} \left[
   \frac{1}{2} \log[1 + (n-1)t] + \frac{n-1}{2} \log(1-t)
    + \alpha \lim_{p \to \infty} \frac{1}{\log p} \log G_{n,p}(t) \right].
\end{align}
We now take the limit $n \to 0$ in eq.~\eqref{eq:Phin_3}. 
Focusing on the last term and using eq.~\eqref{eq:Fnp_RS}, we get:
\begin{equation}
\lim_{p \to \infty} \frac{1}{\log p} \log G_{n,p}(t)
= n\, \lim_{p \to \infty} \frac{1}{\log p} \EE_{\eta}
\bigl[ \log f(t;\eta) \bigr]
+ O(n^2),
\label{eq:F-intermediate}
\end{equation}
where $\beta_t \coloneqq \sqrt{t / (1-t)}$, $\eta \sim \mcN(0, \Id_p)$, and
\begin{equation}
f_p(t;\eta)
\coloneqq \EE_{x \sim \mcN(0, \Id_p)}
\prod_{\rho = 2}^{p}
\Theta\!\left(x_1 - x_\rho + \beta_t(\eta_1 - \eta_\rho) \right) .
\label{eq:f-def}
\end{equation}
All in all, we reach\footnote{We change back notations from $t$ to $q \in [0,1]$. Notice that the transformation of the extremization over $q$ from a supremum to an infimum in the limit $n \to 0$ is also a prediction of the replica method, see e.g.~\cite{maillard2025injectivity}.}
\begin{align}\label{eq:Phi_1}
     \lim_{d \to \infty} \frac{1}{d^2} \EE \log \mcV_\DP(E, U)
    &= 
    \inf_{q \in [0,1)} \left[\frac{1}{2} \log(1-q) +
   \frac{q}{2(1-q)} 
    + \alpha \lim_{p \to \infty} \frac{\EE_\eta[\log f_{p}(q;\eta)]}{\log p} \right],
\end{align}
which ends the first part of~\cref{claim:capacity}, more precisely the derivation of eq.~\eqref{eq:free-entropy-RS}.

\subsubsection{The capacity threshold}

In order to deduce the capacity threshold from eq.~\eqref{eq:Phi_1}, we analyse when the convex space of solution shrinks, i.e.\ when $q = 1$ becomes the global infimum of the functional.
Let us denote 
\begin{equation}\label{eq:def_Gq}
    G(q) \coloneqq \lim_{p \to \infty} \frac{\EE_\eta[\log f_{p}(q;\eta)]}{\log p}.
\end{equation}
The optimal $q$ in eq.~\eqref{eq:Phi_1} satisfies 
\begin{align*}
    \frac{q}{2(1-q)^2} + \alpha G'(q) = 0.
\end{align*}
Thus we can probe the capacity threshold as the following limit 
\begin{equation}\label{eq:capacity_th}
    \alpha_c^{-1} = \lim_{q \to 1} - 2(1-q)^2 G'(q) =  \lim_{q \to 1} - 2(1-q) G(q),
\end{equation}
using L'Hôpital's Rule.
As we show in~\cref{sec:F-computation}, as $q \to 1$, we have $G(q) \sim - (1-q)^{-1}$, which finally yields the prediction $\alpha_c = 1/2$, and ends the derivation of~\cref{claim:capacity}.

\paragraph{Finite-size considerations --}
In~\cref{sec:F-computation} we show that the convergence as $p \to \infty$ of $G(q)$ in eq.~\eqref{eq:def_Gq} is very slow, with an error of order $\tmcO(1 / \log p)$ (where $\tmcO$ hides possible multiplicative $\log \log p$ terms).
In turn, via eq.~\eqref{eq:capacity_th}, we expect the finite-size corrections to the critical threshold to scale as
\begin{equation}\label{eq:alphac_finite_size}
    \alpha_c(p) = \frac{1}{2} + \tmcO\!\left(\frac{1}{\log p}\right).
\end{equation}
Since $p \log p = \alpha d^2$, we have $\log p = 2 \log d + \mcO(\log \log d)$, so that, up to the $\log \log$ factors already hidden in $\tmcO$, the correction can equivalently be written as $\tmcO(1/\log d)$.
To test eq.~\eqref{eq:alphac_finite_size}, we measure the empirical critical threshold $\alpha_c(d)$, defined as the largest load at which all patterns are retrieved without errors (\cref{fig:finite-size}).

\begin{figure}[h!]
    \centering
    \includegraphics[width=0.5\linewidth]{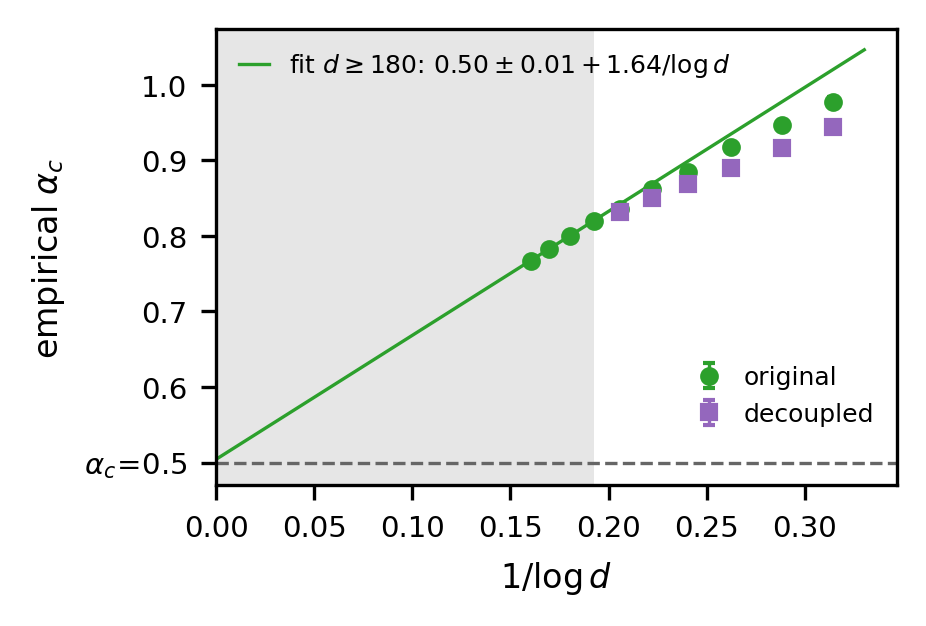}
    \caption{Finite-size scaling of the empirical critical threshold $\alpha_c$ for the original (green) and decoupled (purple) problems, plotted against $1/\log d$ (see \cref{sec:experimental-details} for the numerical procedure). As $d$ increases, the decoupled thresholds converge onto the original ones, with which they coincide within error at $d = 128$. The line is a linear fit to the original thresholds with $d \geq 180$ (shaded); its extrapolation to $d \to \infty$, $\alpha_c = 0.50 \pm 0.01$, agrees with the theoretical value $1/2$ (dashed).}
    \label{fig:finite-size}
\end{figure}

\subsection{Expansion of the energetic term}
\label{sec:F-computation}

In this appendix we analyse the function $G(q)$ of eq.~\eqref{eq:def_Gq}, close to $q = 1$. 
Recall 
\begin{equation}\label{eq:def_Gt_fp}
    \begin{dcases}
    G(t) &\coloneqq \lim_{p \to \infty} \frac{\EE_\eta[\log f_{p}(t;\eta)]}{\log p}, \\ 
    f_p(t;\eta)
    &\coloneqq \EE_{x \sim \mcN(0, \Id_p)}
    \prod_{\rho = 2}^{p}
    \Theta\!\left(x_1 - x_\rho + \beta_t(\eta_1 - \eta_\rho) \right) .
    \end{dcases}
\end{equation}
Let us first derive a very simple heuristic analysis of the behaviour of $G(t)$ as $t \uparrow 1$.  This will also give us insight into the mechanism driving this behaviour.
We will then derive mathematically sound lower and upper bounds on $G(t)$ as $t \to 1$, which will validate this heuristic.

\subsubsection{Heuristic derivation and mechanistic insight }\label{subsubsec:heuristic_G}

\textbf{Heuristic derivation --}
Notice that 
\begin{align*}
    G(t) &= \lim_{p \to \infty} \frac{1}{\log p} \EE_\eta\left[\log \bbP_{x \sim \mcN(0, \Id_p)}
    \left(x_1 + \beta_t \eta_1 \geq \max_{\rho \geq 2} [x_\rho + \beta_t \eta_\rho] \right)\right], 
\end{align*}
For $p \to \infty$ and $\beta_t \gg 1$, the random variable 
$\max_{\rho \geq 2} [x_\rho + \beta_t \eta_\rho] \simeq \beta_t \max_{\rho \geq 2} \eta_\rho \simeq \beta_t \sqrt{2 \log p}$, 
by classical results on the maxima of independent Gaussian variables.
We thus approximate in this limit
\begin{align*}
    &\EE_\eta\left[\log \bbP_{x \sim \mcN(0, \Id_p)}\left(x_1 + \beta_t \eta_1 \geq \max_{\rho \geq 2} [x_\rho + \beta_t \eta_\rho] \right)\right]\\ 
    &\simeq 
    \EE_{\eta_1}\left[\log \bbP_{x_1 \sim \mcN(0, 1)}\left(x_1 + \beta_t \eta_1 \geq \beta_t \sqrt{2 \log p} \right)\right], \\ 
    &\asimeq
    \log \bbP_{x_1 \sim \mcN(0,1)}\left(x_1 \geq \beta_t \sqrt{2 \log p} \right),
\end{align*}
where in $(\rm a)$ we use that $\eta_1 \ll \sqrt{2 \log p}$ with high probability.
The value of this last probability is simply a Gaussian tail: 
\begin{equation}\label{eq:Gaussian_tail}
    \log \bbP_{x_1 \sim \mcN(0,1)}\left(x_1 \geq \beta_t \sqrt{2 \log p} \right)
    \simeq - \beta_t^2 \log p,
\end{equation}
and this probability is dominated by the boundary value $x_1 \simeq \beta_t \sqrt{2 \log p}$.
All in all, this simple heuristic predicts (recall $\beta_t = \sqrt{t/(1-t)}$) that:
\begin{equation}\label{eq:G_t1_heuristic}
    G(t) \sim_{t \uparrow 1} -\frac{1}{1-t}.
\end{equation}
\textbf{Insights on the storing mechanism --}
While heuristic, the derivation of eq.~\eqref{eq:G_t1_heuristic} suggests an interesting insight into the mechanism with which the associations are stored.
Recall that one can interpret the variables in eq.~\eqref{eq:def_Gt_fp} as the scores of a fixed pattern association $\mu \in [p]$ (which we took to be $\mu = 1$): 
for $\rho \in [p]$, the score (diagonal if $\rho = 1$, off-diagonal if $\rho \geq 2$) is equal to $z_\rho = x_\rho + \beta_t \eta_\rho$, as was defined in eq.~\eqref{eq:scores_eta_x}. 

The heuristic derivation above (see eq.~\eqref{eq:Gaussian_tail}) suggests that, to leading order, the score $z_1 = x_1  +\beta_t \eta_1$, is dominated by $x_1 \simeq \beta_t \sqrt{2 \log p}$ and concentrates on the deterministic value $\beta_t \sqrt{2 \log p}$, with fluctuations of sub-leading order (e.g.\ of order $\mcO(\beta_t)$ via the fluctuations of $\eta_1$).
This behaviour is validated by our refined analysis in~\cref{subsubsec:math_G}, and is  strikingly different from the one exhibited by the diagonal scores of the Hebbian ansatz, which have strong fluctuations, as we discuss in~\cref{sec_app:hebbian}. In contrast, our derivation suggests that, when memorisation is possible, the diagonal scores of the optimal matrix are, to leading order, concentrated around a deterministic value as $d, p \to \infty$. This prediction is consistent with the observations of~\cref{fig:hist}, where we see that below the capacity threshold the variance of the diagonal scores is much smaller than the one of the off-diagonal scores.

\subsubsection{Mathematical justification}\label{subsubsec:math_G}

We now provide a mathematical derivation of eq.~\eqref{eq:G_t1_heuristic}, using matching upper and lower bounds.
Denoting $\lambda \coloneqq -\eta_1$ and $\xi \coloneqq x_1$, we can write $f_p$ equivalently as: 
\begin{equation}
\label{eq:f_p_eta}
f_p(t;\lambda,\eta) = \EE_{\xi}
\prod_{\rho = 2}^{p}
\Phi\!\left( \xi - \beta_t(\eta_\rho + \lambda) \right),
\end{equation}
where $\Phi$ denotes the standard normal cumulative distribution function and $\eta = (\eta_2, \cdots, \eta_p)$.

\textbf{Lower bound --}
Let us define 
\begin{equation}\label{eq:def_ap}
    a_p
    \coloneqq \max_{\rho \in\{2,\dots,p\}} (\lambda + \eta_\rho)
\end{equation}
It holds that, for all $\rho \in [p]$,
\[
\Phi\!\bigl( \xi - \beta_t(\lambda + \eta_\rho) \bigr)
\geq
\Phi\!\bigl( \xi - \beta_t a_p \bigr) .
\]
In particular,
\begin{equation}
\prod_{\rho=2}^p
\Phi\!\bigl( \xi - \beta_t(\lambda + \eta_\rho) \bigr)
\geq
\Phi\!\bigl( \xi - \beta_t a_p \bigr)^{p-1} .
\end{equation}
Therefore we obtain the lower bound
\begin{equation}
\EE_{\lambda,\eta}
\bigl[ \log f(t;\lambda,\eta) \bigr]
\geq
\EE_{a_p}
\log \!\left[
\EE_{\xi \sim \mcN(0,1)}
\Phi\!\bigl( \xi - \beta_t a_p \bigr)^{p-1}
\right] .
\end{equation}
Let us compute the inner expectation. For any $\varepsilon \in (0,1)$,
\begin{equation}
\begin{aligned}
\EE_{\xi \sim \mcN(0,1)}
\Phi\!\bigl( \xi - \beta_t a_p \bigr)^{p-1}
&\geq
\int_{\beta_t a_p + (1-\varepsilon)a_p}^{\beta_t a_p + (1+\varepsilon)a_p}
d\xi \,
\frac{e^{-\xi^2/2}}{\sqrt{2\pi}}
\, \Phi\!\bigl( \xi - \beta_t a_p \bigr)^{p-1} \\
&=
a_p \int_{1-\varepsilon}^{1+\varepsilon}
dx \,
\frac{e^{-(\beta_t + x)^2 a_p^2/2}}{\sqrt{2\pi}}
\, \Phi(x a_p)^{p-1} .
\end{aligned}
\end{equation}
Since $a_p \sim \sqrt{2 \log p}$, we use the approximation for any $x \in (1-\eps, 1 + \eps)$
\[
\Phi(x a_p)
\approx
1 - \frac{e^{-x^2 a_p^2/2}}{\sqrt{2\pi}\, x a_p}
\gtrsim
1 - p^{-x^2},
\]
which yields
\[
\EE_{\xi \sim \mcN(0,1)}
\Phi\!\bigl( \xi - \beta_t a_p \bigr)^p
\gtrsim
a_p \int_{1-\varepsilon}^{1+\varepsilon}
dx \,
p^{-(\beta_t + x)^2}
\left( 1 - \frac{1}{p^{x^2}} \right)^{p-1} .
\]
Since the integrand is positive, we can lower-bound the integral by restricting 
the integration domain to $[1, 1+\varepsilon]$. On this subinterval, 
$p^{-(\beta_t + x)^2} \geq p^{-(\beta_t + 1 + \varepsilon)^2}$. Furthermore, we 
observe that for any $r \geq 1$:
\[
\lim_{p \to \infty}
\left( 1 - \frac{1}{p^{r}} \right)^{p-1}
=
\begin{cases}
e^{-1}, & r = 1, \\[4pt]
1,      & r > 1,
\end{cases}
\]
so for $x \in [1, 1+\varepsilon]$ the factor $\left(1 - p^{-x^2}\right)^p$ 
is bounded below by $1/e$, which we absorb into the 
$\gtrsim$ symbol. Combining these two bounds yields
\[
\mathbb{E}_{\xi \sim \mathcal{N}(0,1)}\!\left[\Phi(\xi - \beta_t a_p)^p\right]
\;\gtrsim\; \varepsilon\, a_p\, p^{-(\beta_t + 1 + \varepsilon)^2}.
\]
In particular, we are interested in the behaviour as $t \to 1$, where 
$\beta_t = \sqrt{t/(1-t)}$ becomes large. Taking logarithms yields
\begin{equation}
\lim_{p \to \infty} \frac{1}{\log p}\EE_{\lambda,\eta}\bigl[ \log f(t;\lambda,\eta) \bigr]
\geq
-(\beta_t + 1 + \varepsilon)^2, 
\end{equation}
and the convergence occurs at a rate at most $O(\log\log p / \log p)$.
Since $\varepsilon > 0$ is arbitrary, we conclude
\begin{equation*}
G(t) = \lim_{p \to \infty} \frac{1}{\log p}\EE_{\lambda,\eta}\bigl[ \log f(t;\lambda,\eta) \bigr]
\geq
- \left[1 + \sqrt{\frac{t}{1 - t}}\right]^2,
\end{equation*}
In particular, as $t \to 1$, we have 
\begin{equation}\label{eq:lb_G}
    \lim_{t \to 1} (1-t) G(t) \geq -1.
\end{equation}

\paragraph{Upper Bound.}
Fix an integer $k \geq 1$, and let $z_\rho \coloneqq \lambda + \eta_\rho$. 
Since $\Phi \in (0,1)$ and is decreasing,
we can bound the product by:
\begin{equation}
\prod_{\rho=2}^p \Phi\!\bigl(\xi - \beta_t z_\rho\bigr)
\leq
\Phi\!\bigl(\xi - \beta_t a_{p-k}\bigr)^k ,
\end{equation}
where $a_{p-k}$ denotes the $(p-k)$-th order statistic of the variables $\{z_\rho\}$, i.e.
\[
a_1 \leq \cdots \leq a_{p-k} \leq a_p .
\]
This implies the upper bound
\begin{equation}\label{eq:ub_f_1}
\EE_{\lambda,\eta}
\bigl[ \log f(t;\lambda,\eta) \bigr]
\leq
\EE
\log \!\left[
\EE_{\xi \sim \mcN(0,1)}
\Phi\!\bigl(\xi - \beta_t a_{p-k}\bigr)^k
\right] .
\end{equation}
We now tackle the inner expectation. Observe that
\[
\EE_{\xi \sim \mcN(0,1)}
\Phi\!\bigl(\xi - \beta_t a_{p-k}\bigr)^k
=
\EE_{\xi}
\prod_{i=1}^k
\EE_{\xi_i \sim \mcN(0,1)}
\Id\!\bigl( 0 \leq \xi - \xi_i - \beta_t a_{p-k} \bigr).
\]
Introducing the variables $y_i = \xi - \xi_i - \beta_t a_{p-k}$, the vector $y=(y_1,\dots,y_k)$ is (conditionally on $a_{p-k}$) Gaussian with mean
$\EE[y_i] = - \beta_t a_{p-k}$,
and covariance matrix $\Sigma_{ij} = 1 + \delta_{ij}$. Therefore,
\begin{align*}
&\EE_{\xi \sim \mcN(0,1)}
\Phi\!\bigl(\xi - \beta_t a_{p-k}\bigr)^k \\ 
&=
\int_{[0,\infty)^k} d^k y \,
\frac{1}{\sqrt{\det(2\pi \Sigma)}}
\exp\!\left[
-\frac{1}{2}
\bigl(y + \mathbf{1}\beta_t a_{p-k}\bigr)^\top
\Sigma^{-1}
\bigl(y + \mathbf{1}\beta_t a_{p-k}\bigr)
\right] , \\
&=
\frac{e^{-\frac{1}{2}\beta_t^2 a_{p-k}^2 \mathbf{1}^\top \Sigma^{-1}\mathbf{1}}}
{\sqrt{\det(2\pi \Sigma)}}
\int_{[0,\infty)^k} d^k y \,
\exp\!\left[
-\frac{1}{2} y^\top \Sigma^{-1} y
- \beta_t a_{p-k} \mathbf{1}^\top \Sigma^{-1} y
\right] .
\end{align*}
Rescaling $y' = \beta_t a_{p-k} y$, we obtain (we use that $a_{p-k} > 0$ with overwhelming probability for any fixed $k$ as $p \to \infty$)
\begin{align*}
    &\EE_{\xi \sim \mcN(0,1)}
\Phi\!\bigl(\xi - \beta_t a_{p-k}\bigr)^k \\ 
&=
\frac{e^{-\frac{1}{2}\beta_t^2 a_{p-k}^2 \mathbf{1}^\top \Sigma^{-1}\mathbf{1}}}
{(\beta_t a_{p-k})^k \sqrt{\det(2\pi \Sigma)}}
\int_{[0,\infty)^k} d^k y \,
\exp\!\left[
-\frac{1}{2(\beta_t a_{p-k})^2} y^\top \Sigma^{-1} y
- \mathbf{1}^\top \Sigma^{-1} y
\right] .
\end{align*}
In the limit $p \to \infty$, we have $a_{p-k} \sim \sqrt{2 \log p}$ for any fixed $k \geq 1$.
We thus reach in this limit
\begin{align*}
    \int_{[0,\infty)^k} d^k y \,
\exp\!\left[
-\frac{1}{2(\beta_t a_{p-k})^2} y^\top \Sigma^{-1} y
- \mathbf{1}^\top \Sigma^{-1} y
\right] \simeq \int_{[0,\infty)^k} d^k y \,
\exp\!\left[- \mathbf{1}^\top \Sigma^{-1} y
\right]\eqqcolon C_k,
\end{align*}
a constant only depending on $k$.
Using the Sherman--Morrison formula, we get
\[
\Sigma^{-1} = I_k - \frac{1}{k+1}\mathbf{1}\mathbf{1}^\top ,
\]
so that
\[
\mathbf{1}^\top \Sigma^{-1} \mathbf{1} = \frac{k}{k+1}.
\]
Therefore, we reach:
\[
\EE_{\xi \sim \mcN(0,1)}
\Phi\!\bigl(\xi - \beta_t a_{p-k}\bigr)^k
\simeq
\frac{e^{-\frac{k}{2(k+1)} \beta_t^2 a_{p-k}^2}}{(\beta_t a_{p-k})^k} \, \tC_k ,
\]
where $\tC_k$ is a finite constant depending only on $k$.
Since $a_{p-k} \sim \sqrt{2\log p}$ for finite $k \geq 1$, we reach by combining this with eq.~\eqref{eq:ub_f_1}:
\begin{align}\label{eq:ub_G_k}
    G(t) = \lim_{p \to \infty} 
    \frac{1}{\log p} \EE_{\lambda,\eta}
\bigl[ \log f(t;\lambda,\eta) \bigr] \leq - \frac{k}{k+1} \frac{t}{1-t}.
\end{align}
Similarly to the upper bound, the corrections of order $\mcO(\log \log p / \log p)$.
Since eq.~\eqref{eq:ub_G_k} holds for any $k \geq 1$, we get by combining it with eq.~\eqref{eq:lb_G}:
\begin{equation}\label{eq:behavior_G}
    \lim_{t \to 1} (1-t) G(t) = -1.
\end{equation}
\textbf{Slow convergence in $p$ --}
Our derivation suggests that the typical corrections to $G(t)$ for finite $p$ are of order $\mcO(\log \log p / \log p)$. This is validated by the slow convergence that we observe for the capacity threshold, which is compatible with this convergence rate~\cref{fig:finite-size}.

Finally, we emphasize that our approach essentially made mathematically sound the intuition laid out in~\cref{subsubsec:heuristic_G}, by showing that the inner expectation is dominated by values of $\xi$ tightly concentrated around $\beta_t \sqrt{2\log p}$.

\subsection{Rank-constrained weights: derivation of \cref{claim:rank}}
\label{app:rank_prior}

We now consider weight matrices constrained to have rank at most $m = \kappa d$ 
for some $\kappa \in (0, 1]$. We assume that the prior $P_m(\rd W)$
is given by $W \deq QR^\T / \sqrt{dm}$ for $Q, R \in \bbR^{d \times m}$ with i.i.d.\ $\mcN(0, 1)$ elements. Notice in particular that $\EE[\Tr[WW^\T]] = d$, and that the law of $W$ is invariant by left and right orthogonal transformations: in particular
if $\sigma_1, \cdots, \sigma_m$ are the non-zero singular values of $W$, $P_m$ only depends on $W$ through the empirical distribution $\hat{\rho} \coloneqq (1/m) \sum_{i=1}^m \delta_{\sigma_i}$. 
While the precise form of $P_m$ chosen here does not bear importance in the derivation of~\cref{claim:rank}, we will use that the law of $\hat{\rho}$ satisfies, under $P_m$, a large deviations principle in the scale $\mcO(d^2)$, i.e.\ that 
\begin{equation}\label{eq:ldp}
  P_m(\hat{\rho} \simeq \mu) \sim \exp\{-d^2 J_\kappa(\rho)\},
\end{equation}
for some function $J_\kappa$
: this is widely expected to hold for a large class of random matrix ensembles~\citep{Anderson_Guionnet_Zeitouni_2009,Potters_Bouchaud_2020}, and in our heuristic computation we will assume $P_m$ satisfies such a property. 

\subsubsection{The replica method}

The computation of the moments of $\mcV_\DP(E, U)$ in this setting is fairly similar to the one of~\cref{app:decoupled}. In particular, a completely similar Gaussian equivalence phenomenon can be postulated in the high-dimensional limit (see the discussion in~\cref{subsubsec:moments_concentration}), and leads to:
\begin{equation}
\label{eq:replicated_VDP_rk}
\frac{1}{d^2} \log \EE \mcV_\DP(E, U) 
= \frac{1}{d^2} \log \int \prod_{a=1}^n P_m(\rd W^a) \, F_{n,p}(Q)^p + \smallO(1),
\end{equation}
again with 
\begin{equation}
    Q_{ab} \coloneqq \frac{1}{d}\Tr[W^a W^{b \T}].
\end{equation}

\textbf{Non-convexity and the replica-symmetric ansatz --}
Note that the support of the 
rank-constrained prior $P_m(W)$ is non-convex for $\kappa < 1$.
The convexity of the problem was important in justifying the validity of the replica-symmetric assumption.
Here, we nonetheless work under the same assumption: it is generically known that it yields an upper bound on the limiting volume (see e.g.~\citep{maillard2025injectivity}), and its predictions are in close agreement with numerical experiments, as detailed in the main text of this paper.
Likewise, when the rank constraint is enforced through the two-layer 
parametrisation $W = QR^\top$ (as in our experiments), the loss becomes 
non-convex in $(Q, R)$, so gradient-based optimisation is not guaranteed to find 
the global optimum.

One applies then the replica method to eq.~\eqref{eq:replicated_VDP_rk}, in a very similar way to~\cref{app:decoupled}.
The only difference arises from the ``entropic'' term
\begin{equation*}
    J(Q) \coloneqq \lim_{d \to \infty} \frac{1}{d^2} \log \int \prod_{a=1}^n P_m(\rd W^a) \, \prod_{1 \leq a \leq b \leq n} \delta(d^2 Q_{ab} - d \Tr[W^a W^{b \T}]),
\end{equation*}
which under the replica-symmetric assumption reads: 
\begin{align*}
    J(Q,q) &= \extr_{\hQ \in \mcS_n}\left[\frac{1}{2} \sum_{a, b} Q_{ab} \hQ_{ab}  
    + \lim_{d \to \infty} \frac{1}{d^2} \log \int \prod_{a=1}^n P_m(\rd W^a) \, e^{-\frac{d}{2} \sum_{1 \leq a,b \leq n}\hQ_{ab} \Tr[W^a W^{b \T}] }\right], \\ 
    &= 
     \extr_{\hQ, \hq}\left[\frac{n Q \hQ}{2}  - \frac{n(n-1)}{2} q \hq \right. \\
     &
     \left. \qquad + \lim_{d \to \infty} \frac{1}{d^2} \log \int \prod_{a=1}^n P_m(\rd W^a) \, e^{-d\frac{(\hQ+\hq)}{2} \sum_{a=1}^n \|W^a\|_F^2 + d\frac{\hq}{2}\left\|\sum_{a=1}^n W^a\right\|_F^2 }\right],
\end{align*}
where we used the replica-symmetric assumption $\hQ_{ab} = -\hq$, $\hQ_{aa} = \hQ$ for $a \neq b$.
We use the identity:
\[
e^{d\frac{\hat{q}}{2}\left\|\sum_a W^a\right\|_F^2} = 
\int \frac{dZ}{(2\pi/d)^{d^2/2}}\, e^{-\frac{d}{2}\|Z\|_F^2 
+ d \sqrt{\hat{q}}\,\mathrm{Tr}\!\left(Z^\top \sum_a W^a\right)},
\]
This decouples the replicas: each copy of $W^a$ now 
interacts independently with the shared auxiliary field $Z$, and the entropic 
term takes the form
\begin{equation}
J(Q, q) = \extr_{\hQ, \hq}\left[\frac{n Q \hQ}{2}  - \frac{n(n-1)}{2} q \hq  + \lim_{d \to \infty} \frac{1}{d^2} \log \,\EE_{Z \sim \mathrm{Gin}(d)} [I(Z)]^n \right],
\end{equation}
where
$Z \sim \mathrm{Gin}(d)$ denotes a $d \times d$ matrix with i.i.d.\ $\mcN(0, 1/d)$ entries, and
\begin{equation}\label{eq:def_I_rank}
I(Z) \coloneqq \int P_m(\rd W)\, 
e^{-\frac{d}{2}(\hat{Q} + \hat{q})\|W\|_F^2/2 + d \sqrt{\hat{q}}\,\mathrm{Tr}(Z^\top W)}.
\end{equation}
All in all we reach: 
\begin{align}
    \label{eq:replicated_VDP_rk_2}
\lim_{d \to \infty} \frac{1}{d^2} \log \EE \mcV_\DP(E, U)^n 
    &= \extr_{Q, q, \hQ, \hq} \left[\
    \frac{n Q \hQ}{2}  - \frac{n(n-1)}{2} q \hq  \right. \\ 
    \nonumber
    &\left. + \lim_{d \to \infty} \frac{1}{d^2} \log \,\EE_{Z \sim \mathrm{Gin}(d)} [I(Z)]^n + \alpha \lim_{p \to \infty} \frac{\log G_{n,p}(q/Q)}{\log p}\right],
\end{align}
where recall that we defined $G_{n,p}$ in eq.~\eqref{eq:Fnp_RS}.
Taking then the limit $n \downarrow 0$ in the replica method, we get 
a similar expression with respect to eq.~\eqref{eq:Phi_1}, with the first entropic term modified:
\begin{align}\label{eq:RS_formula_rank}
     \lim_{d \to \infty} \frac{1}{d^2} \EE \log \mcV_\DP(E, U)
    = 
    \extr_{Q, q, \hQ, \hq} &\left[
    \frac{Q \hQ}{2} + \frac{q \hq}{2} 
    + \lim_{d \to \infty} \frac{1}{d^2} \,\EE_{Z \sim \mathrm{Gin}(d)} \log[I(Z)]\right.
    \\ 
    \nonumber
    &\left.+ \alpha \lim_{p \to \infty} \frac{\EE_\eta[\log f_{p}(q/Q;\eta)]}{\log p} \right],
\end{align}
We now focus on analysing this entropic term. 
We note that this analysis is very close to the ones performed in~\cite{maillard2024fitting,erba2025bilinear}.

\paragraph{SVD reduction--}
The auxiliary field $Z \sim \mathrm{Gin}(d)$ also enjoys bi-orthogonal invariance, 
sharing the same rotational symmetry as $P_m(W)$. This allows us to integrate out 
the orthogonal degrees of freedom and reduce the problem to one involving only 
the singular values of $W$ and $Z$. We decompose $W$ via its compact SVD: 
$W = U\Sigma V^\top$, where $U, V \in \mathrm{St}(d,m)$ and 
$\Sigma = \mathrm{diag}(\sigma_1,\ldots,\sigma_m)$ with $\sigma_i > 0$. The change of variables 
becomes \citep{Anderson_Guionnet_Zeitouni_2009}:
\begin{equation}
\rd W = c_{d,m}\,|\Delta(s^2)|\,d\mu(U)\,d\mu(V)\prod_{i=1}^m \sigma_i^{d-m}\,d\sigma_i,
\end{equation}
where $d\mu(U)$, $d\mu(V)$ are the Haar measures on $\mathrm{St}(d,m)$, and 
$|\Delta(s^2)| = \prod_{1\leq i<j\leq m}|\sigma_i^2 - \sigma_j^2|$ is the Vandermonde 
determinant. Since $P(W)$ depends only on the singular values and 
$\|W\|_F^2 = \sum_i \sigma_i^2$, the integral $I(Z)$ of eq.~\eqref{eq:def_I_rank} becomes
\begin{equation}
\begin{aligned}
I(Z) =\; & c_{d,m} \int_{\mathbb{R}_+^m} \prod_{i=1}^m d\sigma_i\,\sigma_i^{d-m}\,
|\Delta(\sigma^2)|\,P_m(\Sigma)\,e^{-\frac{d}{2}(\hat{Q}+\hat{q})\sum_i \sigma_i^2} \\
&\times \int_{\mathrm{St}(d,m)} d\mu(U)\int_{\mathrm{St}(d,m)} d\mu(V)\,
\exp\!\Big\{d\sqrt{\hat{q}}\;\mathrm{Tr}\!\big(U^\top ZV\Sigma\big)\Big\}.
\end{aligned}
\end{equation}
We introduce the empirical singular value density
\[
\hat{\rho}(\sigma) = \frac{1}{m}\sum_{i=1}^m \delta(\sigma - \sigma_i),
\]
and change integration variables from $(\sigma_1,\ldots,\sigma_m)$ to a generic 
probability density $\rho$. All terms in the integrand depend on the singular 
values only through $\hat{\rho}$:
\begin{align}
\begin{dcases}
\frac{1}{2}\sum_{i \neq j} \log|\sigma_i^2 - \sigma_j^2| &= 
\frac{m^2}{2}\dashint \hrho(\sigma)\hrho(\sigma')\log|\sigma^2 - \sigma'^2|\, 
d\sigma\, d\sigma', \\
\sum_{i=1}^m \log \sigma_i &= m \int \hrho(\sigma)\, \log \sigma\, d\sigma, \\
\sum_{i=1}^m \sigma_i^2 &= m \int \hrho(\sigma)\, \sigma^2\, d\sigma.
\end{dcases}
\end{align}
Finally, recall the large deviations assumption stated at the beginning of this section in eq.~\eqref{eq:ldp}.
The change of variables from the $m$-dimensional empirical density to a generic density $\rho$ introduces a Jacobian of order $\exp(O(m)) = \exp(O(d))$, which is subleading with respect to the $\exp(O(d^2))$ terms and can therefore be neglected\footnote{See~\cite{livan2018introduction} for a general justification of this point, and a a similar discussion in~\cite{maillard2024fitting}.}. Using $m = \kappa d$ and performing Laplace's 
method in the space of
probability distributions, we finally reach that 
\begin{equation}
\begin{aligned}
\lim_{d\to \infty} \frac{1}{d^2} \EE_Z \log I(Z) = \sup_{\rho \in \mcM_1^+(\bbR_+)} \Bigg[
& \frac{\kappa^2}{2}\dashint \rho(\sigma)\rho(\sigma')\log|\sigma^2 - \sigma'^2|\, d\sigma\, d\sigma' \\
& + \kappa(1-\kappa) \dashint \rho(\sigma) \log \sigma \, d\sigma \\
& - \frac{\kappa(\hQ+ \hq)}{2} \int \rho(\sigma)\, \sigma^2\, d\sigma \\
& - J_\kappa(\rho)
+ \mathcal{I}_{\mathrm{HCIZ}}(\hat{q};\rho, \rho_z)
\Bigg].
\end{aligned}
\end{equation}
where $\mcM_1^+(\bbR_+)$ denotes probability distributions supported on $\bbR_+$, and we defined the rectangular \emph{HCIZ integral} (see~\cite{Anderson_Guionnet_Zeitouni_2009})
\begin{equation}\label{eq:def_HCIZ}
\mathcal{I}_{\mathrm{HCIZ}}(\hat{q};\rho, \rho_Z) \coloneqq\lim_{d \to \infty} d^{-2}\log \int_{\mathrm{St}(d,m)^2} d\mu(U) d\mu(V)\, \exp\!\Big\{d\sqrt{\hat{q}}\;\mathrm{Tr}\!\big(U^\top Z V \Sigma\big)\Big\},
\end{equation}
which only depends on the asymptotic empirical distribution $\rho_Z$ of the singular values of $Z$. 
The later is known to be the \emph{quarter-circle law}:
\begin{equation}
\rho_z(\sigma) = \rho_{\mathrm{q.c.}}(\sigma) \coloneqq \frac{1}{\pi}\sqrt{4 - \sigma^2}, \qquad \sigma \in [0, 2].
\end{equation}
Note that since $W$ has rank at most $m = \kappa d$, only $m$ singular values are non-zero; the overall singular value distribution of $W$ is therefore
\begin{equation}
\label{eq:W-singval-density}
\rho_W(\sigma) = (1 - \kappa)\,\delta(\sigma) + \kappa\,\rho(\sigma),
\end{equation}
where $\rho$ is the density of the non-zero singular values obtained from the saddle-point equation above.
Overall, we reach the final prediction of the replica method (rescaling $t = q/Q \in [0,1)$):
\begin{align}
\label{eq:variational-rank}
&\lim_{d\to\infty} \frac{1}{d^2}\EE\log \mcV_{\text{DP}}(E, U) \\ 
\nonumber
&= \underset{\substack{Q,\, t,\, \hat{Q},\, \hat{q}}}{\mathrm{extr}} \sup_{\rho \in \mathcal{M}_1^+(\mathbb{R}_+)} \;\left\{
\begin{aligned}
&\frac{Q\hat{Q}}{2} + \frac{1}{2}tQ\hat{q} \\
&+ \frac{\kappa^2}{2}\dashint \rho(\sigma)\rho(\sigma')\log|\sigma^2 - \sigma'^2|\, d\sigma\, d\sigma'\\
&+ \kappa(1-\kappa) \dashint \rho(\sigma)\, \log \sigma\, d\sigma  - J_\kappa(\rho)\\
&- \frac{\kappa}{2}\!\left(\hat{Q} + \hat{q}\right)
\int \rho(\sigma)\, \sigma^2\, d\sigma
+ \mathcal{I}_{\mathrm{HCIZ}}(\hat{q}; \rho, \rho_{\mathrm{q.c.}}) \\
& + \alpha G(t)
\end{aligned}
\right\}.
\end{align}
where recall that we defined 
\begin{equation*}
    G(t) \coloneqq \lim_{p \to \infty} \frac{\EE_\eta[\log f_{p}(t;\eta)]}{\log p}.
\end{equation*}

\subsubsection{Computation of the capacity threshold}

With respect to the non-rank-constrained problem (cf eq.~\eqref{eq:Phi_1}), the difference resides in the ``entropic'' contribution to the formula. While it is an involved term in eq.~\eqref{eq:variational-rank}, we can evaluate it close to the capacity threshold thanks to explicit expansions of the HCIZ integral. Interestingly, this also gives us access to the value of the maximising density $\rho$, corresponding to the asymptotic singular value density of solutions close to the capacity transition.
This derivation is in essence close to ones performed in~\cite{maillard2024fitting,erba2025bilinear}, and we refer to these works for more details on some technical points.
Let us first write the saddle point equations associated to eq.~\eqref{eq:variational-rank}: 
\begin{subequations}
\label{eq:full_saddles_rank}
\begin{empheq}[left = \empheqlbrace]{align}
    & \hat{q} = -\frac{2\alpha}{Q G'(t)}  \\
    & \hat{Q} + t \hq = 0 \\
    & Q = \kappa \int \rho(\sigma) \sigma^2 d\sigma  \\
    & Q(1-t) = 2\frac{\partial}{\partial \hq} I_\HCIZ(\hq;\rho, \rho_{\rm q.c.})  \\
    & \int \rho(\sigma) d\sigma = 1  \\
    & \begin{aligned} 
        &\kappa^2 \dashint \rho(\sigma') \log|\sigma^2 - \sigma'^2| d\sigma' + \kappa(1-\kappa) \log \sigma - \frac{\kappa(\hQ+\hq)}{2} \sigma^2 \\
        &\quad + \frac{\delta}{\delta\rho(\sigma)}\left[I_\HCIZ(\hq;\rho, \rho_{\rm q.c.}) - J_\kappa(\rho)\right] - \lambda_0  = 0
      \end{aligned}
\end{empheq}
\end{subequations}
where the last equation holds for all $\sigma \in \mathrm{supp}(\rho)$, with $\lambda_0$ a Lagrange multiplier enforcing the normalisation of the probability density.

\paragraph{Dilute limit ($\hat{q}\to \infty$)}
As we approach the capacity threshold, we have $t \to 1$. As we have seen in~\cref{sec:F-computation}, in this limit $G'(t) \sim -(1-t)^{-2}$, 
and thus $\hq \sim \hq_0 / (1-t)^2 \to \infty$, with $\hq_0 = 2 \alpha / Q$.
We can exploit this, to compute the leading order of the HCIZ integral when the argument becomes very large. 
Indeed, in this case it is dominated by the maximiser of the integrand in eq.~\eqref{eq:def_HCIZ}, see e.g.~\cite{bun2014instanton}:
\begin{equation}
\mathcal{I}_{\mathrm{HCIZ}}(\hat{q};\rho, \rho_{\mathrm{q.c.}}) \;\sim_{\hq \to \infty}\; d^{-1}\sqrt{\hat{q}}\;\max_{U,V \in \mathrm{St}(d,m)} \mathrm{Tr}\!\big(U^\top Z\, V\, \Sigma\big).
\end{equation}
The maximum is achieved when $U$ and $V$ align the $m$ largest singular values of $Z$ with the $m$ non-zero singular values of $\Sigma$:
\begin{equation}
\max_{U,V \in \mathrm{St}(d,m)} \mathrm{Tr}\!\big(U^\top Z\, V\, \Sigma\big) = \sum_{i=1}^m z_{(i)}^{\downarrow}\;\sigma_{(i)}^{\downarrow} = \sum_{i=1}^d z_{(i)}^{\downarrow}\;\sigma_{W,(i)}^{\downarrow},
\end{equation}
where $z_{(1)}^{\downarrow} \geq \cdots \geq z_{(m)}^{\downarrow}$ are the $m$ largest singular values of $Z$, and $\sigma_{W,(i)}^{\downarrow}$ denotes the $i$-th largest singular value of $W$ (with $\sigma_{W,(i)}^{\downarrow} = 0$ for $i > m$). 
As $d \to \infty$, the leading order of this sum is simply an integral given by the quantile functions of $\rho_W$ and $\rho_{\mathrm{q.c.}}$. To sum up:
\begin{equation}\label{eq:HCIZ-dilute}
\mathcal{I}_{\mathrm{HCIZ}}(\hat{q};\rho, \rho_{\mathrm{q.c.}}) \;\sim_{q\to\infty}\; \sqrt{\hat{q}}\;\int_{1-\kappa}^1 X_{\mathrm{q.c.}}(p)\; X_W(p)\; dp\,
\end{equation}
where $X_W$ is the quantile function of $\rho_W = (1-\kappa)\delta_0 + \kappa \rho$. 
Notice that since the zero singular values of $W$ do not contribute, the integral is restricted to $p \in [1-\kappa, 1]$.

\paragraph{Saddle-point equations close to the SAT/UNSAT transition.}

We can simplify eq.~\eqref{eq:full_saddles_rank} close to the transition, using 
that $G'(t) \sim -(1-t)^2$ as $t \uparrow 1$, and using eq.~\eqref{eq:HCIZ-dilute}. 
Notice that the terms in the last equation of eq.~\eqref{eq:full_saddles_rank} that do not depend on $\hq, \hQ$ become subdominant\footnote{It is at this point in the derivation that the specific choice of the prior $P_m$ becomes irrelevant when considering the capacity threshold.} as $t \to 1$, and we get to leading order:
\begin{subequations}
\label{eq:full_system}
\begin{empheq}[left = \empheqlbrace]{align}
    & \hat{q} \sim \frac{2\alpha}{Q(1-t)^2} \label{eq:q_hat} \\
    & \hat{Q} + \hq = \hQ + t\hq + (1 - t) \hq = (1-t)\hq \sim \frac{2\alpha}{Q(1-t)} \label{eq:Q_hat} \\
    & Q = \kappa \int \rho(\sigma) \sigma^2 d\sigma \label{eq:Q_int} \\
    & Q(1-t) \sim \frac{1}{\sqrt{\hq}} \int_{1-\kappa}^1 X_{\mathrm{q.c.}}(p) X_W(p) dp \label{eq:diff_q} \\
    & \int \rho(\sigma) d\sigma = 1 \label{eq:rho_norm} \\
    & \begin{aligned} 
        &\kappa\frac{\hQ + \hq}{2}\sigma^2 = \sqrt{\hq} \frac{\delta}{\delta\rho(\sigma)}\left[\int_{1-\kappa}^1 X_{\mathrm{q.c.}}(p) X_W(p) dp\right]
      \end{aligned} \label{eq:long_functional}
\end{empheq}
\end{subequations}
Letting 
$\hat{q}_0 \coloneqq \frac{2\alpha}{Q}$, we reach from eq.~\eqref{eq:diff_q}.
\begin{equation*}
\sqrt{\hat{q}_0}\, Q = \int_{1-\kappa}^1 dp\; X_{\mathrm{q.c.}}(p)\, X_W(p).
\end{equation*}
Recall that $X_W(p)$ is the quantile function of the singular value distribution of $W$, i.e.\ the inverse of the cumulative distribution corresponding to the density in \cref{eq:W-singval-density}:
\begin{equation}
F_W(x) = \int_{-\infty}^{x} d\sigma'\, \Big((1-\kappa)\,\delta(\sigma') + \kappa\, \rho(\sigma')\Big),
\end{equation}
For $p \in [1-\kappa, 1]$, differentiating the identity $F_W(X_W(p)) = p$ with respect to $\rho(\sigma)$ gives:
\begin{equation}
\frac{\delta F_W(X_W(p))}{\delta \rho(\sigma)}
= \kappa\, \rho(X_W(p))\, \frac{\delta X_W(p)}{\delta \rho(\sigma)}
+ \kappa\, \mathbf{1}[\sigma \leq X_W(p)] = 0.
\end{equation}
This gives access to the functional derivative of the quantile function:
\begin{equation}
\frac{\delta X_W(p)}{\delta \rho(\sigma)} = -\frac{\mathbf{1}[\sigma \leq X_W(p)]}{\rho(X_W(p))}.
\end{equation}
Substituting into eq.~\eqref{eq:long_functional} and performing the change of variables $p = F_W(u)$ for $p \in (1-\kappa, 1)$, we obtain to leading order as $t \to 1$:
\begin{equation}
\frac{\hat{q}_0}{2}\,\sigma^2
+ \sqrt{\hat{q}_0}\int_{\sigma}^{\max \supp \rho} X_{\mathrm{q.c.}}(F_W(u))\, du = 0,
\end{equation}
which holds for all $\sigma \in \supp(\rho$).
Differentiating with respect to $\sigma$ yields
\begin{equation}
\sqrt{\hat{q}_0}\,\sigma = X_{\mathrm{q.c.}}\!\Big((1-\kappa) + \kappa\, F_\rho(\sigma)\Big),
\end{equation}
from which we read off the cumulative distribution $F_\rho(\sigma)$ of non-zero singular values:
\begin{equation}\label{eq:dist_svalues}
F_\rho(\sigma) = \frac{1}{\kappa}\left(F_{\mathrm{q.c.}}\!\left(\sqrt{\hat{q}_0}\,\sigma\right) - (1-\kappa)\right).
\end{equation}
From eq.~\eqref{eq:dist_svalues}, we can see that $\rho$ is a quarter-circle law
with variance $1/\hq_0$, and truncated from below: it is supported in $(\sigma_{\min}, \sigma_{\max})$, with
\begin{equation}
    \sigma_{\min} = \frac{X_{\qc}(1-\kappa)}{\sqrt{\hq_0}}, \textrm{ and }
    \sigma_{\max} = \frac{2}{\sqrt{\hq_0}},
\end{equation}
and for $\sigma \in (\sigma_{\min}, \sigma_{\max})$ we have
\begin{equation}
    \rho(\sigma) = \frac{\sqrt{\hq_0}}{\kappa} \rho_{\qc}(\sqrt{\hq_0} \sigma).
\end{equation}
Substituting 
the above singular value distribution into eq.~\eqref{eq:Q_int}, we get:
\begin{equation}
Q = \frac{1}{\hq_0}\int_{X_{\mathrm{q.c.}}(1-\kappa)}^{2} \rho_{\mathrm{q.c.}}(\sigma)\, 
\sigma^2\, d\sigma.
\end{equation}
Recalling the definition $\hat{q}_0 = 2\alpha/Q$, we obtain the equation satisfied by the value $\alpha_c(\kappa)$ of the critical threshold:
\begin{equation}
\alpha_c(\kappa) = \frac{1}{2} \int_{X_{\mathrm{q.c.}}(1-\kappa)}^{2} 
\rho_{\mathrm{q.c.}}(\sigma)\, \sigma^2\, d\sigma.
\end{equation}
Notably, $X_{\mathrm{q.c.}}(0) = 0$ and $\int_0^2 \rho_{\mathrm{q.c.}}(\sigma)\,\sigma^2\, d\sigma = 1$, we recover $\alpha_c = \tfrac{1}{2}$ when $\kappa = 1$.

Notice that the equations above do not constrain the value of $Q$, which is the norm of solutions: this is consistent with the scale invariance of the problem, as the set of solutions is always a cone.
\Cref{fig:spectra} compares the singular value spectra (up to a global normalisation) at capacity for the decoupled and original problems, showing excellent agreement between simulations and theoretical predictions and highlighting the equivalence of the two setups.

\paragraph{Comparison with spectrum at initialisation} 
It is instructive to contrast the singular value distribution of the 
optimal weight matrix at the capacity threshold, given by Claim~\ref{claim:rank}, 
with the one at initialisation. The natural baseline depends on the parametrisation. 
For the full-rank case $\kappa = 1$, one typically initialises $W \in \mathbb{R}^{d \times d}$ 
directly with i.i.d.\ Gaussian entries; in this case the empirical singular value 
distribution of $W$ converges to the quarter-circle law $\rho_{\text{q.c.}}$ on $[0, 2]$, 
which already coincides with the spectrum predicted by Claim~\ref{claim:rank} 
at the capacity threshold. This coincidence is, however, specific to $\kappa = 1$: 
for $\kappa < 1$ the rank constraint cannot be enforced through such an initialisation, 
and one resorts instead to the two-layer parametrisation $W = U V^\top $ with 
$U, V \in \mathbb{R}^{d \times m}$, $m = \kappa d$, and i.i.d.\ Gaussian entries. 
Standard results in free probability (see e.g.\ \citet{Potters_Bouchaud_2020}) imply that, in the high-dimensional limit, 
the empirical singular value distribution of $W$ then converges to the free product $\mathrm{MP}_\kappa \boxtimes \mathrm{MP}_\kappa$ of two Marchenko--Pastur laws 
of shape $\kappa$. Equivalently, denoting by $G$ the Stieltjes transform of the law of 
the squared singular values of $W$, one has the cubic equation
\begin{equation}
\kappa^2 x^2\, G(x)^3 + 2\kappa(1-\kappa)\, x\, G(x)^2 
+ \bigl((1-\kappa)^2 - x\bigr) G(x) + 1 = 0,
\end{equation}
from which the bulk density follows by the inverse Stieltjes transformation. 
The spectra at initialisation (\cref{fig:init-spectra}) differ substantially from those 
of trained matrices in the $UV^\top$ parametrisation 
(\cref{fig:spectra}). 

\begin{figure}
    \centering
    \includegraphics[width=1\linewidth]{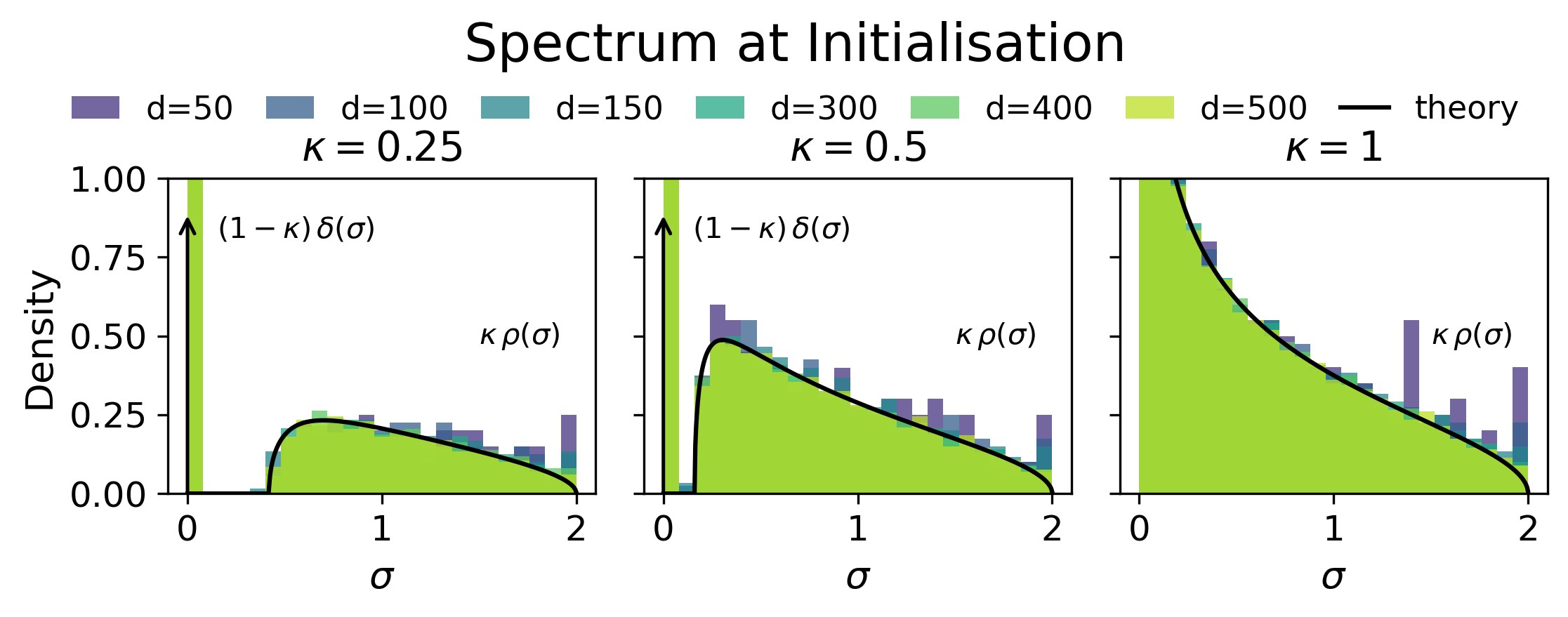}
    \caption{Distribution of singular values of $W = U V^\top$ at initialisation, with $U, V \in \mathbb{R}^{d \times \kappa d}$ having i.i.d.\ Gaussian entries, for $\kappa \in \{0.25, 0.5, 1\}$, rescaled so that the largest singular value equals $2$. Histograms from simulations (shaded) are overlaid with the theoretical prediction (solid line) given by the free multiplicative convolution $\mathrm{MP}_\kappa \boxtimes \mathrm{MP}_\kappa$~\citep{Potters_Bouchaud_2020}. For $\kappa < 1$, a delta peak at zero (arrow) accounts for the fraction $(1-\kappa)$ of vanishing singular values.}
    \label{fig:init-spectra}
\end{figure}
\section{Proof of the capacity threshold of the rank-constrained decoupled problem}
\label{app:rank_proof}

In this appendix, we prove~\cref{thm:rank_capacity}: the capacity threshold of
the decoupled problem~\eqref{eq:DP} with weights of rank at most $m = \kappa d$
is $\alpha_c^{\DP}(\kappa) = \alpha_c(\kappa)$, the value derived with the
replica method in~\cref{claim:rank}. For $\kappa = 1$, this gives
$\alpha_c^{\DP} = 1/2$, which, combined with the result
of~\citet{barnfield2026sharpcapacitythresholdslinear} for the original problem,
proves~\cref{conjecture-capacity} for full-rank matrices (\cref{cor:full_rank}).
The achievability part of the proof also applies to the original
problem~\eqref{eq:OP}: for every $\kappa\in(0,1]$ and $\alpha<\alpha_c(\kappa)$,
it admits a solution of rank at most $m$ with probability tending to one
(\cref{thm:sharp-achievability}).

\subsection{Setting and main result}
\label{subsec:rank_proof_setting}

\paragraph{Setting.}
The events in~\eqref{eq:OP} and~\eqref{eq:DP} are invariant under a common
rescaling of all embeddings, so we normalise them to have law
$\mcN(0,\Id_d/d)$ rather than $\mcN(0,\Id_d)$. Sample indices are denoted by
$\mu,\rho\in[p]$, while $i,j\in[d]$ are reserved for vector and matrix
coordinates. We treat the two problems in a common notation. For each query
$\mu\in[p]$, we call $u_\mu$ its \emph{target} and
$(u_\rho^{(\mu)})_{\rho\ne\mu}$ its \emph{competitors}. In both problems, the
inputs and the targets
\[
 (e_\mu)_{\mu\in[p]},\quad (u_\mu)_{\mu\in[p]}
 \quad\stackrel{\mathrm{independent}}{\sim}\quad
 \mcN(0,\Id_d/d)
\]
are mutually independent, while the competitors differ:
\begin{itemize}[leftmargin=*]
 \item in the decoupled problem~\eqref{eq:DP}, $u_\mu\coloneqq u_\mu^{(\mu)}$,
 and the competitors
 $\bigl(u_\rho^{(\mu)}\bigr)_{\mu,\rho\in[p],\,\rho\ne\mu}$ are independent
 $\mcN(0,\Id_d/d)$ vectors, independent of the inputs and the targets,
 including across queries;
 \item in the original problem~\eqref{eq:OP}, all queries share their
 competitors, which are the other targets: $u_\rho^{(\mu)}\coloneqq u_\rho$ for
 all $\mu$ and $\rho\ne\mu$.
\end{itemize}
For $\mcP\in\{\OP,\DP\}$, the rank-constrained retrieval event is
\begin{equation}
 \PR^{\mcP}_{p,d,m}
 \coloneqq\left\{\exists W\in\bbR^{d\times d}:\rank(W)\le m,\quad
 u_\mu^{\top}We_\mu>
 \bigl(u_\rho^{(\mu)}\bigr)^{\top}We_\mu
 \ \ \forall\mu,\ \forall\rho\ne\mu\right\}.
 \label{eq:model}
\end{equation}
A matrix has rank at most $m$ if and only if it factorises through a hidden
layer of width $m$, so~\eqref{eq:model} is the event that a two-layer linear
network with hidden dimension $m$ solves problem~$\mcP$, as
in~\cref{subsec:capacity_rank_constrained}. We assume $1\le m\le d$ without
loss of generality, since a larger rank budget is equivalent to $m=d$. We work
in the joint limit $p,d,m\to\infty$ with
\[
 \frac md\longrightarrow\kappa\in(0,1],
 \qquad
 \frac{p\log p}{d^2}\longrightarrow\alpha\in(0,\infty).
\]
As in~\cref{claim:rank}, let
$\rho_{\qc}(s)=\frac1\pi\sqrt{4-s^2}\,\mathbf{1}_{[0,2]}(s)$ be the
quarter-circle density and $X_{\qc}$ its quantile function, and define the
top-$\kappa$ spectral mass
\begin{equation}
 E_\kappa
 \coloneqq\int_{X_{\qc}(1-\kappa)}^2
       \rho_{\qc}(s)s^2\,ds.
 \label{eq:E-kappa}
\end{equation}

\paragraph{Main result.}
The threshold~\eqref{eq:alphac_kappa} of~\cref{claim:rank} is
$\alpha_c(\kappa)=E_\kappa/2$, so in this notation~\cref{thm:rank_capacity}
states that
\begin{equation}
 \bbP\bigl(\PR^{\DP}_{p,d,m}\bigr)\longrightarrow
 \begin{cases}
  1 & \text{if }\alpha<E_\kappa/2,\\
  0 & \text{if }\alpha>E_\kappa/2.
 \end{cases}
 \label{eq:rank_capacity_restated}
\end{equation}
and that $\bbP(\PR^{\OP}_{p,d,m})\to1$ for $\alpha<E_\kappa/2$.
For $\kappa=1$, $X_{\qc}(0)=0$ and
$E_1=\frac1\pi\int_0^2s^2\sqrt{4-s^2}\,ds=1$, so that $\alpha_c^{\DP}=1/2$.
Since~\citet[Theorem~1]{barnfield2026sharpcapacitythresholdslinear} prove
$\alpha_c^{\OP}=1/2$ under the same Gaussian model, this
gives~\cref{cor:full_rank}.

For the original problem above the threshold, the obstacle is the rank
constraint. The result of~\citet{barnfield2026sharpcapacitythresholdslinear}
gives $\bbP(\PR^{\OP}_{p,d,m})\to0$ only for $\alpha>1/2$, where no matrix of
any rank is a solution, because their argument cannot account for the rank
constraint (see below). For $\kappa=1$ this matches the threshold $E_1/2=1/2$.
For $\kappa<1$, however, $E_\kappa<1$, so it remains open whether~\eqref{eq:OP}
admits a solution of rank at most $m$ for $E_\kappa/2<\alpha<1/2$.

\paragraph{Relation to~\citet{barnfield2026sharpcapacitythresholdslinear}.}
The achievability part follows the sparse-repair construction
of~\citet{barnfield2026sharpcapacitythresholdslinear} for the full-rank
original problem. Their memory is block-diagonal: a large block stores all
associations in the correlation matrix $C=\sum_\mu u_\mu e_\mu^\top$ (the
Hebbian ansatz up to normalisation), and a block of vanishing relative
dimension $d_1$, of order $d(\log p)^{-1/8}$, adds weighted rank-one
corrections for the few queries whose margin is small. We replace the
correlation matrix in the large block by its rank-$m_0$ truncated singular
value decomposition, with $m_0=m-d_1$. This lowers the target scores from $1$
to $E_\kappa$ and the competitor variances from $\alpha/\log p$ to
$\alpha E_\kappa/\log p$ (\cref{lem:self,lem:competitor-scores}), which moves the threshold from
$1/2$ to $E_\kappa/2$. For $\kappa=1$, no truncation occurs and the
construction coincides with theirs.
The impossibility part uses a different argument.
\citet{barnfield2026sharpcapacitythresholdslinear} construct a dual
certificate: a nonnegative combination of the constraint matrices, not all
zero, that vanishes. Such a certificate cannot account for the rank
constraint: it excludes every matrix $W$ at once, and a convex argument of
this kind cannot distinguish the matrices of rank at most $m$ from their convex
hull, which is the whole of $\bbR^{d\times d}$. It can therefore only prove impossibility
when no matrix of any rank is a solution, i.e.\ for $\alpha>1/2$. We give
instead a direct argument: summing the constraints over the queries yields a
necessary condition, comparing the total target score with the sum of the
competitor maxima, and we bound both sides uniformly over all matrices of rank
at most $m$ (\cref{lem:uniform-competitor-max,lem:target-quotient}). Although summing
discards information, the resulting condition is tight: it fails for
$\alpha>E_\kappa/2$, exactly the threshold of the achievability part of the
proof, making the threshold sharp.

\paragraph{Where the decoupling is used.}
Only the impossibility part of the proof uses the decoupling.
\begin{itemize}[leftmargin=*]
 \item \emph{Achievability} (\cref{thm:sharp-achievability}) holds for both
 problems. It handles one query at a time: besides properties of the inputs
 and the targets, it only uses that the competitor scores of each query have
 a rotation-invariant law (\cref{lem:competitor-scores}). This holds in both
 problems, even though in the original problem the competitors also enter the
 memory. The dependence between the competitor scores of different queries
 plays no role.
 \item \emph{Impossibility} (\cref{thm:sharp-impossibility}) must handle all
 queries jointly, uniformly over all matrices of rank at most $m$.
 In~\cref{lem:uniform-competitor-max}, the competitors of different queries
 are independent, which gives the product tail
 bound~\eqref{eq:imp-weighted-tail} needed for a union bound over all such
 matrices. In the original problem all queries share the same competitors,
 and this product structure is lost.
\end{itemize}
For full-rank matrices, \citet{barnfield2026sharpcapacitythresholdslinear}
prove impossibility for the original problem through convexity, with a dual
certificate that, as explained above, cannot see the rank constraint. Our
argument instead accommodates the rank constraint, but relies on the
decoupling. An impossibility argument that handles both the coupling of the
original problem and the non-convexity of the rank constraint is the missing
step towards~\cref{conjecture-capacity} for $\kappa<1$.

\paragraph{Organisation.}
\cref{subsec:rank_proof_prelim} collects properties of the truncated
correlation matrix and of the competitor scores, for both problems.
\cref{subsec:rank_proof_achievability} proves achievability for both problems
and~\cref{subsec:rank_proof_impossibility} proves impossibility for the
decoupled problem; together they prove~\cref{thm:rank_capacity}.
\cref{subsec:rank_proof_lemmas} contains the remaining technical lemmas.

\subsection{Truncated-SVD preliminaries}
\label{subsec:rank_proof_prelim}

The correlation memory below is the starting point of the full-rank
construction of \citet{barnfield2026sharpcapacitythresholdslinear}; here we use its
rank-$m$ truncation. Except for~\cref{lem:competitor-scores}, which treats the
competitors of both problems, this subsection involves only the inputs and the
targets, and therefore applies to both problems.

Set
\begin{equation}
 E=[e_1,\ldots,e_p],\qquad U=[u_1,\ldots,u_p],\qquad
 C=UE^{\top}=\sum_{\mu=1}^p u_\mu e_\mu^{\top}.
 \label{eq:truncated}
\end{equation}
Note that $U$ contains only the targets. In the decoupled problem, it is
independent of the competitors; in the original problem, its columns are also
the competitors. Choose a singular value decomposition
$C=\sum_{r=1}^d s_r(C)a_rb_r^{\top}$, where
$s_1(C)\ge\cdots\ge s_d(C)\ge0$ and $(a_r)$ and $(b_r)$ are orthonormal
bases of $\bbR^d$.  Define the rank-$m$ truncation by
\begin{equation}
 T_m(C)\coloneqq \sum_{r=1}^m s_r(C)a_rb_r^{\top}.
 \label{eq:truncated-svd-definition}
\end{equation}
Conditionally on $U$, the entries of $C$ are jointly Gaussian with a covariance
that is almost surely non-degenerate
(see~\eqref{eq:C-conditional-covariance} below), so $C$ almost surely has
distinct singular values, and $T_m(C)$ does not depend on the choice of the
decomposition.

\begin{lemma}[Conditional Gaussian structure of the correlation memory]
\label{lem:gaussian-correlation-memory}
Set
\begin{equation}
 S_U\coloneqq \frac d p UU^{\top},
 \qquad
 S_E\coloneqq \frac d p EE^{\top}.
 \label{eq:SU}
\end{equation}
If $G\in\bbR^{d\times d}$ has independent $\mcN(0,1/d)$ entries and is
independent of $U$, then, conditionally on $U$,
\begin{equation}
 \sqrt{\frac d p}\,C
 \ \stackrel{\mathrm{law}}{=}\
 S_U^{1/2}G.
 \label{eq:conditional-law-C}
\end{equation}
Moreover, for a universal constant $c>0$, each of the two bounds
\begin{equation}
 \norm{S_U-\Id_d}_{\op},\ \norm{S_E-\Id_d}_{\op}
 \le c\left(\sqrt{\frac d p}+\frac d p\right)
 =O\!\left(\sqrt{\frac{\log p}{d}}\right)
 =o(1)
 \label{eq:SU-concentration}
\end{equation}
holds with probability at least $1-2e^{-d}$. Consequently, for all
sufficiently large $p$,
\begin{equation}
 \norm{\sqrt{\frac d p}\,C}_{\op}\le3
 \label{eq:correlation-memory-operator-bound}
\end{equation}
with probability at least $1-4e^{-c'd}$, for a universal constant $c'>0$.
\end{lemma}

\begin{proof}
For $i,j\in[d]$,
\begin{equation}
 C_{ij}=\sum_{\mu=1}^p u_{i\mu}e_{j\mu}.
 \label{eq:C-entrywise}
\end{equation}
Conditional on $U$, the entries of $C$ are jointly centered Gaussian, and
independence of the entries of $E$ gives
\begin{align}
 \operatorname{Cov}\!\left(C_{ij},C_{i'j'}\mid U\right)
 &=\sum_{\mu,\nu=1}^p u_{i\mu}u_{i'\nu}
   \EE[e_{j\mu}e_{j'\nu}] \notag\\
 &=\frac{\mathbf{1}_{\{j=j'\}}}{d}(UU^{\top})_{ii'}.
 \label{eq:C-conditional-covariance}
\end{align}
Hence the columns of $\sqrt{d/p}\,C$ are conditionally independent centered
Gaussian vectors with covariance $S_U/d$, which proves
\eqref{eq:conditional-law-C}.

Writing $x_\mu=\sqrt d\,u_\mu$, we have
$S_U=p^{-1}\sum_{\mu=1}^p x_\mu x_\mu^{\top}$ with
$x_\mu\sim\mcN(0,\Id_d)$ independently.  Let $A\in\bbR^{p\times d}$ have
rows $x_\mu^{\top}$, so $S_U=A^{\top}A/p$.  The
two-sided sub-Gaussian matrix bound
\citep[Theorem~4.6.1]{vershynin2018high}, applied with $t=\sqrt d$, gives
\[
 \bbP\!\left(\norm{S_U-\Id_d}_{\op}
 >c\left(\sqrt{\frac d p}+\frac d p\right)\right)\le2e^{-d},
\]
which is \eqref{eq:SU-concentration} for $S_U$. Since $E$ has the same law
as $U$, the same bound holds for $S_E$.  Finally, the Gaussian
operator-norm bound
\citep[Theorem~7.3.1 and Corollary~7.3.3]{vershynin2018high} gives
$\norm G_{\op}\le5/2$ except with probability $2e^{-c'd}$. On the
intersection of this event with~\eqref{eq:SU-concentration} for $S_U$, we
have $\norm{S_U^{1/2}G}_{\op}\le\norm{S_U}_{\op}^{1/2}\norm G_{\op}
\le(1+o(1))\,5/2\le3$ for all large $p$, and
\eqref{eq:conditional-law-C} proves
\eqref{eq:correlation-memory-operator-bound}.
\end{proof}

\begin{lemma}[Truncated-SVD self-averaging properties]
\label{lem:self}
Let $m/d\to\kappa\in(0,1]$ and
$p\log p/d^2\to\alpha\in(0,\infty)$, and set $T=T_m(C)$.
\begin{enumerate}[label=(\alph*)]
 \item There is a universal constant $c>0$ such that, for every fixed
 $\eta>0$ and all sufficiently large $p$,
 \begin{equation}
  \bbP\!\left(\abs{\frac1p\norm{T}_{\rF}^2-E_\kappa}>\eta\right)
  \le e^{-cd}.
  \label{eq:energy}
 \end{equation}
 \item For every fixed $\varepsilon>0$,
 \begin{equation}
  \max_{\mu\le p}\bbP\!\left(
  \abs{u_\mu^{\top}Te_\mu-E_\kappa}>\varepsilon\right)
  =O_\varepsilon\!\left(\frac1{\log p}\right).
  \label{eq:quantitative-target-score}
 \end{equation}
 \item In probability,
 \begin{equation}
  \max_{\mu\le p}\abs{\frac d p\norm{Te_\mu}^2-E_\kappa}
  \longrightarrow0.
  \label{eq:uniform-needed-variance}
 \end{equation}
 \item For some fixed $A_0=A_0(\alpha)<\infty$,
 \begin{equation}
  \bbP\!\left(\max_{\mu\le p}\abs{u_\mu^{\top}Te_\mu}\le A_0\right)
  \longrightarrow1.
  \label{eq:crude-uniform}
 \end{equation}
\end{enumerate}
\end{lemma}

\begin{proof}
Throughout the proof, $c,c'>0$ denote universal constants whose values may
change from line to line. Let
$b_d\coloneqq c(\sqrt{d/p}+d/p)=o(1)$ be the bound
in~\eqref{eq:SU-concentration}, and define the event
\begin{equation}
 \Omega_0\coloneqq\left\{
 \begin{aligned}
  &\norm{S_E-\Id_d}_{\op}\le b_d,\quad \norm{E}_{\op}\le2\sqrt{p/d},\\
  &\norm{U}_{\op}\le2\sqrt{p/d},\quad \norm{C}_{\op}\le3\sqrt{p/d}
 \end{aligned}
 \right\}.
 \label{eq:Omega0}
\end{equation}
The Gaussian operator-norm bound
\citep[Theorem~7.3.1 and Corollary~7.3.3]{vershynin2018high} gives
$\norm{E}_{\op}\le\sqrt{p/d}+2\le2\sqrt{p/d}$ for all large $p$, except with
probability $2e^{-cd}$, and likewise for $U$. Together with
Lemma~\ref{lem:gaussian-correlation-memory}, this gives
$\bbP(\Omega_0^c)\le e^{-cd}$ for all large $p$. On $\Omega_0$,
since $\rank(T)\le m\le d$,
\begin{equation}
 \norm{T}_{\op}\le\norm{C}_{\op}\le3\sqrt{\frac pd},
 \qquad
 \norm{T}_{\rF}^2\le m\norm{C}_{\op}^2\le9p.
 \label{eq:basic-norm-bounds}
\end{equation}

\medskip\noindent\textbf{Rotation invariance.}
Let $Q$ be Haar distributed on the orthogonal group $O(p)$ and independent
of $(E,U)$, and let $q_\mu$ denote the $\mu$th column of $Q$. Each
$q_\mu$ is uniform on $\mathbb S^{p-1}$. The rows of $E$ and of $U$ are
independent $\mcN(0,\Id_p/d)$ vectors, so
$(EQ,UQ)\stackrel{\mathrm{law}}{=}(E,U)$. The substitution
$(E,U)\mapsto(EQ,UQ)$ leaves $C=(UQ)(EQ)^{\top}$, $T$, $S_U$, $S_E$ and
$\Omega_0$ unchanged, and maps $(e_\mu,u_\mu)$ to $(Eq_\mu,Uq_\mu)$.
Hence, setting
\[
 K\coloneqq U^{\top}TE,
 \qquad
 L\coloneqq E^{\top}T^{\top}TE\in\bbR^{p\times p},
\]
which are functions of $(E,U)$ only, we obtain the identity in law
\begin{equation}
 \Bigl(\bigl(u_\mu^{\top}Te_\mu\bigr)_{\mu\le p},\,
       \bigl(\norm{Te_\mu}^2\bigr)_{\mu\le p}\Bigr)
 \stackrel{\mathrm{law}}{=}
 \Bigl(\bigl(q_\mu^{\top}Kq_\mu\bigr)_{\mu\le p},\,
       \bigl(q_\mu^{\top}Lq_\mu\bigr)_{\mu\le p}\Bigr).
 \label{eq:rotation-invariance}
\end{equation}
In Steps~2--4 we bound the right-hand side conditionally on $(E,U)$, using
only the randomness of $Q$. We will use the identities
\begin{equation}
 \operatorname{tr}K=\langle T,C\rangle_{\rF}=\norm{T}_{\rF}^2,
 \qquad
 \operatorname{tr}L=\frac pd\operatorname{tr}\bigl(T^{\top}TS_E\bigr),
 \label{eq:K-L-traces}
\end{equation}
where the first holds because
$\langle T_m(C),C\rangle_{\rF}=\sum_{r\le m}s_r(C)^2$.

\medskip\noindent\textbf{1. Proof of~\eqref{eq:energy}.}
By~\eqref{eq:conditional-law-C}, $\sqrt{d/p}\,C$ has the same law as
$S_U^{1/2}G$, with $G$ as in Lemma~\ref{lem:gaussian-correlation-memory}.
Since the probability in~\eqref{eq:energy} depends only on the law of $C$, it
suffices to prove~\eqref{eq:energy} with $\sqrt{d/p}\,C$ replaced by
$S_U^{1/2}G$, which we do from now on in this step. Let
$\Omega_1\coloneqq\{\norm{S_U-\Id_d}_{\op}\le b_d,\ \norm G_{\op}\le3\}$.
By~\eqref{eq:SU-concentration} and the Gaussian operator-norm bound,
$\bbP(\Omega_1^c)\le e^{-cd}$. On $\Omega_1$, the singular-value form
of Weyl's inequality, which applies even though the matrices need not be
symmetric,\footnote{For arbitrary real $d\times d$ matrices
$A,B$, set
$\mathcal H(A)=\begin{psmallmatrix}0&A\\ A^{\top}&0\end{psmallmatrix}$.
The $d$ largest eigenvalues of the symmetric matrix $\mathcal H(A)$ are
$s_1(A),\ldots,s_d(A)$, while
$\mathcal H(A-B)^2=\operatorname{diag}((A-B)(A-B)^{\top},
(A-B)^{\top}(A-B))$ gives
$\norm{\mathcal H(A)-\mathcal H(B)}_{\op}=\norm{A-B}_{\op}$.
Applying the eigenvalue version of Weyl's inequality to $\mathcal H(A)$ and
$\mathcal H(B)$ yields $\abs{s_r(A)-s_r(B)}\le\norm{A-B}_{\op}$.}
gives
\[
 \max_{r\le d}\abs{\sqrt{\frac d p}\,s_r(C)-s_r(G)}
 \le\norm{(S_U^{1/2}-\Id_d)G}_{\op}
 \le\norm{G}_{\op}\norm{S_U-\Id_d}_{\op}
 \le3b_d,
\]
where we used $\norm{S_U^{1/2}-\Id_d}_{\op}\le\norm{S_U-\Id_d}_{\op}$, which
follows from $S_U\succeq0$ and $\abs{\sqrt\lambda-1}\le\abs{\lambda-1}$ for
$\lambda\ge0$. All the singular values compared here are at most $4$, since
$\norm{S_U^{1/2}G}_{\op}\le3(1+b_d)^{1/2}\le4$, so
\[
 \abs{\frac1p\norm{T}_{\rF}^2-\frac1d\norm{T_m(G)}_{\rF}^2}
 =\frac1d\abs{\sum_{r=1}^m\left(\frac d p\,s_r(C)^2-s_r(G)^2\right)}
 \le24b_d=o(1).
\]

It remains to show that $\frac1d\norm{T_m(G)}_{\rF}^2$ concentrates around
$E_\kappa$. Write $G=Z/\sqrt d$, with $Z$ standard Gaussian, and set
\[
 F(Z)\coloneqq \frac{\norm{T_m(Z)}_{\rF}}d
 =\frac1d\sup_{\substack{\rank(B)\le m\\\norm{B}_{\rF}\le1}}
                \langle Z,B\rangle_{\rF}.
\]
The variational formula makes $F$ $1/d$-Lipschitz.  Gaussian concentration
\citep[Theorem~5.2.2]{vershynin2018high} therefore gives
$\bbP(\abs{F-\EE F}>t)\le2e^{-cd^2t^2}$.  The quarter-circle law for the
singular values of $G$, i.e.\ the Marchenko--Pastur law with unit aspect
ratio for the eigenvalues of
$GG^{\top}$~\citep{Anderson_Guionnet_Zeitouni_2009}, gives\footnote{The
empirical singular value distribution of $G$ converges
weakly in probability to $\rho_{\qc}$, whose distribution function is
continuous and strictly increasing on $[0,2]$, and
$\norm G_{\op}\le5/2$ with probability tending to one. Hence
$s_m(G)\to X_{\qc}(1-\kappa)$ in probability, and for every fixed $\delta>0$,
with probability tending to one, $\frac1d\sum_{r\le m}s_r(G)^2$ lies between
$\int s^2\mathbf 1_{\{s\ge X_{\qc}(1-\kappa)\pm\delta\}}\,d\mu_G(s)$, where
$\mu_G$ is the empirical singular value distribution. Both bounds converge to
within $O(\delta)$ of $E_\kappa$.}
\[
 F^2=\frac1d\norm{T_m(G)}_{\rF}^2
 =\frac1d\sum_{r=1}^m s_r(G)^2
 \longrightarrow
 \int_{X_{\qc}(1-\kappa)}^2 s^2\rho_{\qc}(s)\,ds
 =E_\kappa
\]
in probability. Since $\EE F^2\le\EE\norm{Z}_{\rF}^2/d^2=1$, the family $F$
is uniformly integrable, so $\EE F\to\sqrt{E_\kappa}$. Hence, for each fixed
$\eta>0$ and all large $d$,
\[
 \bbP\!\left(\abs{\frac1d\norm{T_m(G)}_{\rF}^2-E_\kappa}>\frac\eta2\right)
 \le2e^{-c_\eta d^2}.
\]
For all large $p$, $24b_d\le\eta/2$, and the estimates above give
\[
 \bbP\!\left(\abs{\frac1p\norm{T}_{\rF}^2-E_\kappa}>\eta\right)
 \le\bbP(\Omega_1^c)+2e^{-c_\eta d^2}
 \le e^{-cd}+2e^{-c_\eta d^2},
\]
which is at most $e^{-c'd}$ for all large $p$. This
proves~\eqref{eq:energy}.

\medskip\noindent\textbf{2. Proof of~\eqref{eq:quantitative-target-score}.}
Fix $\mu$ and write $q=q_\mu$. By~\eqref{eq:rotation-invariance}
and~\eqref{eq:K-L-traces},
\[
 \bbP\!\left(\abs{u_\mu^{\top}Te_\mu-\frac{\norm T_{\rF}^2}p}
 >\frac\varepsilon2\right)
 =\bbP\!\left(\abs{q^{\top}Kq-\frac{\operatorname{tr}K}p}
 >\frac\varepsilon2\right).
\]
Haar invariance gives $\mathbb E_Q[qq^{\top}]=\Id_p/p$, hence
\begin{equation}
 \mathbb E_Q[q^{\top}Kq]
 =\operatorname{tr}\!\left(K\mathbb E_Q[qq^{\top}]\right)
 =\frac{\operatorname{tr}K}{p}.
 \label{eq:haar-target-mean}
\end{equation}
Write $K_s=(K+K^{\top})/2$, so that $q^{\top}Kq=q^{\top}K_sq$.
The fourth moments of a uniform $q\in\mathbb S^{p-1}$ are
\[
 \mathbb E_Q[q_iq_jq_kq_\ell]
 =\frac{\delta_{ij}\delta_{k\ell}
        +\delta_{ik}\delta_{j\ell}
        +\delta_{i\ell}\delta_{jk}}{p(p+2)}.
\]
Expanding $(q^{\top}K_sq)^2$ gives
$\mathbb E_Q[(q^{\top}Kq)^2]
=((\operatorname{tr}K_s)^2+2\norm{K_s}_{\rF}^2)/(p(p+2))$, so
\begin{align}
 \operatorname{Var}_Q(q^{\top}Kq)
 &=\frac{2}{p(p+2)}
   \left(\norm{K_s}_{\rF}^2-
         \frac{(\operatorname{tr}K_s)^2}{p}\right)
 \le\frac{2\norm{E}_{\op}^2
              \norm{U}_{\op}^2\norm{T}_{\rF}^2}{p^2},
 \label{eq:haar-target-variance}
\end{align}
using $\norm{K_s}_{\rF}\le\norm{K}_{\rF}
\le\norm{U}_{\op}\norm{T}_{\rF}\norm{E}_{\op}$. On the event $\Omega_0$
of~\eqref{eq:Omega0},
by~\eqref{eq:basic-norm-bounds}, the right-hand side is $O(p/d^2)$. Since
$\Omega_0$ depends only on $(E,U)$, Chebyshev's
inequality conditional on $(E,U)$ gives
\begin{equation}
 \bbP\!\left(\abs{q^{\top}Kq-\frac{\operatorname{tr}K}{p}}
 >\frac\varepsilon2\right)
 \le\bbP(\Omega_0^c)+O\!\left(\frac{p}{\varepsilon^2d^2}\right)
 =O\!\left(\frac1{\varepsilon^2\log p}\right),
 \label{eq:target-centered-rate}
\end{equation}
where we used $p/d^2=(p\log p/d^2)/\log p$. Together
with~\eqref{eq:energy} at $\eta=\varepsilon/2$ and the triangle inequality,
this proves~\eqref{eq:quantitative-target-score}; none of the bounds depends
on $\mu$.

\medskip\noindent\textbf{3. Proof of~\eqref{eq:uniform-needed-variance}.}
By~\eqref{eq:rotation-invariance}, it suffices to prove that
$\max_{\mu\le p}\abs{\frac dp q_\mu^{\top}Lq_\mu-E_\kappa}\to0$ in
probability. We first locate the center. By~\eqref{eq:K-L-traces}, on
$\Omega_0$,
\begin{equation}
 \abs{\frac d{p^2}\operatorname{tr}L-\frac{\norm T_{\rF}^2}p}
 =\frac1p\abs{\operatorname{tr}\bigl(T^{\top}T(S_E-\Id_d)\bigr)}
 \le\norm{S_E-\Id_d}_{\op}\frac{\norm T_{\rF}^2}p
 \le9b_d,
 \label{eq:haar-variance-mean}
\end{equation}
so $\frac d{p^2}\operatorname{tr}L\to E_\kappa$ in probability
by~\eqref{eq:energy}. It remains to control the deviations of the diagonal
entries $q_\mu^{\top}Lq_\mu$ from their Haar mean
$\operatorname{tr}L/p$, simultaneously for all $\mu$. On the event $\Omega_0$
of~\eqref{eq:Omega0},
\eqref{eq:basic-norm-bounds} gives
\[
 \norm{L}_{\op}\le\norm{E}_{\op}^2\norm{T}_{\op}^2\le\frac{36p^2}{d^2},
 \qquad
 \operatorname{tr}L\le\norm{E}_{\op}^2\norm{T}_{\rF}^2\le\frac{36p^2}{d},
\]
hence $\norm{L}_{\rF}^2\le\norm{L}_{\op}\operatorname{tr}L\le36^2p^4/d^3$,
since $L\succeq0$. Conditionally on $(E,U)\in\Omega_0$,
Lemma~\ref{lem:haar-diagonal} applied to $H=L$ gives, except with
probability $4/p$,
\[
 \frac dp\max_{\mu\le p}\abs{q_\mu^{\top}Lq_\mu-\frac{\operatorname{tr}L}p}
 \le C_1\frac d{p^2}\left(\norm{L}_{\rF}\sqrt{\log p}
                         +\norm{L}_{\op}\log p\right)
 \le36C_1\left(\sqrt{\frac{\log p}d}+\frac{\log p}d\right),
\]
which tends to zero. Since $\bbP(\Omega_0^c)\to0$, this
proves~\eqref{eq:uniform-needed-variance}.

\medskip\noindent\textbf{4. Proof of~\eqref{eq:crude-uniform}.}
By~\eqref{eq:rotation-invariance}, $\max_{\mu\le p}\abs{u_\mu^{\top}Te_\mu}$
has the law of $\max_{\mu\le p}\abs{q_\mu^{\top}K_sq_\mu}$, with $K_s$ as in
Step~2. On the event $\Omega_0$ of~\eqref{eq:Omega0}, \eqref{eq:K-L-traces}
and~\eqref{eq:basic-norm-bounds} give
$\operatorname{tr}K_s/p=\norm T_{\rF}^2/p\le9$ and, since symmetrisation
does not increase the Frobenius or the operator norm,
\[
 \norm{K_s}_{\rF}\le\norm U_{\op}\norm T_{\rF}\norm E_{\op}
 \le\frac{12p^{3/2}}d,
 \qquad
 \norm{K_s}_{\op}\le\norm U_{\op}\norm T_{\op}\norm E_{\op}
 \le12\left(\frac p d\right)^{3/2}.
\]
Conditionally on $(E,U)\in\Omega_0$, Lemma~\ref{lem:haar-diagonal}
applied to $H=K_s$ gives, except with probability $4/p$,
\[
 \max_{\mu\le p}\abs{q_\mu^{\top}K_sq_\mu}
 \le9+12C_1\left(\frac{\sqrt{p\log p}}{d}
                 +\frac{\sqrt p\log p}{d^{3/2}}\right).
\]
The first ratio tends to $\sqrt\alpha$, while the second equals
$\frac{\sqrt{p\log p}}d\sqrt{\frac{\log p}d}\to0$. Hence
\eqref{eq:crude-uniform} holds with any $A_0>9+12C_1\sqrt\alpha$.
Unlike~\eqref{eq:quantitative-target-score}, this is only a crude bound:
uniformly over $\mu$, the fluctuations of $u_\mu^{\top}Te_\mu$ are of order
$\sqrt{p\log p}/d\to\sqrt\alpha$, which is why the achievability proof
needs a repair step.
\end{proof}

\begin{lemma}[Competitor scores of the truncated memory]
\label{lem:competitor-scores}
Let $m/d\to\kappa\in(0,1]$ and $p\log p/d^2\to\alpha\in(0,\infty)$, and set
$T=T_m(C)$. In either problem, collect the competitors of query $\mu$ in
$U^{(\mu)}\coloneqq\bigl[u_\rho^{(\mu)}\bigr]_{\rho\ne\mu}
\in\bbR^{d\times(p-1)}$, and write
\[
 Z_{\rho\mu}\coloneqq\bigl(u_\rho^{(\mu)}\bigr)^{\top}Te_\mu
 \quad(\rho\ne\mu),
 \qquad
 Z_\mu\coloneqq\bigl(Z_{\rho\mu}\bigr)_{\rho\ne\mu}
 =\bigl(U^{(\mu)}\bigr)^{\top}Te_\mu\in\bbR^{p-1}
\]
for the competitor scores of query $\mu$. In the decoupled problem,
$U^{(\mu)}$ is a Gaussian matrix independent of $(E,U)$; in the original
problem, it is $U$ with its $\mu$th column removed. Then, in both problems, the
following hold.
\begin{enumerate}[label=(\alph*)]
 \item For every $\mu$, the law of $Z_\mu$ is invariant under all orthogonal
 transformations of $\bbR^{p-1}$. Consequently, for all $t,r>0$,
 \begin{equation}
  \bbP\Bigl(\max_{\rho\ne\mu}Z_{\rho\mu}\ge t,\ \norm{Z_\mu}^2\le r(p-1)\Bigr)
  \le(p-1)\,e^{-t^2/(2r)}.
  \label{eq:competitor-tail}
 \end{equation}
 \item In probability,
 \begin{equation}
  \max_{\mu\le p}\abs{\frac{\log p}{p-1}\norm{Z_\mu}^2-\alpha E_\kappa}
  \longrightarrow0.
  \label{eq:competitor-radius}
 \end{equation}
\end{enumerate}
\end{lemma}

\begin{proof}
\medskip\noindent\textbf{1. Rotation invariance.}
In the decoupled problem, conditionally on $(E,U)$, the vector $Te_\mu$ is
fixed and the columns of $U^{(\mu)}$ are still independent
$\mcN(0,\Id_d/d)$ vectors, so $Z_\mu=(U^{(\mu)})^{\top}Te_\mu$ is conditionally
$\mcN\bigl(0,\frac{\norm{Te_\mu}^2}d\Id_{p-1}\bigr)$.
This conditional law is rotation invariant, hence so is the law of $Z_\mu$.

In the original problem, $U^{(\mu)}$ is $U$ without its $\mu$th column. Fix
$R\in O(p-1)$, and let $\Pi\in O(p)$ act as $R$ on the coordinates
$\rho\ne\mu$ and leave the $\mu$th coordinate unchanged. The rows of $E$ and
of $U$ are independent $\mcN(0,\Id_p/d)$ vectors, so
$(E\Pi,U\Pi)\stackrel{\mathrm{law}}{=}(E,U)$. The substitution
$(E,U)\mapsto(E\Pi,U\Pi)$ leaves $C=(U\Pi)(E\Pi)^{\top}$, hence $T$,
unchanged, and it also leaves $e_\mu$, the $\mu$th column of $E$, unchanged.
It maps $U^{(\mu)}$ to $U^{(\mu)}R$, and therefore
$Z_\mu=(U^{(\mu)})^{\top}Te_\mu$ to $R^{\top}Z_\mu$. Thus $R^{\top}Z_\mu\stackrel{\mathrm{law}}{=}Z_\mu$ for every
$R\in O(p-1)$.

\medskip\noindent\textbf{2. Proof of~\eqref{eq:competitor-tail}.}
We use the spherical-cap bound~\citep[Lemma~2.2]{ball1997elementary}: if
$\theta$ is uniformly distributed on the unit sphere $\mathbb S^{n-1}$
of~$\bbR^n$ and $v\in\mathbb S^{n-1}$, then
\begin{equation}
 \bbP\bigl(\theta^{\top}v\ge s\bigr)\le e^{-ns^2/2}
 \qquad\text{for all }s\ge0
 \label{eq:spherical-cap}
\end{equation}
(for $s\ge1$, the left-hand side is zero). Let
$\theta_\mu\coloneqq Z_\mu/\norm{Z_\mu}$ on $\{Z_\mu\ne0\}$, let $v_\rho$ be the
coordinate vector of $\bbR^{p-1}$ indexed by $\rho\ne\mu$, and set
$s\coloneqq t/\sqrt{r(p-1)}$. Then
\begin{align*}
 \bbP\Bigl(\max_{\rho\ne\mu}Z_{\rho\mu}\ge t,\ \norm{Z_\mu}^2\le r(p-1)\Bigr)
 &\overset{\text{(i)}}{\le}
  \bbP\Bigl(Z_\mu\ne0,\ \exists\rho\ne\mu:\ \theta_\mu^{\top}v_\rho\ge s\Bigr)\\
 &\overset{\text{(ii)}}{\le}
  \sum_{\rho\ne\mu}\bbP\bigl(\theta_\mu^{\top}v_\rho\ge s\bigm|Z_\mu\ne0\bigr)\\
 &\overset{\text{(iii)}}{\le}
  (p-1)\,e^{-(p-1)s^2/2}=(p-1)\,e^{-t^2/(2r)}.
\end{align*}
(i)~If $Z_{\rho\mu}\ge t>0$ and $\norm{Z_\mu}\le\sqrt{r(p-1)}$, then
$Z_\mu\ne0$ and $\theta_\mu^{\top}v_\rho=Z_{\rho\mu}/\norm{Z_\mu}\ge s$.
(ii)~Union bound over $\rho\ne\mu$.
(iii)~By Step~1, given $Z_\mu\ne0$, $\theta_\mu$ is uniform on
$\mathbb S^{p-2}$; apply~\eqref{eq:spherical-cap} with $n=p-1$.

\medskip\noindent\textbf{3. Proof of~\eqref{eq:competitor-radius}.}
Since $\frac{\log p}d\norm{Te_\mu}^2=\frac{p\log p}{d^2}\cdot\frac dp\norm{Te_\mu}^2$,
\eqref{eq:uniform-needed-variance} and $p\log p/d^2\to\alpha$ give
\begin{equation}
 \max_{\mu\le p}\abs{\frac{\log p}{d}\norm{Te_\mu}^2-\alpha E_\kappa}
 \longrightarrow0
 \qquad\text{in probability}.
 \label{eq:competitor-radius-proxy}
\end{equation}
It therefore suffices to show that
\[
 \max_{\mu\le p}\abs{\frac{\log p}{p-1}\norm{Z_\mu}^2-\frac{\log p}{d}\norm{Te_\mu}^2}
 \longrightarrow0
 \qquad\text{in probability}.
\]

In the decoupled problem, by Step~1, conditionally on $(E,U)$,
\[
 \norm{Z_\mu}^2=\sum_{\rho\ne\mu}Z_{\rho\mu}^2
\]
is a sum of $p-1$ independent squares of
$\mcN\bigl(0,\frac{\norm{Te_\mu}^2}d\bigr)$ variables. Bernstein's inequality \citep[Theorem~2.8.1]{vershynin2018high} and a union
bound over $\mu$ give, for every fixed $\delta\in(0,1]$, with probability at
least $1-2pe^{-c(p-1)\delta^2}$,
\[
 \abs{\frac{\log p}{p-1}\norm{Z_\mu}^2-\frac{\log p}{d}\norm{Te_\mu}^2}
 \le\delta\,\frac{\log p}{d}\norm{Te_\mu}^2
 \qquad\text{for every }\mu\le p.
\]
Since $\delta$ is arbitrary, \eqref{eq:competitor-radius-proxy}
gives~\eqref{eq:competitor-radius}.

In the original problem,
$U^{(\mu)}(U^{(\mu)})^{\top}=UU^{\top}-u_\mu u_\mu^{\top}
=\frac pdS_U-u_\mu u_\mu^{\top}$, so
\begin{align*}
 &\abs{\frac{\log p}{p-1}\norm{Z_\mu}^2-\frac{\log p}{d}\norm{Te_\mu}^2}\\
 &\qquad\le\left(\frac p{p-1}\norm{S_U-\Id_d}_{\op}+\frac1{p-1}\right)
   \frac{\log p}d\norm{Te_\mu}^2
   +\frac{\log p}{p-1}\bigl(u_\mu^{\top}Te_\mu\bigr)^2.
\end{align*}
The first term tends to zero uniformly in $\mu$ by~\eqref{eq:SU-concentration}
and~\eqref{eq:competitor-radius-proxy}. By~\eqref{eq:crude-uniform}, the
second is at most $A_0^2\log p/(p-1)\to0$ with probability tending to one.
This proves~\eqref{eq:competitor-radius}.
\end{proof}

\subsection{Achievability for both problems}
\label{subsec:rank_proof_achievability}

\begin{theorem}[Rank-constrained achievability]
\label{thm:sharp-achievability}
If
\begin{equation}
 \alpha<\frac{E_\kappa}{2},
 \label{eq:sharp-ach-condition}
\end{equation}
then, for both $\mcP\in\{\OP,\DP\}$,
\[
 \bbP\bigl(\PR^{\mcP}_{p,d,m}\bigr)\longrightarrow1.
\]
\end{theorem}

\begin{proof}
Set
\begin{equation}
 d_1=\left\lfloor\frac d{(\log p)^{1/8}}\right\rfloor,
 \qquad d_0=d-d_1,
 \qquad m_0=m-d_1.
 \label{eq:repair-dimensions}
\end{equation}
For all sufficiently large $p$, $d_1\le m$, and
$m_0/d_0\to\kappa$ and $p\log p/d_0^2\to\alpha$.  We reserve the first
$d_0$ coordinates for the standard truncated-SVD construction and the
remaining $d_1$ coordinates for a correction.

Let $P_a:\bbR^d\to\bbR^{d_a}$ denote the coordinate projection onto block
$a\in\{0,1\}$.  For every input, target, and competitor, define the
normalized blocks
\[
 e_\mu^a\coloneqq \sqrt{\frac d{d_a}}P_a e_\mu,
 \qquad
 u_\mu^a\coloneqq \sqrt{\frac d{d_a}}P_a u_\mu,
 \qquad
 u_\rho^{(\mu),a}
 \coloneqq \sqrt{\frac d{d_a}}P_a u_\rho^{(\mu)}.
\]
Thus the two blocks are independent, and each block has the standard
$\mathcal N(0,\Id_{d_a}/d_a)$ normalization.  Write
\[
 E_0=[e_1^0,\ldots,e_p^0],\qquad
 U_0=[u_1^0,\ldots,u_p^0],
\]
and take the standard construction on the large block to be
\[
 W_0\coloneqq T_{m_0}(U_0E_0^{\top})
 =T_{m_0}\!\left(\sum_{\mu=1}^p u_\mu^0(e_\mu^0)^{\top}\right).
\]
For each query $\mu$, its margin supplied by this base block is
\begin{equation}
 M_\mu^{(0)}\coloneqq (u_\mu^0)^{\top}W_0e_\mu^0
 -\max_{\rho\ne\mu}(u_\rho^{(\mu),0})^{\top}W_0e_\mu^0.
 \label{eq:base-margin}
\end{equation}
Fix a margin level $\gamma>0$, to be chosen later, and define the set of
queries not safely handled by the base construction by
\begin{equation}
 \mathcal B\coloneqq \{\mu\in[p]:M_\mu^{(0)}\le\gamma\}.
 \label{eq:bad-target-set}
\end{equation}

\medskip\noindent\textbf{Bounding the size of the bad set.}

For $\mu\in[p]$, write
\[
 S_\mu^{(0)}\coloneqq (u_\mu^0)^{\top}W_0e_\mu^0,
 \qquad
 Z_{\rho\mu}^{(0)}\coloneqq (u_\rho^{(\mu),0})^{\top}W_0e_\mu^0
 \quad(\rho\ne\mu),
 \qquad
 Z_\mu^{(0)}\coloneqq\bigl(Z_{\rho\mu}^{(0)}\bigr)_{\rho\ne\mu}.
\]
Thus $M_\mu^{(0)}=S_\mu^{(0)}-\max_{\rho\ne\mu}Z_{\rho\mu}^{(0)}$. These are the
target and competitor scores of~\cref{lem:self,lem:competitor-scores} for the
base block, which satisfies the assumptions of both lemmas
by~\eqref{eq:repair-dimensions}.
Choose $\gamma\in(0,E_\kappa/2)$ sufficiently small that
\begin{equation}
 \lambda\coloneqq \frac{(E_\kappa-2\gamma)^2}{2\alpha E_\kappa}>1;
 \label{eq:base-tail-choice}
\end{equation}
this is possible by \eqref{eq:sharp-ach-condition}. Then
\begin{equation}
 \bbP(\mu\in\mathcal B)
 \le \bbP\bigl(S_\mu^{(0)}\le E_\kappa-\gamma\bigr)
 +\bbP\!\left(\max_{\rho\ne\mu}Z_{\rho\mu}^{(0)}\ge E_\kappa-2\gamma\right).
 \label{eq:bad-set-split}
\end{equation}
The quantitative bound~\eqref{eq:quantitative-target-score}
of~\cref{lem:self} gives, uniformly in $\mu$,
\begin{equation}
 \bbP\bigl(S_\mu^{(0)}\le E_\kappa-\gamma\bigr)
 =O_\gamma\!\left(\frac1{\log p}\right).
 \label{eq:base-target-tail}
\end{equation}
For the competitors, fix $\lambda'\in(1,\lambda)$, put
$r_\star\coloneqq(E_\kappa-2\gamma)^2/(2\lambda')$, and define
\begin{equation}
 \mathcal H\coloneqq \left\{
 \max_{\mu\le p}\frac{\log p}{p-1}\norm{Z_\mu^{(0)}}^2\le r_\star
 \right\}.
 \label{eq:base-variance-event}
\end{equation}
Since $\lambda'<\lambda$, we have $r_\star>\alpha E_\kappa$,
so~\eqref{eq:competitor-radius} gives
\begin{equation}
 \bbP(\mathcal H)\longrightarrow1.
 \label{eq:base-variance-event-probability}
\end{equation}
On $\mathcal H$, $\norm{Z_\mu^{(0)}}^2\le r_\star(p-1)/\log p$ for every $\mu$.
Hence~\eqref{eq:competitor-tail}, applied with $t=E_\kappa-2\gamma$ and
$r=r_\star/\log p$, gives
\begin{equation}
\bbP\!\left(
 \max_{\rho\ne\mu}Z_{\rho\mu}^{(0)}\ge E_\kappa-2\gamma,\ \mathcal H
 \right)
\le(p-1)\exp\!\left(-\frac{(E_\kappa-2\gamma)^2\log p}{2r_\star}\right)
\le p^{1-\lambda'}.
\label{eq:base-competitor-tail-explicit}
\end{equation}
Combining \eqref{eq:bad-set-split}, \eqref{eq:base-target-tail}, and
\eqref{eq:base-competitor-tail-explicit} gives, uniformly in $\mu$,
\begin{equation}
 \bbP(\mu\in\mathcal B,\ \mathcal H)
 \le O\!\left(\frac1{\log p}\right)+p^{1-\lambda'}.
 \label{eq:fixed-target-bad-explicit}
\end{equation}
Summing over $\mu$ and applying Markov's inequality,
\begin{align}
\bbP\!\left(
 \abs{\mathcal B}>\frac{p}{\sqrt{\log p}},\ \mathcal H
 \right)
&\le \frac{\sqrt{\log p}}{p}\sum_{\mu\le p}\bbP(\mu\in\mathcal B,\ \mathcal H) \notag\\
 &\le O\!\left((\log p)^{-1/2}\right)
      +\sqrt{\log p}\,p^{1-\lambda'}=o(1).
 \label{eq:bad-set-markov}
\end{align}
Together with \eqref{eq:base-variance-event-probability}, this proves
\begin{equation}
 \abs{\mathcal B}\le\frac{p}{\sqrt{\log p}}
 \qquad\text{with probability tending to one}.
 \label{eq:bad-set-small}
\end{equation}

\medskip\noindent\textbf{Uniform control of the base margins.}

Put $B_0\coloneqq 3\sqrt{r_\star}$.\footnote{The union bound below ranges over fewer
than $p^2$ scores and vanishes for any fixed $B_0>2\sqrt{r_\star}$.
The factor $3$ gives the exponent $2-9/2=-5/2$.}
The uniform target-score bound~\eqref{eq:crude-uniform} of
Lemma~\ref{lem:self} gives
\begin{equation}
 \bbP\left(\min_{\mu\le p}S_\mu^{(0)}<-A_0\right)\longrightarrow0.
 \label{eq:base-target-fixed-bound}
\end{equation}
Moreover, \eqref{eq:competitor-tail} with $t=B_0$ and $r=r_\star/\log p$,
together with a union bound over $\mu$, gives
\begin{align}
 \bbP\left(
  \max_{\substack{\mu\le p\\\rho\ne\mu}}Z_{\rho\mu}^{(0)}>B_0,
  \ \mathcal H\right)
 &\le p(p-1)\exp\left(-\frac{B_0^2\log p}{2r_\star}\right) \notag\\
 &\le p^{-5/2}=o(1).
 \label{eq:base-competitor-fixed-bound}
\end{align}
Since $\bbP(\mathcal H^c)\to0$, the competitor bound also holds without
intersecting with $\mathcal H$. Together with \eqref{eq:bad-set-small} and
\eqref{eq:base-target-fixed-bound}, it follows that the base-block event
\begin{equation}
 \mathcal G\coloneqq \left\{\abs{\mathcal B}\le\frac p{\sqrt{\log p}}\right\}
 \cap\left\{\min_{\mu\le p}S_\mu^{(0)}\ge-A_0\right\}
 \cap\left\{\max_{\substack{\mu\le p\\\rho\ne\mu}}
                         Z_{\rho\mu}^{(0)}\le B_0\right\}
 \label{eq:base-good-event}
\end{equation}
satisfies $\bbP(\mathcal G)\to1$.  Notice that $\mathcal G$ is measurable
with respect to the first coordinate block of the data (the inputs, the
targets and, in the decoupled problem, the competitors), and is therefore
independent of all the variables in the repair block.

\medskip\noindent\textbf{Independent sparse repair.}\par\nopagebreak

For $\mu\in[p]$, define the nonnegative base-margin deficit
\begin{equation}
 w_\mu\coloneqq (2\gamma-M_\mu^{(0)})\mathbf{1}_{\{\mu\in\mathcal B\}},
 \qquad
 W_1\coloneqq \sum_{\mu=1}^p w_\mu u_\mu^1(e_\mu^1)^{\top}.
 \label{eq:repair-matrix}
\end{equation}
On $\mathcal G$, $M_\mu^{(0)}=S_\mu^{(0)}-\max_{\rho\ne\mu}Z_{\rho\mu}^{(0)}\ge-A_0-B_0$ for every $\mu$.
Set the fixed constant $\beta\coloneqq 2\gamma+A_0+B_0$. Then
\begin{equation}
 \norm{w}_\infty\le\beta,
 \qquad
 \norm{w}_2\le\beta\sqrt{\abs{\mathcal B}}
 \le\frac{\beta\sqrt p}{(\log p)^{1/4}}.
 \label{eq:repair-weight-bounds}
\end{equation}

Condition on all first-block variables.  The weights
\eqref{eq:repair-matrix} are then deterministic, and the repair block is
independent of them.  On $\mathcal G$, apply Lemma~\ref{lem:weighted-repair},
for the problem at hand, to the repair block, i.e.\ with $n=d_1$ and with
$(e_\nu,u_\nu,u_\rho^{(\mu)})$ replaced by
$(e_\nu^1,u_\nu^1,u_\rho^{(\mu),1})$, so that $R(w)=W_1$, and with $\beta$
as defined above. Since the error scale $\Delta_{d_1}(w,\beta)$ is increasing
in $\norm w_2$, the bound \eqref{eq:repair-weight-bounds} yields
$C_0\Delta_{d_1}(w,\beta)\le\bar\Delta_p$, where
\begin{align}
 \bar\Delta_p
 &\coloneqq C_0\beta\sqrt{\frac{\log p}{d_1}}
   +\frac{C_0\beta\sqrt{p\log p}}
          {d_1(\log p)^{1/4}} \notag\\
 &=o(1).
 \label{eq:repair-error-vanishes}
\end{align}
For the last equality, use \eqref{eq:repair-dimensions} and
$p\log p/d^2\to\alpha$: the final term is
$O((\log p)^{-1/8})$, while the first term vanishes more rapidly.
Since the failure bound of Lemma~\ref{lem:weighted-repair} does not depend on
the deterministic weights or on $\beta$, it holds uniformly over the
conditioning. Thus, with conditional probability $1-o(1)$ on $\mathcal G$,
\begin{align}
 \max_{\mu\le p}
 \abs{(u_\mu^1)^{\top}W_1e_\mu^1-w_\mu}
 &\le\bar\Delta_p, \label{eq:repair-target-control}\\
 \max_{\substack{\mu\le p\\\rho\ne\mu}}
 \abs{(u_\rho^{(\mu),1})^{\top}W_1e_\mu^1}
 &\le\bar\Delta_p. \label{eq:repair-competitor-control}
\end{align}

Embed the two matrices into the original coordinates and compensate for the
block normalization by setting
\begin{equation}
 W_\star\coloneqq \frac d{d_0}P_0^{\top}W_0P_0
          +\frac d{d_1}P_1^{\top}W_1P_1.
 \label{eq:final-memory}
\end{equation}
The two summands have orthogonal coordinate supports.  Hence
\begin{equation}
 \rank(W_\star)\le\rank(W_0)+\rank(W_1)
 \le m_0+d_1=m.
 \label{eq:final-rank}
\end{equation}
Since $P_ax=\sqrt{d_a/d}\,x^a$ for every input, target and competitor $x$,
each block contributes
$\frac d{d_a}(P_ax)^{\top}W_a(P_ae_\mu)=(x^a)^{\top}W_ae_\mu^a$ to the score
of $x$ under $W_\star$. Hence, for every $\mu$ and $\rho\ne\mu$, we have the
exact margin identity
\begin{align}
 u_\mu^{\top}W_\star e_\mu
 -(u_\rho^{(\mu)})^{\top}W_\star e_\mu
 &=(S_\mu^{(0)}-Z_{\rho\mu}^{(0)})
   +(u_\mu^1)^{\top}W_1e_\mu^1
   -(u_\rho^{(\mu),1})^{\top}W_1e_\mu^1 \notag\\
 &\ge M_\mu^{(0)}+w_\mu-2\bar\Delta_p.
 \label{eq:final-margin-bound}
\end{align}
For $\mu\notin\mathcal B$, the right-hand side is strictly positive for
all large $p$ because $M_\mu^{(0)}>\gamma$ and $w_\mu=0$.  For
$\mu\in\mathcal B$, it equals $2\gamma-2\bar\Delta_p>0$.  Therefore,
on an event of probability tending to one, $W_\star$ has rank at most $m$
and retrieves every association.
\end{proof}

\subsection{Sharp rank-constrained impossibility}
\label{subsec:rank_proof_impossibility}

This subsection concerns the decoupled problem only. We first bound the
summed competitor scores simultaneously for all $W$.

\begin{lemma}[Uniform lower bound for the summed competitor scores]
\label{lem:uniform-competitor-max}
For $W\in\bbR^{d\times d}$, write
\[
 D_E(W)\coloneqq \frac{\sqrt d}{p}\sum_{\mu=1}^p\norm{We_\mu}.
\]
With probability tending to one, simultaneously for every
$W\in\bbR^{d\times d}$,
\begin{equation}
 \sum_{\mu=1}^p\max_{\rho\ne\mu}
 (u_\rho^{(\mu)})^{\top}We_\mu
 \ge(1-o(1))\frac{p\sqrt{2\log p}}d\,D_E(W),
 \label{eq:uniform-competitor-max-bound}
\end{equation}
where the $o(1)$ is deterministic and uniform in $W$.
\end{lemma}

\begin{proof}
\medskip\noindent\textbf{1. The computation for a fixed matrix.}

Let $S(W)$ denote the summed competitor score:
\[
 S(W)\coloneqq 
 \sum_{\mu=1}^p\max_{\rho\ne\mu}
 (u_\rho^{(\mu)})^{\top}We_\mu.
\]
Both $S(W)$ and $D_E(W)$ are positively homogeneous in
$W$, so it is enough to consider $\norm{W}_{\rF}=1$.  Fix such a
deterministic $W$ and condition on $(e_\mu)_{\mu\le p}$.  The
calculation below gives a strong estimate for this fixed $W$; the only
difficulty is to make it simultaneous over all $W$.

For each query, set
\begin{equation}
 \sigma_\mu\coloneqq \frac{\norm{We_\mu}}{\sqrt d},
 \qquad
 \sum_{\mu=1}^p \sigma_\mu=\frac{p}{d}D_E(W).
 \label{eq:imp-query-weights}
\end{equation}
Since $u_\rho^{(\mu)}\sim\mcN(0,\Id_d/d)$ is independent of
$(e_\nu)_{\nu\le p}$,
each competitor score has, conditionally on $(e_\nu)_{\nu\le p}$, mean
zero and variance
\begin{equation}
 \operatorname{Var}\!\left(
 (u_\rho^{(\mu)})^{\top}We_\mu
 \,\middle|\,(e_\nu)_{\nu\le p}\right)
 =\frac{\norm{We_\mu}^2}{d}
 =\sigma_\mu^2.
 \label{eq:imp-score-variance}
\end{equation}
When $\sigma_\mu>0$, the normalized scores
$(u_\rho^{(\mu)})^{\top}We_\mu/\sigma_\mu$ are therefore independent
standard Gaussians as $\rho$ varies.  Define their maximum by
\[
 Y_\mu(W)\coloneqq \max_{\rho\ne\mu}
 \frac{(u_\rho^{(\mu)})^{\top}We_\mu}{\sigma_\mu}.
\]
If $\sigma_\mu=0$, every competitor score for query $\mu$ is zero.
For $\sigma_\mu>0$, since each query has its own competitors, the maxima
$Y_\mu(W)$ are also independent across $\mu$; this is where the decoupling
enters the impossibility proof.  Moreover,
\begin{equation}
 \max_{\rho\ne\mu}(u_\rho^{(\mu)})^{\top}We_\mu
 =\sigma_\mu Y_\mu(W)
 \qquad (\sigma_\mu>0).
 \label{eq:imp-score-factorization}
\end{equation}

Put
\begin{equation}
 t_p\coloneqq \left(1-\frac{4\log\log p}{\log p}\right)\sqrt{2\log p}.
 \label{eq:imp-parameters}
\end{equation}
For every $\mu$ with $\sigma_\mu>0$, Mills' lower bound
$1-\Phi(t)\ge c t^{-1}e^{-t^2/2}$, valid for $t\ge1$, yields
\begin{align}
 \bbP\!\left(Y_\mu(W)<t_p\,\middle|\,(e_\nu)_{\nu\le p}\right)
 &=\Phi(t_p)^{p-1}
 \le\exp\{-(p-1)(1-\Phi(t_p))\}\notag\\
 &\le\exp\{-(\log p)^{7}\}
 \label{eq:imp-one-block-tail}
\end{align}
for all sufficiently large $p$.\footnote{The first inequality uses
$x\le e^{x-1}$ for $x\in[0,1]$, with $x=\Phi(t_p)$.}  Indeed,
\[
 p\,e^{-t_p^2/2}
 =p^{1-(1-4\log\log p/\log p)^2}
 =(\log p)^{8-o(1)},
\]
so that
$(p-1)(1-\Phi(t_p))\ge c(\log p)^{8-o(1)}/\sqrt{2\log p}=(\log p)^{15/2-o(1)}$,
which exceeds $(\log p)^7$ for all large $p$.  Only the crude exponent $7$ is
used below.

The next step explains why we aggregate over queries. 

\medskip\noindent\textbf{2. The uniformity budget.}

Define the unit Frobenius sphere and the net mesh by
\begin{equation}
 \mathcal S\coloneqq \{W\in\bbR^{d\times d}:\norm{W}_{\rF}=1\},
 \qquad
 \epsilon_p\coloneqq \frac{1}{\sqrt d\,(\log p)^{3/2}}.
 \label{eq:imp-net-mesh}
\end{equation}
Choose a deterministic Frobenius $\epsilon_p$-net $\mathcal N_p$ of
$\mathcal S$.  The Euclidean volumetric bound in dimension $d^2$ gives
\begin{equation}
 \abs{\mathcal N_p}\le
 \left(\frac3{\epsilon_p}\right)^{d^2},
 \qquad
 \log\abs{\mathcal N_p}
 =O(d^2\log d)
 =O(p(\log p)^2).
 \label{eq:matrix-sphere-net-size}
\end{equation}
Thus the single-query failure probability in
\eqref{eq:imp-one-block-tail}, whose exponent is only $(\log p)^7$,
cannot be union-bounded over $\mathcal N_p$.  We need a failure exponent
of order $p$ times a power of $\log p$.  This is why we aggregate the
maxima across the $p$ queries.

\medskip\noindent\textbf{3. Queries with a small maximum have negligible
total weight.}

Let
\[
 \mathcal I(W)\coloneqq \{\mu:\sigma_\mu>0,\ Y_\mu(W)<t_p\}.
\]
On the event of Lemma~\ref{lem:key-frame}, equations
\eqref{eq:key-frame-first-moment} and
\eqref{eq:key-frame-second-moment} give
\begin{equation}
 \sum_{\mu=1}^p \sigma_\mu=\frac{p}{d}D_E(W)
 \ge\frac{p}{d\sqrt{2\pi}},
 \qquad
 \sum_{\mu=1}^p \sigma_\mu^2
 =\frac1d\sum_{\mu=1}^p\norm{We_\mu}^2
 \le\frac{2p}{d^2}.
 \label{eq:imp-weight-bounds}
\end{equation}
Therefore, if the indices in $\mathcal I(W)$ carry more than a
$1/\log p$ fraction of the total weight, then Cauchy--Schwarz gives
\[
 \frac{p}{d\sqrt{2\pi}\log p}
 <\sum_{\mu\in\mathcal I(W)}\sigma_\mu
 \le\sqrt{\abs{\mathcal I(W)}}
      \left(\sum_{\mu=1}^p \sigma_\mu^2\right)^{1/2}
 \le\frac{\sqrt{2p\abs{\mathcal I(W)}}}{d}.
\]
Hence
\begin{equation}
 \abs{\mathcal I(W)}\ge\frac{c_1p}{(\log p)^2}.
 \label{eq:imp-many-bad-indices}
\end{equation}

For any fixed set $I$ of indices with positive weights, conditional
independence and \eqref{eq:imp-one-block-tail} give
\[
 \bbP\bigl(Y_\mu(W)<t_p\text{ for every }\mu\in I
 \,\bigm|\,(e_\mu)_{\mu\le p}\bigr)\le\exp\{-(\log p)^7\abs I\}.
\]
For realizations of $(e_\mu)_{\mu\le p}$ on the event of
Lemma~\ref{lem:key-frame}, a
union bound over the at most $2^p$ possible sets $I$, together with
\eqref{eq:imp-many-bad-indices}, therefore yields
\begin{align}
 \bbP\left(
 \sum_{\mu\in\mathcal I(W)}\sigma_\mu
 >\frac1{\log p}\sum_{\mu=1}^p \sigma_\mu
 \ \middle|\ (e_\mu)_{\mu\le p}\right)
 &\le2^p\exp\{-c_1p(\log p)^5\}\notag\\
 &\le\exp\{-c_2p(\log p)^{5}\}.
 \label{eq:imp-weighted-tail}
\end{align}
Thus, for each fixed $W$ and conditional on such a realization of
$(e_\mu)_{\mu\le p}$,
with probability at least $1-\exp\{-c_2p(\log p)^5\}$ over the competitors,
\[
 \sum_{\mu\in\mathcal I(W)}\sigma_\mu
 \le\frac1{\log p}\sum_{\mu=1}^p \sigma_\mu
 =\frac{p}{d\log p}D_E(W).
\]

\medskip\noindent\textbf{4. Discarding the bad queries.}

By definition, for $\mu\in\mathcal I(W)$, we cannot use $Y_\mu(W)\ge t_p$.  We want to
omit these summands from a lower bound on $S(W)$, which is valid only if
they are nonnegative.  We prove the stronger statement that every
competitor maximum is nonnegative, simultaneously for all $\mu$ and
$W$.  For each fixed $\mu$, the $p-1$ points $(u_\rho^{(\mu)})_{\rho\ne\mu}$
are i.i.d.\ with a symmetric law and almost surely in general position.
Wendel's theorem~\citep{wendel1962problem} computes the probability that they
lie in a common half-space through the origin, which almost surely is the
event that $0$ is not in their convex hull:
\begin{align}
 \bbP\left(0\notin\operatorname{conv}
 \{u_\rho^{(\mu)}:\rho\ne\mu\}\right)
 =2^{-p+2}\sum_{r=0}^{d-1}\binom{p-2}{r}.
 \label{eq:wendel-formula}
\end{align}
Since $\sum_{r<d}\binom{p-2}{r}\le(ep/d)^d$ and $d\log(ep/d)=o(p)$ because
$d/p\to0$, the right-hand side is at most $2^{-p+2}e^{o(p)}\le e^{-cp}$ for all
large $p$.  A union
bound over $\mu$ therefore shows that, with probability tending to one,
\begin{equation}
 \max_{\rho\ne\mu}(u_\rho^{(\mu)})^{\top}x\ge0
 \quad\text{for every }\mu\in[p]\text{ and every }x\in\bbR^d.
 \label{eq:imp-nonnegative}
\end{equation}
Indeed, if zero is a convex combination of the competitors of query $\mu$,
their inner products with any $x$ cannot all be negative.  Notice that
\eqref{eq:imp-nonnegative} is already simultaneous in $x$, so it
requires no union bound over $W$.

On \eqref{eq:imp-nonnegative}, we may omit the summands indexed by
$\mathcal I(W)$.  For each remaining index with $\sigma_\mu>0$,
\eqref{eq:imp-score-factorization} is at least $t_p\sigma_\mu$; if
$\sigma_\mu=0$, the summand is zero.  Together with
\eqref{eq:imp-weighted-tail}, this gives
\begin{equation}
 S(W)
 \ge t_p\sum_{\mu\notin\mathcal I(W)}\sigma_\mu
 \ge\frac{t_p(1-1/\log p)p}{d}D_E(W).
 \label{eq:imp-fixed-W}
\end{equation}

\medskip\noindent\textbf{5. The union bound and transfer from the net.}

The exponent $p(\log p)^5$ in \eqref{eq:imp-weighted-tail} now exceeds
the net budget in \eqref{eq:matrix-sphere-net-size}. Indeed, condition on
inputs $(e_\mu)_{\mu\le p}$ in the event of Lemma~\ref{lem:key-frame}. For each
$V\in\mathcal N_p$, let $\mathcal F_V$ be the event
that~\eqref{eq:imp-nonnegative} holds but~\eqref{eq:imp-fixed-W} fails
for $V$. By Step~4, $\mathcal F_V$ is contained in the event
in~\eqref{eq:imp-weighted-tail}, so a union bound gives
\[
 \bbP\Bigl(\bigcup_{V\in\mathcal N_p}\mathcal F_V
 \Bigm|(e_\mu)_{\mu\le p}\Bigr)
 \le\abs{\mathcal N_p}\,e^{-c_2p(\log p)^5}
 \le e^{O(p(\log p)^2)-c_2p(\log p)^5},
\]
which tends to zero. Since the event of Lemma~\ref{lem:key-frame}
and~\eqref{eq:imp-nonnegative} have probability tending to one,
\eqref{eq:imp-fixed-W} holds simultaneously for every $V\in\mathcal N_p$
with probability tending to one.

It remains to transfer this estimate to all of $\mathcal S$.  First,
the triangle inequality and Lemma~\ref{lem:key-frame} give, for any $W,V$,
\begin{equation}
 \abs{D_E(W)-D_E(V)}
 \le D_E(W-V)
 \le\sqrt2\norm{W-V}_{\rF}.
 \label{eq:DE-lipschitz}
\end{equation}
For the competitor scores, Gaussian norm concentration and a union bound
over fewer than $p^2$ vectors give
\begin{equation}
 \max_{\mu,\rho:\rho\ne\mu}\norm{u_\rho^{(\mu)}}\le2
 \label{eq:imp-max-norm}
\end{equation}
with probability tending to one; here we use $\log p=o(d)$.  On this
event,
\begin{align}
 \abs{S(W)-S(V)}
 &\le\sum_{\mu=1}^p\max_{\rho\ne\mu}
 \abs{(u_\rho^{(\mu)})^{\top}(W-V)e_\mu}\notag\\
 &\le2\sum_{\mu=1}^p\norm{(W-V)e_\mu}\notag\\
 &=\frac{2p}{\sqrt d}D_E(W-V)\notag\\
 &\le\frac{2\sqrt2p}{\sqrt d}\norm{W-V}_{\rF}.
 \label{eq:competitor-lipschitz}
\end{align}

For $W\in\mathcal S$, choose $V\in\mathcal N_p$ with
$\norm{W-V}_{\rF}\le\epsilon_p$.  Combining
\eqref{eq:imp-fixed-W}--\eqref{eq:competitor-lipschitz} gives
\begin{align*}
 S(W)
 &\ge S(V)
      -\frac{2\sqrt2p}{\sqrt d}\epsilon_p\\
 &\ge\frac{t_p(1-1/\log p)p}{d}D_E(V)
      -\frac{2\sqrt2p}{\sqrt d}\epsilon_p\\
 &\ge\frac{t_p(1-1/\log p)p}{d}
      \bigl(D_E(W)-\sqrt2\epsilon_p\bigr)
      -\frac{2\sqrt2p}{\sqrt d}\epsilon_p.
\end{align*}
By Lemma~\ref{lem:key-frame}, $D_E(W)\ge1/\sqrt{2\pi}$ on $\mathcal S$.
Dividing the last bound by $p\sqrt{\log p}\,D_E(W)/d$ gives
\begin{align*}
 \frac{dS(W)}{p\sqrt{\log p}\,D_E(W)}
 &\ge\frac{t_p(1-1/\log p)}{\sqrt{\log p}}
       \left(1-\frac{\sqrt2\epsilon_p}{D_E(W)}\right)
       -\frac{2\sqrt2\epsilon_p}{D_E(W)}
        \sqrt{\frac d{\log p}}\\
 &\ge\sqrt2-o(1),
\end{align*}
uniformly for $W\in\mathcal S$.  Indeed,
$t_p/\sqrt{\log p}\to\sqrt2$, $\epsilon_p\to0$, and
$\epsilon_p\sqrt{d/\log p}=1/(\log p)^2\to0$; the lower bound on
$D_E(W)$ makes both error terms uniform.  This proves
\eqref{eq:uniform-competitor-max-bound} on $\mathcal S$.  Positive homogeneity
extends it to every nonzero $W$; the case $W=0$ is immediate.
\end{proof}

We next bound the summed target score simultaneously for all $W$ of rank at
most $m$.

\begin{lemma}[Uniform rank-$m$ bound for $\langle C,W\rangle_{\rF}$]
\label{lem:target-quotient}
With probability tending to one, simultaneously for every $W$ with
$\rank(W)\le m$,
\begin{equation}
  \frac{1}{\sqrt{p}}\sum_{\mu=1}^p u_\mu^\top W e_\mu=\frac{\langle C,W\rangle_{\rF}}{\sqrt p}
 \le\left(
       \frac{\norm{T_m(C)}_{\rF}}{\sqrt p}+o(1)
     \right)D_E(W),
 \label{eq:target-quotient}
\end{equation}
where the $o(1)$ is deterministic and uniform in $W$.
\end{lemma}

\begin{proof}
Both sides of \eqref{eq:target-quotient} scale by $t$ under
$W\mapsto tW$ for $t>0$.  It is therefore enough to prove the inequality
for $\norm{W}_{\rF}=1$.  Write an ordered singular value decomposition
\[
 W=\sum_{j=1}^d s_ja_jb_j^{\top},
 \qquad s_1\ge\cdots\ge s_d\ge0,
\]
where $s_j=0$ for $j>\rank(W)$.  Let $k=\lceil(\log p)^6\rceil$.  For
all sufficiently large $p$, $k<m$.  Split
\[
 W_{\le k}\coloneqq\sum_{j\le k}s_ja_jb_j^{\top},
 \qquad W_{>k}\coloneqq W-W_{\le k}=\sum_{j>k}s_ja_jb_j^{\top}.
\]
The column spaces of $W_{\le k}$ and $W_{>k}$ are orthogonal.  Hence, pointwise in $e$,
\begin{equation}
 \norm{We}^2
 =\norm{W_{\le k}e}^2+\norm{W_{>k}e}^2,
 \qquad D_E(W)\ge D_E(W_{>k}).
 \label{eq:head-tail-orthogonality}
\end{equation}
The largest singular value of $W_{>k}$ is $s_{k+1}$.  Moreover, monotonicity
and $\sum_{j=1}^d s_j^2=\norm{W}_{\rF}^2=1$ give
$k s_{k+1}^2\le\sum_{j=1}^k s_j^2\le1$.  Therefore,
\begin{equation}
 \norm{W_{>k}}_{\op}=s_{k+1}\le k^{-1/2}\le(\log p)^{-3}.
 \label{eq:tail-operator-bound}
\end{equation}
By Lemmas~\ref{lem:gaussian-correlation-memory}, \ref{lem:key-frame}, and
\ref{lem:stable-rank-frame}, all the estimates below hold simultaneously on
an event of probability tending to one.  We work on this event.  In
particular,
\[
 \norm{C}_{\op}\le3\sqrt{\frac p d},
 \qquad
 D_E(W)\ge\frac{\norm{W}_{\rF}}{\sqrt{2\pi}}=\frac1{\sqrt{2\pi}}.
\]
We distinguish two cases according to the size of the tail $W_{>k}$.

\smallskip\noindent\emph{Case 1: $\norm{W_{>k}}_{\rF}\ge(\log p)^{-1}$.}
By \eqref{eq:tail-operator-bound},
\[
 \frac{\norm{W_{>k}}_{\op}}{\norm{W_{>k}}_{\rF}}
 \le\frac{(\log p)^{-3}}{(\log p)^{-1}}
 =(\log p)^{-2}.
\]
Thus $W_{>k}/\norm{W_{>k}}_{\rF}$ has Frobenius norm one and operator norm at most
$(\log p)^{-2}$.  Lemma~\ref{lem:stable-rank-frame}, the homogeneity of
$D_E$, and \eqref{eq:head-tail-orthogonality} therefore give
\begin{equation}
 D_E(W)\ge D_E(W_{>k})
 =\norm{W_{>k}}_{\rF}D_E\!\left(\frac{W_{>k}}{\norm{W_{>k}}_{\rF}}\right)
 \ge(1-o(1))\norm{W_{>k}}_{\rF}.
 \label{eq:tail-DE-lower}
\end{equation}
We next bound separately the contributions of $W_{\le k}$ and $W_{>k}$.
For the head, whose singular values are $s_1,\ldots,s_k$, von Neumann's trace
inequality, $s_j(C)\le\norm{C}_{\op}$, and Cauchy--Schwarz give
\[
 \abs{\langle C,W_{\le k}\rangle_{\rF}}
 \le\sum_{j=1}^{k}s_j(C)s_j
 \le\norm{C}_{\op}\sqrt k\left(\sum_{j=1}^k s_j^2\right)^{1/2}
 \le3\sqrt{\frac{pk}{d}},
\]
where we used $\sum_{j\le k}s_j^2\le\norm{W}_{\rF}^2=1$ in the last step.
Moreover, $\rank(W_{>k})\le\rank(W)\le m$.  Von Neumann's trace inequality,
followed by Cauchy--Schwarz, gives
\begin{align*}
 \abs{\langle C,W_{>k}\rangle_{\rF}}
 &\le \sum_{j=1}^{m}s_j(C)s_j(W_{>k})\\
 &\le\left(\sum_{j=1}^m s_j(C)^2\right)^{1/2}
       \norm{W_{>k}}_{\rF}
 =\norm{T_m(C)}_{\rF}\norm{W_{>k}}_{\rF}.
\end{align*}
Adding these two estimates and using $W=W_{\le k}+W_{>k}$, we obtain
\begin{equation}
 \langle C,W\rangle_{\rF}
 \le \abs{\langle C,W_{\le k}\rangle_{\rF}}
    +\abs{\langle C,W_{>k}\rangle_{\rF}}
 \le3\sqrt{\frac{pk}{d}}
   +\norm{T_m(C)}_{\rF}\norm{W_{>k}}_{\rF}.
 \label{eq:target-head-tail-bound}
\end{equation}
Dividing by \eqref{eq:tail-DE-lower} gives
\begin{align}
 \frac{\langle C,W\rangle_{\rF}}{D_E(W)}
 &\le
 \frac{3\sqrt{pk/d}}{(1-o(1))\norm{W_{>k}}_{\rF}}
 +\frac{\norm{T_m(C)}_{\rF}}{1-o(1)}\notag\\
 &\le\norm{T_m(C)}_{\rF}+o(\sqrt p).
 \label{eq:target-large-tail-case}
\end{align}
Indeed, $\norm{W_{>k}}_{\rF}^{-1}\le\log p$ and
$\sqrt k=O((\log p)^3)$ make the first term
$O(\sqrt{p/d}(\log p)^4)=o(\sqrt p)$.  Moreover,
\[
 \norm{T_m(C)}_{\rF}
 \le\sqrt m\norm{C}_{\op}
 \le3\sqrt p,
\]
so replacing its factor $(1-o(1))^{-1}$ by $1$ costs only $o(\sqrt p)$.
Here $(\log p)^4=o(\sqrt d)$ since
$p\log p/d^2\to\alpha\in(0,\infty)$.

\smallskip\noindent\emph{Case 2: $\norm{W_{>k}}_{\rF}<(\log p)^{-1}$.}
In this case we do not need the stable-rank estimate for $W_{>k}$.  Indeed,
\eqref{eq:target-head-tail-bound}, $\sqrt k=O((\log p)^3)$, and
$\norm{T_m(C)}_{\rF}\le3\sqrt p$ give
\[
 \langle C,W\rangle_{\rF}
 \le O\!\left(\sqrt{\frac p d}(\log p)^3\right)
      +O(\sqrt p/\log p)
 =o(\sqrt p).
\]
The displayed $o(\sqrt p)$ is deterministic and uniform over the matrices in
this case.  The input-frame bound of Lemma~\ref{lem:key-frame} and $\norm{W}_{\rF}=1$ give
$D_E(W)\ge1/\sqrt{2\pi}$.
Dividing by this lower bound yields
\[
 \frac{\langle C,W\rangle_{\rF}}{D_E(W)}
 \le\sqrt{2\pi}\,o(\sqrt p)=o(\sqrt p).
\]
Since $\norm{T_m(C)}_{\rF}\ge0$, this implies
\eqref{eq:target-large-tail-case} also in Case~2.  The errors in both cases are
deterministic and uniform in $W$; after division by $\sqrt p$, they give the
$o(1)$ in \eqref{eq:target-quotient}.  Rescaling $W$ proves the claim; $W=0$
is immediate.
\end{proof}

Summing the retrieval constraints over the queries gives a necessary condition
for feasibility, and the two lemmas show that it fails for every $W$ of rank at
most $m$ when $\alpha>E_\kappa/2$.

\begin{theorem}[Rank-constrained impossibility]
\label{thm:sharp-impossibility}
Suppose that
\begin{equation}
 \alpha>\frac{E_\kappa}{2}.
 \label{eq:sharp-impossibility-condition}
\end{equation}
Then
\[
 \bbP\bigl(\PR^{\DP}_{p,d,m}\bigr)\longrightarrow0.
\]
\end{theorem}

\begin{proof}
Suppose that $W$ is feasible.  It must be nonzero because the inequalities
in \eqref{eq:model} are strict.  For each query $\mu$, feasibility gives
\[
 u_\mu^{\top}We_\mu>
 \max_{\rho\ne\mu}
 (u_\rho^{(\mu)})^{\top}We_\mu.
\]
Summing over the queries and using the definition of $C$ yields the
deterministic necessary condition
\begin{equation}
 \langle C,W\rangle_{\rF}
 >\sum_{\mu=1}^p\max_{\rho\ne\mu}
 (u_\rho^{(\mu)})^{\top}We_\mu.
 \label{eq:summed-feasibility}
\end{equation}

By~\eqref{eq:sharp-impossibility-condition}, we may fix $\delta>0$ with
$\sqrt{E_\kappa}+\delta<\sqrt{2\alpha}-\delta$. By~\eqref{eq:energy},
$\norm{T_m(C)}_{\rF}/\sqrt p\le\sqrt{E_\kappa}+\delta/2$ with probability
tending to one. On this event, Lemma~\ref{lem:target-quotient} gives, for all
large $p$ and simultaneously for every matrix $W$ of rank at most $m$,
\begin{equation}
 \langle C,W\rangle_{\rF}
 \le\sqrt p\left(\sqrt{E_\kappa}+\delta\right)D_E(W).
 \label{eq:summed-target-upper}
\end{equation}
By Lemma~\ref{lem:uniform-competitor-max} and $p\log p/d^2\to\alpha$, the right-hand
side of~\eqref{eq:summed-feasibility} satisfies, for all large $p$,
\begin{align}
 \sum_{\mu=1}^p\max_{\rho\ne\mu}
 (u_\rho^{(\mu)})^{\top}We_\mu
 &\ge(1-o(1))\frac{p\sqrt{2\log p}}dD_E(W)\notag\\
 &\ge\sqrt p\left(\sqrt{2\alpha}-\delta\right)D_E(W).
 \label{eq:summed-competitor-lower}
\end{align}
By Lemma~\ref{lem:key-frame}, $D_E(W)\ge\norm{W}_{\rF}/\sqrt{2\pi}>0$, so the
lower bound~\eqref{eq:summed-competitor-lower} is strictly larger than the
upper bound~\eqref{eq:summed-target-upper}. This
contradicts~\eqref{eq:summed-feasibility}. The events used above hold
simultaneously with probability tending to one, which proves the theorem.
\end{proof}

\begin{proof}[Proof of~\cref{thm:rank_capacity}]
The first case of~\eqref{eq:rank_capacity_restated}, together with the
statement for~\eqref{eq:OP}, is~\cref{thm:sharp-achievability}, and the second
case is~\cref{thm:sharp-impossibility}.
\end{proof}

\subsection{Technical lemmas}
\label{subsec:rank_proof_lemmas}

\begin{lemma}[Diagonal of a Haar-rotated matrix]
\label{lem:haar-diagonal}
There is a universal constant $C_1<\infty$ such that the following holds for
all $p\ge3$. Let $H\in\bbR^{p\times p}$ be symmetric and deterministic, let
$Q$ be Haar distributed on $O(p)$, and let $q_\mu$ denote the $\mu$th column
of $Q$. Then,
with probability at least $1-4/p$,
\begin{equation}
 \max_{\mu\le p}\abs{q_\mu^{\top}Hq_\mu-\frac{\operatorname{tr}H}{p}}
 \le C_1\,\frac{\norm{H}_{\rF}\sqrt{\log p}+\norm{H}_{\op}\log p}{p}.
 \label{eq:haar-diagonal-bound}
\end{equation}
\end{lemma}

\begin{proof}
The constants $C,c,c'$ below are universal. Fix $\mu$. Since $q_\mu$ is
uniform on $\mathbb S^{p-1}$, we may write $q_\mu=g/\norm g$ with
$g\sim\mcN(0,\Id_p)$. Then
\[
 q_\mu^{\top}Hq_\mu-\frac{\operatorname{tr}H}{p}
 =\frac{g^{\top}Hg-\operatorname{tr}H}{\norm g^2}
 +\frac{\operatorname{tr}H}{p}\,\frac{p-\norm g^2}{\norm g^2}.
\]
Let $t\ge1$. The Hanson--Wright inequality~\citep{rudelson2013hanson} and
Bernstein's inequality \citep[Theorem~2.8.1]{vershynin2018high} give, except
with probability $4e^{-ct}$,
\[
 \abs{g^{\top}Hg-\operatorname{tr}H}
 \le C\bigl(\norm H_{\rF}\sqrt t+\norm H_{\op}t\bigr),
 \qquad
 \abs{\norm g^2-p}\le C\bigl(\sqrt{pt}+t\bigr).
\]
If $t\le c'p$, with $c'$ small enough that $C(\sqrt{c'}+c')\le1/2$, the
second bound gives $\norm g^2\ge p/2$. Using also
$\abs{\operatorname{tr}H}\le\sqrt p\norm H_{\rF}$ and
$\abs{\operatorname{tr}H}\le p\norm H_{\op}$, both terms of the
decomposition are then at most
$2C(\norm H_{\rF}\sqrt t+\norm H_{\op}t)/p$. If $t>c'p$, the left-hand side
is at most $2\norm H_{\op}<(2/c')\norm H_{\op}t/p$. Hence, with
$C'\coloneqq\max(4C,2/c')$, for every $t\ge1$,
\[
 \bbP\!\left(
 \abs{q_\mu^{\top}Hq_\mu-\frac{\operatorname{tr}H}{p}}
 >C'\,\frac{\norm{H}_{\rF}\sqrt t+\norm{H}_{\op}t}{p}
 \right)\le 4e^{-ct}.
\]
Take $t=A\log p$ with $A\coloneqq\max(1,2/c)$, so that $t\ge1$ and
$\sqrt t\le A\sqrt{\log p}$. A union bound over $\mu\le p$, which does not
require the columns of $Q$ to be independent, gives
\eqref{eq:haar-diagonal-bound} with $C_1\coloneqq AC'$, except with
probability at most $4p\,e^{-cA\log p}\le4/p$.
\end{proof}

\begin{lemma}[Uniform control of an independent weighted repair]
\label{lem:weighted-repair}
Let $n=n_p$ satisfy $\log p=o(n)$. Let the inputs $(e_\nu)_{\nu\in[p]}$, the
targets $(u_\nu)_{\nu\in[p]}$ and the competitors
$\bigl(u_\rho^{(\mu)}\bigr)_{\mu,\rho\in[p],\,\rho\ne\mu}$ be as in either
problem of~\cref{subsec:rank_proof_setting}, but in dimension $n$: the inputs
and the targets are independent $\mcN(0,\Id_n/n)$ vectors, and the
competitors are either independent $\mcN(0,\Id_n/n)$ vectors, independent of
them (decoupled problem), or $u_\rho^{(\mu)}=u_\rho$ (original problem).
Let $w\in\bbR^p$ be deterministic, let
$\beta\ge\norm w_\infty$ be a deterministic bound, and put
\[
 R(w)\coloneqq \sum_{\nu=1}^p w_\nu u_\nu e_\nu^{\top}.
\]
Define the error scale
\begin{equation}
 \Delta_n(w,\beta)\coloneqq 
 \underbrace{\beta\sqrt{\frac{\log p}{n}}}
            _{\text{self-terms }\nu=\mu}
 \ +\
 \underbrace{\frac{\norm{w}_2\sqrt{\log p}}{n}}
            _{\text{cross-terms }\nu\ne\mu}.
 \label{eq:repair-error-scale}
\end{equation}
There are universal constants $C_0,c_0>0$ such that the two bounds
\begin{align}
 \max_{\mu\le p}
 \abs{u_\mu^{\top}R(w)e_\mu-w_\mu}
 &\le C_0\Delta_n(w,\beta), \label{eq:repair-diagonal}\\
 \max_{\substack{\mu\le p\\\rho\ne\mu}}
 \abs{(u_\rho^{(\mu)})^{\top}R(w)e_\mu}
 &\le C_0\Delta_n(w,\beta) \label{eq:repair-offdiagonal}
\end{align}
hold jointly with probability at least $1-C_0p^{-3}-C_0p^2e^{-c_0n}$.
\end{lemma}

\begin{proof}
The case $w=0$ is trivial, so assume $\norm w_\infty>0$, hence $\beta>0$.
The constants $C,c$ below are universal and may change from line to line.

\medskip\noindent\textbf{1. Exceptional norm events.}

Bernstein's inequality \citep[Theorem~2.8.1]{vershynin2018high}, applied to the
centered sub-exponential variables $g_i^2-1$ of a standard Gaussian vector
$g\in\bbR^n$, gives for every $t>0$
\[
 \bbP\left(\abs{\frac{\norm g^2}{n}-1}
 >C\left(\sqrt{\frac tn}+\frac tn\right)\right)\le2e^{-t}.
\]
Put
\[
 \delta_n\coloneqq C\left(\sqrt{\frac{\log p}n}+\frac{\log p}n\right)=o(1),
\]
the last equality by $\log p=o(n)$.  Taking $t=4\log p$ and a union bound over
the $2p$ vectors $(e_\mu)_{\mu\le p}$, $(u_\mu)_{\mu\le p}$ gives
\begin{equation}
 \max_{\mu\le p}\abs{\norm{e_\mu}^2-1}
 \ \vee\
 \max_{\mu\le p}\abs{\norm{u_\mu}^2-1}\le\delta_n
 \label{eq:repair-norms-ab}
\end{equation}
with failure probability at most $4p^{-3}$.  In particular, on this event
$\norm{e_\mu},\norm{u_\mu}\le2$ for every $\mu$ and all large $p$.
Separately, in the decoupled problem, the same tail applied to the fewer than
$p^2$ competitors gives
\begin{equation}
 \max_{\substack{\mu\le p\\\rho\ne\mu}}
 \norm{u_\rho^{(\mu)}}\le2
 \label{eq:repair-norms-v}
\end{equation}
with failure probability at most $Cp^2e^{-cn}$. In the original problem, the
competitors are targets, and~\eqref{eq:repair-norms-v} holds
on~\eqref{eq:repair-norms-ab}.

Conditioning on \eqref{eq:repair-norms-ab} or \eqref{eq:repair-norms-v} would
destroy the Gaussian conditional laws used below, so we never do so.  Instead,
each estimate in Steps 2 and 3 is a conditional tail bound evaluated at fixed
conditioning vectors of norm at most $2$; integrating it therefore controls the
probability of the corresponding bad event \emph{intersected with} that norm
restriction.  The probability that some norm restriction fails is added once,
in Step 4.

\medskip\noindent\textbf{2. Diagonal scores.}

Fix $\mu$ and separate the term $\nu=\mu$:
\begin{equation}
 u_\mu^{\top}R(w)e_\mu-w_\mu
 =\underbrace{w_\mu\bigl(\norm{e_\mu}^2\norm{u_\mu}^2-1\bigr)}_{\eqqcolon D_\mu}
 +\underbrace{\sum_{\nu\ne\mu}w_\nu
   (e_\mu^{\top}e_\nu)(u_\nu^{\top}u_\mu)}_{\eqqcolon X_\mu}.
 \label{eq:repair-diagonal-expansion}
\end{equation}
On the event \eqref{eq:repair-norms-ab},
\[
 \abs{\norm{e_\mu}^2\norm{u_\mu}^2-1}
 \le\abs{\norm{e_\mu}^2-1}\norm{u_\mu}^2
   +\abs{\norm{u_\mu}^2-1}
 \le3\delta_n,
\]
so $\max_{\mu\le p}\abs{D_\mu}\le3\beta\delta_n$, which is the self-term of
\eqref{eq:repair-error-scale}.

For $X_\mu$, condition on $(e_\mu,u_\mu)$.  The pairs
$(e_\nu,u_\nu)_{\nu\ne\mu}$ are independent of one another and of the
conditioning, so the summands are conditionally independent; they are centered
because $e_\nu$ and $u_\nu$ are independent, whence
\[
 \EE\bigl[(e_\mu^{\top}e_\nu)(u_\nu^{\top}u_\mu)
 \bigm|e_\mu,u_\mu\bigr]
 =\EE\bigl[e_\mu^{\top}e_\nu\bigm|e_\mu\bigr]\,
  \EE\bigl[u_\nu^{\top}u_\mu\bigm|u_\mu\bigr]=0.
\]
For conditioning vectors of norm at most $2$, the factors
$e_\mu^{\top}e_\nu\sim\mcN(0,\norm{e_\mu}^2/n)$ and
$u_\nu^{\top}u_\mu\sim\mcN(0,\norm{u_\mu}^2/n)$ are independent with
sub-gaussian norm at most $C/\sqrt n$, so each unweighted product has
sub-exponential norm at most $C/n$ \citep[Lemma~2.7.7]{vershynin2018high}.  The
weighted Bernstein inequality \citep[Theorem~2.8.2]{vershynin2018high}, whose
$\ell^\infty$ branch we relax using $\norm w_\infty\le\beta$, therefore gives
for every $t>0$
\begin{equation}
 \bbP\bigl(\abs{X_\mu}>t\bigm|e_\mu,u_\mu\bigr)
 \le2\exp\left[-c\min\left\{
 \frac{n^2t^2}{\norm{w}_2^2},\frac{nt}{\beta}\right\}\right].
 \label{eq:repair-diagonal-bernstein}
\end{equation}
Take
\begin{equation}
 \tau_p\coloneqq C\left(\frac{\norm{w}_2\sqrt{\log p}}n
               +\frac{\beta\log p}n\right).
 \label{eq:repair-threshold}
\end{equation}
The first summand is the cross-term scale in \eqref{eq:repair-error-scale};
the second is Bernstein's linear-tail scale.
Its first summand alone forces $n^2\tau_p^2/\norm w_2^2\ge C^2\log p$ and its
second alone forces $n\tau_p/\beta\ge C\log p$, so for $C$ large enough the
right-hand side of \eqref{eq:repair-diagonal-bernstein} at $t=\tau_p$ is at
most $2p^{-5}$.  Integrating and summing over the $p$ indices,
\[
 \bbP\Bigl(\max_{\mu\le p}\abs{X_\mu}>\tau_p
 \ \text{and \eqref{eq:repair-norms-ab} holds}\Bigr)\le2p^{-4}.
\]
Since $\log p/n=o(1)$, we have $\beta\log p/n\le
\beta\sqrt{\log p/n}$ for all large $p$. Thus
$3\beta\delta_n+\tau_p\le C\Delta_n(w,\beta)$, proving
\eqref{eq:repair-diagonal}.

\medskip\noindent\textbf{3. Off-diagonal scores.}

Fix $\mu$ and $\rho\ne\mu$, write $v=u_\rho^{(\mu)}$, and separate $\nu=\mu$
again:
\begin{equation}
 v^{\top}R(w)e_\mu
 =\underbrace{w_\mu\norm{e_\mu}^2(u_\mu^{\top}v)}_{\eqqcolon D_{\mu\rho}'}
 +\underbrace{\sum_{\nu\ne\mu}w_\nu
   (e_\mu^{\top}e_\nu)(u_\nu^{\top}v)}_{\eqqcolon X_{\mu\rho}'}.
 \label{eq:repair-offdiagonal-expansion}
\end{equation}
Condition on $(e_\mu,v)$ only, so that $u_\mu$ keeps its law $\mcN(0,\Id_n/n)$.
For conditioning vectors of norm at most $2$, $D_{\mu\rho}'$ is centered
Gaussian with standard deviation
$\abs{w_\mu}\norm{e_\mu}^2\norm v/\sqrt n\le8\beta/\sqrt n$, so
$\abs{D_{\mu\rho}'}\le C\beta\sqrt{\log p/n}$ with conditional failure
probability at most $2p^{-5}$.  This again reproduces the self-term of
\eqref{eq:repair-error-scale}.

In the decoupled problem, the competitors $\bigl(u_\rho^{(\mu)}\bigr)$ form a
family independent of $(e_\nu)_{\nu\le p}$ and $(u_\nu)_{\nu\le p}$.  Under
the same conditioning the summands of $X_{\mu\rho}'$ are therefore
independent and centered, with sub-exponential norm at most $C/n$, and the
analogue of \eqref{eq:repair-diagonal-bernstein} bounds $\abs{X_{\mu\rho}'}$
by $\tau_p$ with conditional failure probability at most $2p^{-5}$.

In the original problem, $v=u_\rho$ also enters $R(w)$, through the summand
$\nu=\rho$ of $X_{\mu\rho}'$, which equals
$w_\rho\norm{u_\rho}^2(e_\mu^{\top}e_\rho)$. Under the same conditioning,
$e_\rho$ keeps its law, so this summand is centered Gaussian with standard
deviation at most $8\beta/\sqrt n$, and it is bounded by
$C\beta\sqrt{\log p/n}$ exactly like $D_{\mu\rho}'$. The remaining summands,
$\nu\notin\{\mu,\rho\}$, are conditionally independent and centered, with
sub-exponential norm at most $C/n$, and are bounded by $\tau_p$ as in the
decoupled problem.

In both problems, integrating and summing over the fewer than $p^2$ pairs
$(\mu,\rho)$ costs at most $6p^{-3}$. Since
$2C\beta\sqrt{\log p/n}+\tau_p\le C\Delta_n(w,\beta)$, this proves
\eqref{eq:repair-offdiagonal}.

\medskip\noindent\textbf{4. Total failure probability.}

Adding the four contributions---$4p^{-3}$ from \eqref{eq:repair-norms-ab},
$Cp^2e^{-cn}$ from \eqref{eq:repair-norms-v}, $2p^{-4}$ from Step~2 and
$6p^{-3}$ from Step~3---gives the bound stated in the lemma. The constants
in this probability estimate do not depend on the deterministic weights $w$
or their bound $\beta$. Hence the estimate also applies after conditioning
on weights chosen independently of the Gaussian family above.
\end{proof}

\begin{lemma}[Uniform Gaussian input-frame bounds]
\label{lem:key-frame}
With probability tending to one, simultaneously for every
$W\in\bbR^{d\times d}$,
\begin{align}
 \frac{\norm{W}_{\rF}}{\sqrt{2\pi}}
 \le D_E(W)&\le\sqrt2\norm{W}_{\rF},
 \label{eq:key-frame-first-moment}\\
 \frac d p\sum_{\mu=1}^p\norm{We_\mu}^2
 &\le2\norm{W}_{\rF}^2.
 \label{eq:key-frame-second-moment}
\end{align}
Recall that
\[
 D_E(W)=\frac{\sqrt d}{p}\sum_{\mu=1}^p\norm{We_\mu}.
\]
\end{lemma}

\begin{proof}
For fixed $v\in\mathbb S^{d-1}$ and $e\sim\mcN(0,\Id_d/d)$,
$\sqrt d\,v^{\top}e\sim\mcN(0,1)$.  Set
\[
 F(v)\coloneqq \frac{\sqrt d}{p}\sum_{\mu=1}^p\abs{v^{\top}e_\mu},
 \qquad a\coloneqq \sqrt d\,\mathbb E_e\abs{v^{\top}e}=\sqrt{2/\pi}.
\]
For this fixed $v$, $F(v)$ is $p^{-1/2}$-Lipschitz as a function of the standard Gaussian vectors
$(\sqrt d\,e_\mu)_{\mu\le p}$.  Gaussian concentration
\citep[Theorem~5.2.2]{vershynin2018high} therefore gives
\begin{equation}
 \bbP\left(\abs{F(v)-a}>\frac a4\right)
 \le2e^{-cp}.
 \label{eq:key-fixed-direction}
\end{equation}
Let $\mathcal N$ be a fixed $\varepsilon_0$-net of
$\mathbb S^{d-1}$, where $\varepsilon_0=1/(8\sqrt\pi)=a/(8\sqrt2)$.  The volumetric
net bound \citep[Corollary~4.2.13]{vershynin2018high} gives
$\abs{\mathcal N}\le(3/\varepsilon_0)^d$.  A union bound in
\eqref{eq:key-fixed-direction} gives
\[
 \bbP\bigl(F(w)<3a/4\text{ for some }w\in\mathcal N\bigr)
 \le2(3/\varepsilon_0)^d e^{-cp}=o(1),
\]
where the last equality uses $p/d\to\infty$.  Thus, with probability
tending to one, $F(w)\ge3a/4$ for every net point $w$.

Recall $S_E=(d/p)\sum_{\mu=1}^p e_\mu e_\mu^{\top}$ from~\eqref{eq:SU}.
By~\eqref{eq:SU-concentration}, $\norm{S_E}_{\op}\le2$ for all large $p$,
except with probability $2e^{-d}$. On this covariance event, the reverse
triangle inequality and Cauchy--Schwarz give, for any
$v,w\in\mathbb S^{d-1}$,
\begin{align*}
 \abs{F(v)-F(w)}
 &\le\frac{\sqrt d}{p}\sum_{\mu=1}^p
 \abs{(v-w)^{\top}e_\mu}\\
 &\le\left(\frac d p\sum_{\mu=1}^p
       ((v-w)^{\top}e_\mu)^2\right)^{1/2}\\
 &=\sqrt{(v-w)^{\top}S_E(v-w)}\\
 &\le\sqrt{\norm{S_E}_{\op}}\,\norm{v-w}
 \le\sqrt2\norm{v-w}.
\end{align*}
On the intersection of the two high-probability events, take any
$v\in\mathbb S^{d-1}$ and choose $w\in\mathcal N$ with
$\norm{v-w}\le\varepsilon_0$.  Then
\[
 F(v)\ge F(w)-\sqrt2\varepsilon_0
 \ge\frac{3a}{4}-\frac a8
 =\frac{5a}{8}\ge\frac a2=\frac1{\sqrt{2\pi}}.
\]
The two bounds just proved hold on the same event, with probability tending
to one.

It remains to derive the matrix consequences.  Write
\[
 W^{\top}W=\sum_{r=1}^d s_r^2v_rv_r^{\top},
\]
where $(v_r)_{r\le d}$ is an orthonormal basis of right singular vectors.
If $W\ne0$, the numbers $\theta_r=s_r^2/\norm{W}_{\rF}^2$ are nonnegative and
sum to one.  The root-mean-square--arithmetic-mean inequality gives
\[
 \sqrt d\,\norm{We_\mu}
 =\norm{W}_{\rF}
   \left(\sum_r \theta_r d(v_r^{\top}e_\mu)^2\right)^{1/2}
 \ge\norm{W}_{\rF}\sum_r \theta_r\sqrt d\,
      \abs{v_r^{\top}e_\mu}.
\]
Averaging over $\mu$ and using the uniform bound on $F$ gives
$D_E(W)\ge\norm{W}_{\rF}\sum_r\theta_rF(v_r)\ge\norm{W}_{\rF}/\sqrt{2\pi}$,
the first inequality in~\eqref{eq:key-frame-first-moment}.
Moreover,
\begin{align*}
 \frac d p\sum_\mu\norm{We_\mu}^2
 &=\operatorname{tr}(W^{\top}WS_E)
 \le\norm{S_E}_{\op}\norm{W}_{\rF}^2
 \le2\norm{W}_{\rF}^2.
\end{align*}
This is \eqref{eq:key-frame-second-moment}.  Cauchy--Schwarz now gives
\[
 D_E(W)=\frac1p\sum_{\mu=1}^p\sqrt d\,\norm{We_\mu}
 \le\left(\frac d p\sum_{\mu=1}^p
 \norm{We_\mu}^2\right)^{1/2}
 \le\sqrt2\norm{W}_{\rF},
\]
which is the upper bound in \eqref{eq:key-frame-first-moment}.
The case $W=0$ is immediate.
\end{proof}

\begin{lemma}[Uniform estimate for $D_E$ at large stable rank]
\label{lem:stable-rank-frame}
With $D_E$ as in Lemma~\ref{lem:uniform-competitor-max}, with probability tending to
one,
\begin{equation}
 \sup_{\substack{W\in\bbR^{d\times d}:\ \norm{W}_{\rF}=1\\
                  \norm{W}_{\op}\le(\log p)^{-2}}}
 \abs{D_E(W)-1}=o(1),
 \label{eq:stable-rank-frame}
\end{equation}
where the $o(1)$ is deterministic and uniform in $W$.
\end{lemma}

\begin{proof}
Fix a $W$ in the displayed class and let $e\sim\mcN(0,\Id_d/d)$.  For
$X=d\norm{We}^2$,
\begin{equation}
 \EE X=1,
 \qquad
 \operatorname{Var}(X)
 =2\operatorname{tr}((W^{\top}W)^2)
 \le2\norm{W}_{\op}^2
 \le2(\log p)^{-4}.
 \label{eq:stable-rank-moments}
\end{equation}
Jensen's inequality and $\abs{\sqrt x-1}\le\abs{x-1}$ for $x\ge0$
therefore give
\begin{equation}
 0\le1-\sqrt d\,\EE\norm{We}
 \le\EE\abs{X-1}
 \le\sqrt2(\log p)^{-2}.
 \label{eq:stable-rank-mean}
\end{equation}

As a function of the $p$ standard Gaussian vectors $(\sqrt d\,e_\mu)_{\mu\le p}$,
$D_E(W)$ is $\norm{W}_{\op}/\sqrt p$-Lipschitz.  Gaussian concentration
\citep[Theorem~5.2.2]{vershynin2018high} hence gives
\begin{equation}
 \bbP\left(\abs{D_E(W)-\EE D_E(W)}>(\log p)^{-1/2}\right)
 \le2\exp\{-cp(\log p)^3\}.
 \label{eq:stable-rank-fixed-tail}
\end{equation}
Choose a maximal Frobenius $(\log p)^{-1/2}/(4\sqrt2)$-separated subset of the
class. It is a $(\log p)^{-1/2}/(4\sqrt2)$-net of the class by maximality, and
its elements belong to the class, so~\eqref{eq:stable-rank-fixed-tail} applies
to each of them. The Euclidean volumetric bound in dimension $d^2$ gives it
cardinality at most
\[
 \left(C\sqrt{\log p}\right)^{d^2}
 =\exp\{O(d^2\log\log p)\}.
\]
Since $d^2=O(p\log p)$, a union bound in
\eqref{eq:stable-rank-fixed-tail} closes.  Finally, on the event in
\eqref{eq:key-frame-first-moment},
\[
 \abs{D_E(W)-D_E(W')}
 \le D_E(W-W')
 \le\sqrt2\norm{W-W'}_{\rF}.
\]
This transfers the estimate from the net to the entire class. Together
with \eqref{eq:stable-rank-mean}, it gives, with probability tending to one,
$\abs{D_E(W)-1}\le\frac54(\log p)^{-1/2}+\sqrt2(\log p)^{-2}$ for every $W$ in
the class, which proves \eqref{eq:stable-rank-frame}.
\end{proof}

\section{Heuristic derivation of the capacity of the Hebbian ansatz}\label{sec_app:hebbian}

Recall that the (normalised) Hebbian ansatz reads:
\begin{equation}
    W_{\Hebb} \coloneqq \frac{1}{d} \sum_{\mu=1}^p u_\mu e_\mu^\T.
\end{equation}
Let us consider the normalised scores 
\begin{equation}\label{eq:Hebbian_scores}
    s_{\mu \rho} \coloneqq \frac{1}{d} u_\mu^\T W_\Hebb e_\rho.
\end{equation}
For any given $\mu, \rho$, one can easily compute that for $p \gg d$:
\begin{equation}\label{eq:scores_Hebb}
\begin{dcases}
    \EE[s_{\mu \rho}] &= \delta_{\mu \rho}, \\ 
    \Var[s_{\mu \rho}] &= \frac{p}{d^2} (1+ \smallO(1)) \simeq \frac{\alpha}{\log p}.
\end{dcases}
\end{equation}
While the central limit theorem ensures that the marginal distribution of any given score is close to a Gaussian distribution in the high-dimensional limit,
they are correlated with one another.
Nevertheless, a very simple heuristic is to approximate the joint distribution of these scores by independent Gaussians.
Interestingly, this heuristic predicts a threshold for the capacity of this ansatz at $\alpha = 1/8$, which seems compatible with numerical simulations presented in~\cref{fig:accuracy-equivalence}.
We formalise it in the following theorem.
\begin{theorem}\label{thm:Hebbian}
    Let $d, p \geq 1$, $\alpha > 0$, and $(s_{\mu \rho})_{\mu, \rho \in [p]}$ be independent Gaussian random variables, with mean and covariance structure given by:
    \begin{equation}
        \begin{dcases}
            \EE[s_{\mu \rho}] &= \delta_{\mu \rho}, \\ 
            \Var[s_{\mu \rho}] &= \frac{\alpha}{\log p}.
        \end{dcases}
    \end{equation}
    Then, as $d, p \to \infty$ with $(p \log p)/d^2 \to \alpha$:
    \begin{align*}
        \lim_{d \to \infty}
        \bbP[\forall \mu \in [p], \, s_{\mu \mu} \geq \max_{\rho (\neq \mu)} s_{\mu \rho}] &= 
        \begin{dcases}
            1 & \textrm{ if } \alpha < 1/8, \\
            0 & \textrm{ if } \alpha > 1/8.
        \end{dcases}
    \end{align*}
\end{theorem}
\begin{figure}
    \centering
    \includegraphics[width=1\linewidth]{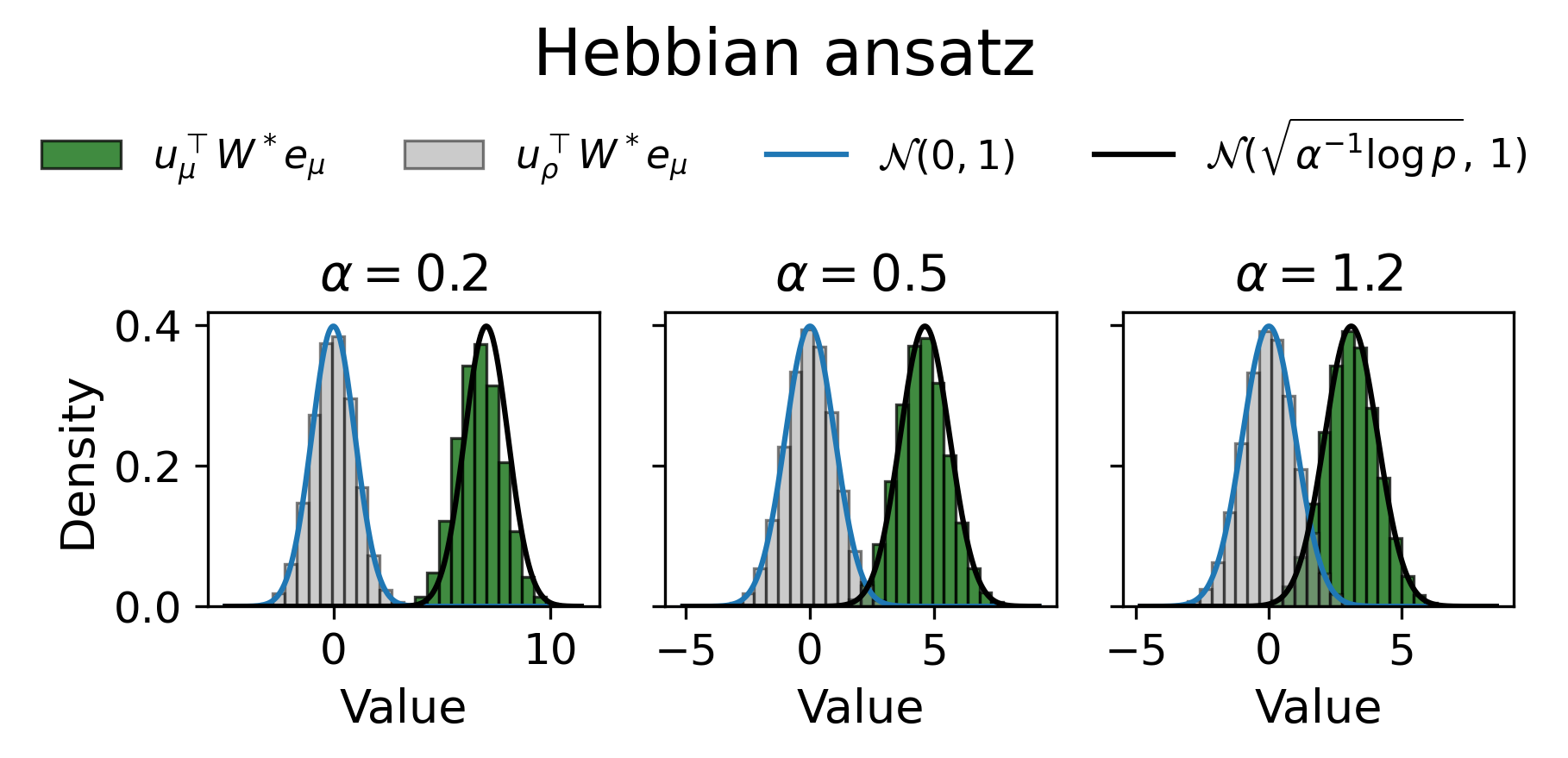}
    \caption{Distribution of the diagonal and off-diagonal scores of eq.~\eqref{eq:Hebbian_scores}, for $d=1000$. We renormalise the scores so that the off-diagonal scores have variance $1$, similarly to~\cref{fig:hist}.
    We also compare to the Gaussian distributions with theoretical mean and covariance as $d, p \to \infty$.
    }
    \label{fig:hebbian}
\end{figure}

In~\cref{fig:hebbian}, we show the histograms of the (normalised) scores of the Hebbian ansatz, validating the heuristic picture of~\cref{thm:Hebbian}: the diagonal and off-diagonal scores both are approximately distributed as Gaussians. Both have unit variance (once normalised), and whether memorisation is possible or not is determined by the distance of their means, which decreases as $\alpha$ increases\footnote{In eq.~\eqref{eq:scores_Hebb} the mean is always equal to $1$ while the variance varies with $\alpha$, which is another equivalent normalisation choice.}. 
We emphasise that this mechanism is very distinct from the one of the optimal solution as shown in~\cref{fig:hist}: the variance of the diagonal scores is there much smaller than the one of the off-diagonal ones, consistently with our theoretical analysis (see~\cref{sec:F-computation}).

We sketch the proof of~\cref{thm:Hebbian} below.
\begin{proof}[Sketch of proof of~\cref{thm:Hebbian}]
    Let $\sigma_p^2 = \alpha / \log p$.
    Let $Z_\mu \coloneqq \max_{\rho (\neq \mu)} s_{\mu \rho}$ denote the maximum of the off-diagonal elements in the $\mu$-th row. Let $E_\mu$ be the event $\{s_{\mu \mu} \geq Z_\mu\}$. 
    By assumption, the elements $(s_{\mu \rho})$ are independent. Thus, the sets of random variables $\{s_{\mu \rho}\}_{\rho=1}^p$ are mutually independent across different rows $\mu$, which makes the events $E_\mu$ mutually independent. 
    Therefore:
    \begin{equation*}
        \bbP\Big[\forall \mu \in [p], \, s_{\mu \mu} \geq \max_{\rho (\neq \mu)} s_{\mu \rho}\Big] = \bbP\left[\bigcap_{\mu=1}^p E_\mu\right] = \big( 1 - \bbP[E_1^c] \big)^p,
    \end{equation*}
    where $E_1^c = \{s_{11} < Z_1\}$ is the failure event for the first row. 
    Note that as $p \to \infty$, the limit is entirely determined by the asymptotic behaviour of $p \bbP[E_1^c]$: more precisely, if $p \bbP[E_1^c] \to 0$, the probability converges to $1$; if $p \bbP[E_1^c] \to \infty$, the probability converges to $0$.
    We evaluate $\bbP[E_1^c]$ by conditioning on the diagonal signal $s_{11}$:
    \begin{equation}\label{eq:P_E1c}
        \bbP[E_1^c] = \int_{\bbR} \bbP[Z_1 > x] f_{s_{11}}(x) \rd x = 
        \int_{\bbR} \bbP[Z_1 > x] \frac{e^{-\frac{(x-1)^2}{2\sigma_p^2}}}{\sqrt{2\pi \sigma_p^2}} \rd x.
    \end{equation}
    with $f_{s_{11}}$ the distribution function of $s_{11}$.
    Notice that terms in the integral are of the order $\exp\{\Theta(\log p)\}$ since $\sigma_p^2 = \alpha / \log p$. 
    More precisely, we have for any $x \in \bbR$ the following elementary lemma.
    \begin{lemma}\label{lemma:ldp_Z1}
    For $Z_1$ defined as above:
        \begin{equation}\label{eq:ldp_Z1}
       \lim_{p \to \infty} \frac{1}{\log p} \log \bbP[Z_1 > x] =  \min\left(1 - \frac{x^2}{2\alpha}, 0\right).
    \end{equation}
    \end{lemma}
    \cref{lemma:ldp_Z1} is shown later on.
    Using Laplace's method in eq.~\eqref{eq:P_E1c} we get:
    \begin{equation}\label{eq:Laplace_E1}
        \lim_{p \to \infty} \frac{\log(p \bbP[E_1^c])}{\log p} 
        = \sup_{x \in \bbR}\left[1+ \min\left(1 - \frac{x^2}{2\alpha}, 0\right) - \frac{(x-1)^2}{2 \alpha}\right] \eqqcolon - \inf_{x \in \bbR} [J_\alpha(x) - 1].
    \end{equation}
    One can easily evaluate the right-hand-side of eq.~\eqref{eq:Laplace_E1}.
    Notice that the rate function $J_\alpha(x)$ has two regions.
    
    \textbf{Region 1} ($x \leq \sqrt{2\alpha}$):
    Here, $J_\alpha(x) = (x-1)^2/2\alpha$.
    Its minimum value is either $0$ if $1 \leq \sqrt{2\alpha}$, or 
    $(\sqrt{2\alpha}-1)^2/(2\alpha)$ otherwise.

    \noindent\textbf{Region 2} ($x > \sqrt{2\alpha}$):
    Here
    \begin{equation*}
        J_\alpha(x) = \frac{(x-1)^2}{2\alpha} + \frac{x^2}{2\alpha} - 1 = \frac{2x^2 - 2x + 1}{2\alpha} - 1.
    \end{equation*}
    $J_\alpha'(x) = (4x - 2)/(2\alpha)$ has a zero at $x_0 = 1/2$, with $J_\alpha(1/2) = (4\alpha)^{-1} - 1$.
    
    We can then separate cases.

    \noindent\textbf{Case 1} ($\alpha < 1/8$): In this regime, $\sqrt{2\alpha} < 1/2$, 
    and $(4\alpha)^{-1} - 1 \leq (\sqrt{2\alpha}-1)^2/(2\alpha)$.
    Then $x_0 = 1/2$ is the global minimum, 
    and $J_\alpha(x_0) = (4\alpha)^{-1} - 1 > 1$. 
    Therefore by eq.~\eqref{eq:Laplace_E1}, we have 
    \begin{equation*}
        \lim_{p \to \infty} \frac{\log(p \bbP[E_1^c])}{\log p} < 0,
    \end{equation*}
    and thus $p \bbP[E_1^c] \to 0$, which implies the sought result.

    \noindent\textbf{Case 2} ($\alpha > 1/2$):
    Then the analysis of Region~1 shows that $\inf_x J_\alpha(x) = 0$, which implies that $p \bbP[E_1^c] \to \infty$, and thus the sought result.

    \noindent\textbf{Case 3} ($1/2 > \alpha > 1/8$):
    Here, $1 > \sqrt{2\alpha} > 1/2$.
    $J_\alpha$ is here decreasing strictly in $(0, \sqrt{2\alpha})$ (Region~1) and strictly increasing in $(\sqrt{2\alpha}, \infty)$ (Region~2).
    By continuity, its global minimum is reached in $x = \sqrt{2\alpha}$, with value 
    \begin{equation*}
        \frac{(\sqrt{2\alpha} - 1)^2}{2\alpha} > 1 \iff 2\alpha - 2\sqrt{2\alpha} + 1 > 2\alpha \iff 1 > 2\sqrt{2\alpha} \iff \alpha < \frac{1}{8}.
    \end{equation*}
    Therefore the minimum value of $J_\alpha$ is $\inf_x J_\alpha(x) < 1$, and we reach from eq.~\eqref{eq:Laplace_E1} that $p \bbP[E_1^c] \to \infty$, ending the proof of~\cref{thm:Hebbian}.

    It simply remains to show~\cref{lemma:ldp_Z1}.
\end{proof}

\begin{proof}[Proof of~\cref{lemma:ldp_Z1}]
Let $q_p(x) = \bbP[s_{1 \rho} > x]$ for $\rho \geq 2$. 
Denoting $\Phi(t) \coloneqq (1/\sqrt{2\pi}) \int_t^\infty e^{-u^2/2} \rd u$, we have:
    \begin{equation*}
        q_p(x)  = \Phi\left( x \sqrt{\frac{\log p}{\alpha}} \right).
    \end{equation*}
    For any $t > 0$, the standard Gaussian tail is bounded by:
    \begin{equation*}
        \frac{t}{1+t^2} \frac{1}{\sqrt{2\pi}} e^{-t^2/2} \le \Phi(t) \le \frac{1}{t \sqrt{2\pi}} e^{-t^2/2}.
    \end{equation*}
    Therefore:
    \begin{equation*}
        \log \Phi(t) = -\frac{t^2}{2} - \log t - \frac{1}{2} \log(2\pi) + O(t^{-2}).
    \end{equation*}
    Substituting $t = x \sqrt{\log p / \alpha}$, the logarithmic behaviour of the marginal tail is:
    \begin{equation}\label{eq:log_qp}
        \log q_p(x) = -\frac{x^2}{2\alpha} \log p - \frac{1}{2} \log(\log p) - \log\left( \frac{x}{\sqrt{\alpha}} \right) - \frac{1}{2}\log(2\pi) + o(1).
    \end{equation}
     Because $Z_1$ is the maximum of $p-1$ independent and identically distributed variables, we have exactly:
    \begin{equation}\label{eq:Z1_exact}
        \bbP[Z_1 > x] = 1 - \big(1 - q_p(x)\big)^{p-1}.
    \end{equation}
    We proceed by establishing matching upper and lower bounds for the limit.
    
    \noindent\textbf{Upper Bound:}
    By the union bound:
    \begin{equation*}
        \bbP[Z_1 > x] \le (p-1) q_p(x) < p \, q_p(x).
    \end{equation*}
    Therefore, using eq.~\eqref{eq:log_qp}:
   \begin{align*}
        \frac{1}{\log p} \log \bbP[Z_1 > x] &\le \frac{1}{\log p} \big( \log p + \log q_p(x) \big) \\
        &= 1 - \frac{x^2}{2\alpha} - \frac{\log(\log p)}{2 \log p} + O\left(\frac{1}{\log p}\right).
    \end{align*}
    Taking the limit as $p \to \infty$, we reach (notice that the limit must be trivially non-positive):
    \begin{equation}\label{eq:upper_bound}
        \limsup_{p \to \infty} \frac{1}{\log p} \log \bbP[Z_1 > x] \le \min \left(1 - \frac{x^2}{2\alpha},0\right).
    \end{equation}

    \noindent\textbf{Lower Bound:}
    Let us assume first $x \ge \sqrt{2\alpha}$. 
    From eq.~\eqref{eq:log_qp} we have:
    \begin{equation*}
        (p-1)q_p(x) \asymp p^{1 - \frac{x^2}{2\alpha}} (\log p)^{-1/2}.
    \end{equation*}
    This implies that $(p-1)q_p(x) \to 0$ as $p \to \infty$.
    For any $y > 0$, standard inequalities give $1 - e^{-y} \ge y - y^2/2$. Using the fact that $(1 - q)^{p-1} \le e^{-(p-1)q}$, we have from eq.~\eqref{eq:Z1_exact}:
    \begin{equation*}
        \bbP[Z_1 > x] \ge 1 - e^{-(p-1)q_p(x)} \ge (p-1)q_p(x) \left( 1 - \frac{(p-1)q_p(x)}{2} \right).
    \end{equation*}
    Taking the logarithm:
    \begin{equation*}
        \log \bbP[Z_1 > x] \ge \log(p-1) + \log q_p(x) + \log\left( 1 - \frac{(p-1)q_p(x)}{2} \right).
    \end{equation*}
    Because $(p-1)q_p(x) \to 0$, the last term converges to $0$. Dividing by $\log p$ gives:
    \begin{align*}
        \frac{1}{\log p} \log \bbP[Z_1 > x] &\ge \frac{\log(p-1)}{\log p} + \frac{\log q_p(x)}{\log p} + o(1)
= 1 - \frac{x^2}{2\alpha} + o(1),
    \end{align*}
    where we used eq.~\eqref{eq:log_qp}.
    Taking the limit as $p \to \infty$ yields :
    \begin{equation}\label{eq:lower_bound}
        \liminf_{p \to \infty} \frac{1}{\log p} \log \bbP[Z_1 > x] \ge 1 - \frac{x^2}{2\alpha}.
    \end{equation}
    It remains to consider the case $x < \sqrt{2\alpha}$.
    We will show that $\bbP[Z_1 > x] \to 1$ in this case.
    Using the standard analytic inequality $1 - y \le e^{-y}$ for all $y \in \bbR$, we obtain:
    \begin{equation}\label{eq:failure_bound}
        1 - \bbP[Z_1 > x] = \bbP[Z_1 \le x] \le \exp\big( -(p-1)q_p(x) \big).
    \end{equation}
    If $x \le 0$, then $q_p(x) \ge 1/2$. Consequently, $(p-1)q_p(x) \to \infty$ trivially as $p \to \infty$. 
    If $0 < x < \sqrt{2\alpha}$, we exponentiate the precise tail expansion from eq.~\eqref{eq:log_qp} to find:
    \begin{equation*}
        q_p(x) = \exp\left( -\frac{x^2}{2\alpha} \log p - \frac{1}{2} \log(\log p) + O(1) \right) \asymp \frac{1}{\sqrt{\log p}} p^{-\frac{x^2}{2\alpha}}.
    \end{equation*}
    Therefore
    \begin{equation*}
        (p-1)q_p(x) \asymp \frac{1}{\sqrt{\log p}} p^{1 - \frac{x^2}{2\alpha}} \to \infty,
    \end{equation*}
    since $1 - x^2/(2\alpha) > 0$.
    In all cases, this shows from eq.~\eqref{eq:failure_bound} that 
    $\lim_{p \to \infty} \bbP[Z_1 > x] = 1$, 
    and in particular
    \begin{equation*}
        \liminf_{p \to \infty} \frac{1}{\log p} \log \bbP[Z_1 > x] = \lim_{p \to \infty} \frac{o(1)}{\log p} = 0.
    \end{equation*}
\end{proof}

\end{document}